\documentclass[nohyperref]{article}

\usepackage{microtype}
\usepackage{graphicx}
\usepackage{caption}
\usepackage{subcaption}
\usepackage{booktabs} 

\usepackage[accepted,nohyperref]{icml2023}

\usepackage{amsmath}
\usepackage{amssymb}
\usepackage{mathtools}
\usepackage{amsthm}

\usepackage{dblfloatfix}

\theoremstyle{plain}
\newtheorem{theorem}{Theorem}[section]
\newtheorem{proposition}[theorem]{Proposition}
\newtheorem{lemma}[theorem]{Lemma}
\newtheorem{corollary}[theorem]{Corollary}
\theoremstyle{definition}
\newtheorem{definition}[theorem]{Definition}
\newtheorem{assumption}[theorem]{Assumption}
\theoremstyle{remark}
\newtheorem{remark}[theorem]{Remark}

\usepackage{tikz,nicefrac}
\usepackage{cancel}
\usepackage{enumitem}
\usepackage{soul}
\usepackage{xcolor}
\usepackage{caption}
\usepackage{graphicx}
\usepackage{mathtools}
\usepackage{bm}

\setlist[enumerate]{label*=\arabic*.}
\setlist{leftmargin=5.5mm} 

\newcommand{\ie}{\textit{i.e., }}
\newcommand{\eg}{\textit{e.g., }}

\newcommand{\suchthat}{\textrm{s.t.}}
\newcommand{\iid}{i.i.d.}
\newcommand\nth{\textsuperscript{th} }
\newcommand\first{\textsuperscript{st} }
\newcommand\nd{\textsuperscript{nd} }

\newcommand{\vect}[1]{\mathbf{#1}}

\newcommand{\doubleN}{\mathbb{N}}

\newcommand{\naturals}{\doubleN^{+}}

\DeclareMathOperator*{\argmin}{argmin}

\newcommand{\vw}{\vect{w}}           
\newcommand{\residual}{\vect{r}}
\newcommand{\1}{\vect{1}}
\newcommand{\0}{\mat{0}}
\newcommand{\teacher}{\vw^{\star}}
\newcommand{\mat}[1]{\mathbf{#1}}
\newcommand{\X}{\mat{X}}

\newcommand{\M}{\mat{M}}

\newcommand{\I}{\mat{I}}

\newcommand{\B}{\mat{B}}
\newcommand{\x}{\vect{x}}
\newcommand{\vv}{\vect{v}}
\newcommand{\e}{\vect{e}}
\newcommand{\valpha}{\bm{\alpha}}
\newcommand{\vu}{\vect{u}}
\newcommand{\y}{\vect{y}}
\newcommand{\w}{\vw}

\def\mP{{\mat{P}}}
\def\mQ{{\mat{Q}}}
\newcommand\smallcdots{\hskip.6pt\cdotp\!\hskip.6pt\cdotp\!\hskip.6pt\cdotp\!\hskip.6pt}
\newcommand\smalldots{.\hskip.8pt\!.\hskip.8pt\!.}
\newcommand{\scirc}{\hskip.8pt\!\circ\hskip.9pt\!}

\def\reals{\mathbb{R}}

\newcommand{\hfrac}[2]{{#1}/{#2}}
\newcommand{\norm}[1]{\left\Vert{#1}\right\Vert}
\newcommand{\abs}[1]{\left\vert{#1}\right\vert}
\newcommand{\cnt}[1]{\left[{#1}\right]}
\newcommand{\explain}[1]{\left[\substack{#1}\right]}
\newcommand{\expectation}{\mathop{\mathbb{E}}}

\newcommand{\prn}[1]{\left({#1}\right)}
\newcommand{\bigprn}[1]{\big({#1}\big)}
\newcommand{\Bigprn}[1]{\Big({#1}\Big)}
\newcommand{\biggprn}[1]{\bigg({#1}\bigg)}
\newcommand{\tprn}[1]{({#1})}
\newcommand{\smallnorm}[1]{\Vert{#1}\Vert}
\newcommand{\tnorm}[1]{\smallnorm{#1}}
\newcommand{\bignorm}[1]{\big\Vert{#1}\big\Vert}
\newcommand{\Bignorm}[1]{\Big\Vert{#1}\Big\Vert}
\newcommand{\biggnorm}[1]{\bigg\Vert{#1}\bigg\Vert}

\newcommand{\teachers}{\mathcal{W}^{\star}}
\newcommand{\feasible}{\mathcal{W}}
\newcommand{\tsum}{{\sum}}
\newcommand{\ball}[1]{\mathcal{B}^{d}}
\newcommand{\supp}{\hat{\mathcal{S}}}
\newcommand{\dataset}{S}
\newcommand{\wlim}{\w_{\infty}}
\newcommand{\lamloss}{\mathcal{L}_{\lambda}}
\newcommand{\dist}{{d}}

\renewcommand\dim{D}

\usepackage{xspace}

\newcommand{\itr}{t}
\newcommand{\bigO}{\mathcal{O}}

\newcommand{\avgitr}{{\overline{\w}}}

\newcommand{\wsign}{s}
\newcommand{\algmargin}{\hspace{-.5em}}

\newtheorem{claim}[theorem]{Claim}

\newtheorem{property}{Property}
\newtheorem{example}{Example}

\newcommand{\remove}[1]{REMOVE!}

\newenvironment{proof-sketch}{\noindent{\bf Proof sketch.}}{}

\def\figref#1{Figure~\ref{#1}}

\def\secref#1{Section~\ref{#1}}

\def\eqref#1{Eq.~(\ref{#1})}

\def\lemref#1{Lemma~\ref{#1}}

\def\thmref#1{Theorem~\ref{#1}}

\def\corref#1{Corollary~\ref{#1}}
\def\clmref#1{Claim~\ref{#1}}
\def\defref#1{Def.~\ref{#1}}
\def\propref#1{Prop.~\ref{#1}}
\def\procref#1{Scheme~\ref{#1}}
\def\asmref#1{Assumption~\ref{#1}}

\def\appref#1{Appendix~\ref{#1}}

\definecolor{residuals}{RGB}{200,27,80}

\floatname{algorithm}{Scheme}

\newcommand{\customcomment}[1]{{{#1}}}

\definecolor{itay}{RGB}{210,30,220}
\definecolor{gon}{RGB}{40,180,220}
\definecolor{mnb}{RGB}{220,180,20}
\definecolor{red2}{RGB}{240,17,17}
\definecolor{edward}{RGB}{30,200,30}
\newcommand{\itay}[1]{TODO}
\newcommand{\edward}[1]{TODO}
\newcommand{\gon}[1]{TODO}
\newcommand{\mnb}[1]{TODO}
\definecolor{todo}{RGB}{50,50,200}
\newcommand{\todo}[1]{TODO}

\newcommand{\deleted}[1]{TODO}
\newcommand{\dnote}[1]{TODO}

\newcommand{\rnotice}[1]{\customcomment{\textcolor{red2}{#1}}}
\newcommand{\nnotice}[1]{\customcomment{\textcolor{blue}{#1}}}
\newcommand{\unnotice}[1]{TODO}

\usepackage{multirow}

\makeatletter
\newenvironment{recall}[1][\proofname]{\par
\normalfont \topsep6\p@\@plus6\p@\relax
\trivlist
\item\relax
{\bfseries
Recall #1}%
{\bfseries\@addpunct{.}}\hspace\labelsep\ignorespaces
}
\makeatother

\long\def\supptitle#1{

   \gdef\@runningheadingerrortitle{0}

   \ifnum\statePaper=0
    {
     \gdef\@runningtitle{Manuscript under review by AISTATS \@conferenceyear}
    }
   \fi

   \ifnum\statePaper=1
   {
   \ifx\undefined\@runningtitle
    {
    \gdef\@runningtitle{#1}
    }
   \fi
   }
   \fi

   \ifnum\@runningheadingerrortitle=0
         {
         \global\setbox\titrun=\vbox{\small\bfseries\@runningtitle}%
         \ifdim\wd\titrun>\textwidth%
            {\gdef\@runningheadingerrortitle{2}
             \gdef\@messagetitle{Running heading title too long}
            }%
         \else\ifdim\ht\titrun>10pt
              {\gdef\@runningheadingerrortitle{3}
              \gdef\@messagetitle{Running heading title breaks the line}
              }%
              \fi
          \fi
         }
    \fi

   \ifnum\@runningheadingerrortitle>0
     {
        \fancyhead[CE]{\small\bfseries\@messagetitle}
        \ifnum\@runningheadingerrortitle>1
           \typeout{}%
           \typeout{}%
           \typeout{*******************************************************}%
           \typeout{Running heading title exceeds size limitations for running head.}%
           \typeout{Please supply a shorter form for the running head}
           \typeout{with \string\runningtitle{...}\space just after \string\begin{document}}%
           \typeout{*******************************************************}%
           \typeout{}%
           \typeout{}%
        \fi
     }
  \else
     {
          \fancyhead[CE]{\small\bfseries\@runningtitle}
     }
  \fi

  \hsize\textwidth
  \linewidth\hsize \toptitlebar {\centering
  {\Large\bfseries #1 \par}}
 \bottomtitlebar
}

\@ifundefined{@pre@hyperref}{}{\@pre@hyperref}
\usepackage[pagebackref=true]{hyperref}  
\RequirePackage{nameref}
\@ifundefined{@post@hyperref}{}{\@post@hyperref}
\hypersetup{colorlinks,
            linkcolor=[RGB]{30,123,20},
            citecolor=[RGB]{192,4,22},
            urlcolor=magenta,
            linktocpage,
            plainpages=true}

\renewcommand*\backref[1]{\ifx#1\relax \else (cited on {p.~#1}) \fi}

\newcommand{\theHalgorithm}{\arabic{algorithm}}

\usepackage[capitalize,noabbrev]{cleveref}

\icmltitlerunning{Continual Learning in Linear Classification on Separable Data}

\begin{document}

\twocolumn[
\icmltitle{Continual Learning in Linear Classification on Separable Data}


\begin{icmlauthorlist}
\icmlauthor{Itay Evron}{tech}
\hspace{1em}
\icmlauthor{Edward Moroshko}{tech}
\hspace{1em}
\icmlauthor{Gon Buzaglo}{tech}
\hspace{1em}
\icmlauthor{Maroun Khriesh}{tech}
\hspace{1em}
\icmlauthor{Badea Marjieh}{tech}
\\
\icmlauthor{Nathan Srebro}{tyt}
\hspace{1em}
\icmlauthor{Daniel Soudry}{tech}
\end{icmlauthorlist}

\icmlaffiliation{tech}{Department of Electrical and Computer Engineering, Technion, Haifa, Israel}
\icmlaffiliation{tyt}{Toyota Technological Institute at Chicago, Chicago IL, USA}

\icmlcorrespondingauthor{Itay Evron}{itay@evron.me}

\icmlkeywords{Continual learning, Linear classification, Catastrophic forgetting}

\vskip 0.3in
]

\printAffiliationsAndNotice{} 

\begin{abstract}
We analyze continual learning on a sequence of separable linear classification tasks with binary labels.
We show theoretically that learning with weak regularization reduces to solving a sequential max-margin problem, corresponding to a special case of the Projection Onto Convex Sets (POCS) framework.
We then develop upper bounds on the forgetting 
and other quantities of interest under various settings with recurring tasks, including cyclic and random orderings of tasks.
We discuss several practical implications to popular training practices like regularization scheduling and weighting.
We point out several theoretical differences between our continual classification setting and a recently studied continual regression setting.
\end{abstract}

\begin{figure*}[!b]
\label{fig:super_feasibility}
\centering
\vspace{-.6em}
\caption{Illustrating our setting from \secref{sec:setting} and
the Sequential Max-Margin dynamics from \procref{proc:adaptive} in \secref{sec:algorithmic_bias}.}
\vspace{.1in}

\hspace{.5em}
\begin{subfigure}[t]{0.31\textwidth}
    \centering
    \frame{\includegraphics[width=.99\linewidth]{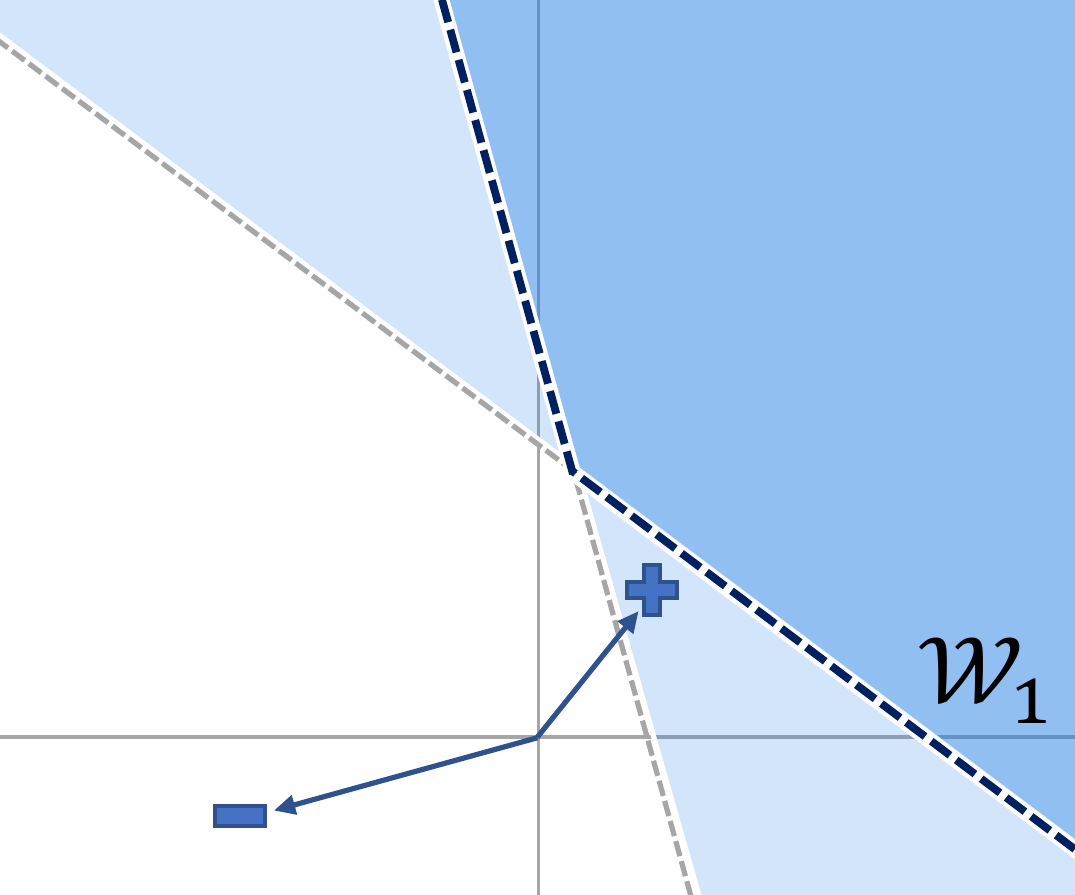}}
    \vspace{-.6em}
    \caption{
    \label{fig:one_task_feasibility}
    A $2$-dimensional task with two samples (one positive and one negative).
    Each sample $(\x,y)$ induces a constraint to a \emph{halfspace} $y\w^\top \x\ge1$.
    The task's feasible set $\feasible_1$ is defined as the intersection of these two halfspaces,
    and is thus an affine polyhedral cone.
    }
\end{subfigure}
\hfill
\begin{subfigure}[t]{0.31\textwidth}
    \centering
    \frame{\includegraphics[width=.99\linewidth]{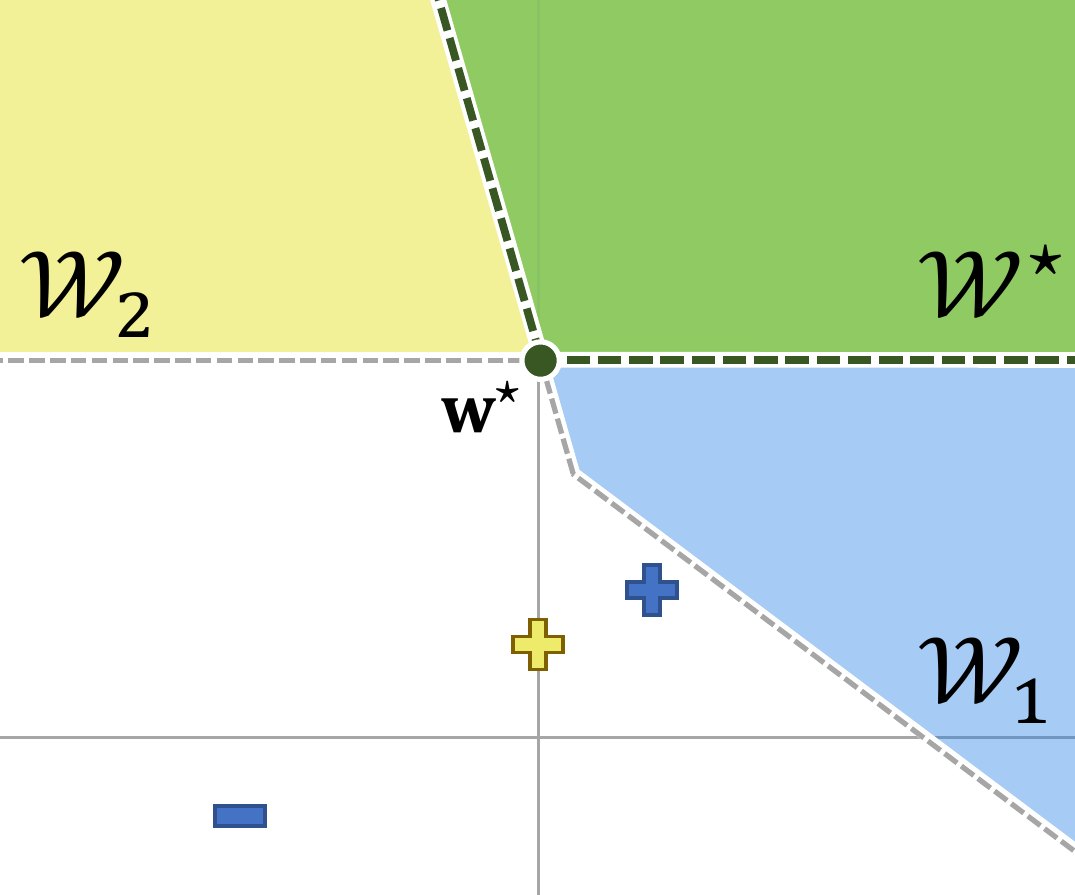}}
    \vspace{-.6em}
    \caption{
    \label{fig:two_tasks_feasibility}
    Two tasks, the first having two samples and the second having a single sample.
    The intersection of $\feasible_1$ and $\feasible_2$ defines the \linebreak
    \emph{offline} feasible set $\teachers$,
    in which all samples of all tasks are correctly classified with a margin of at least $1$.
    Notice the min-norm offline solution $\teacher\in\teachers$.
    }
\end{subfigure}
\hfill
\begin{subfigure}[t]{0.31\textwidth}
    \centering
    \frame{\includegraphics[width=.99\linewidth]{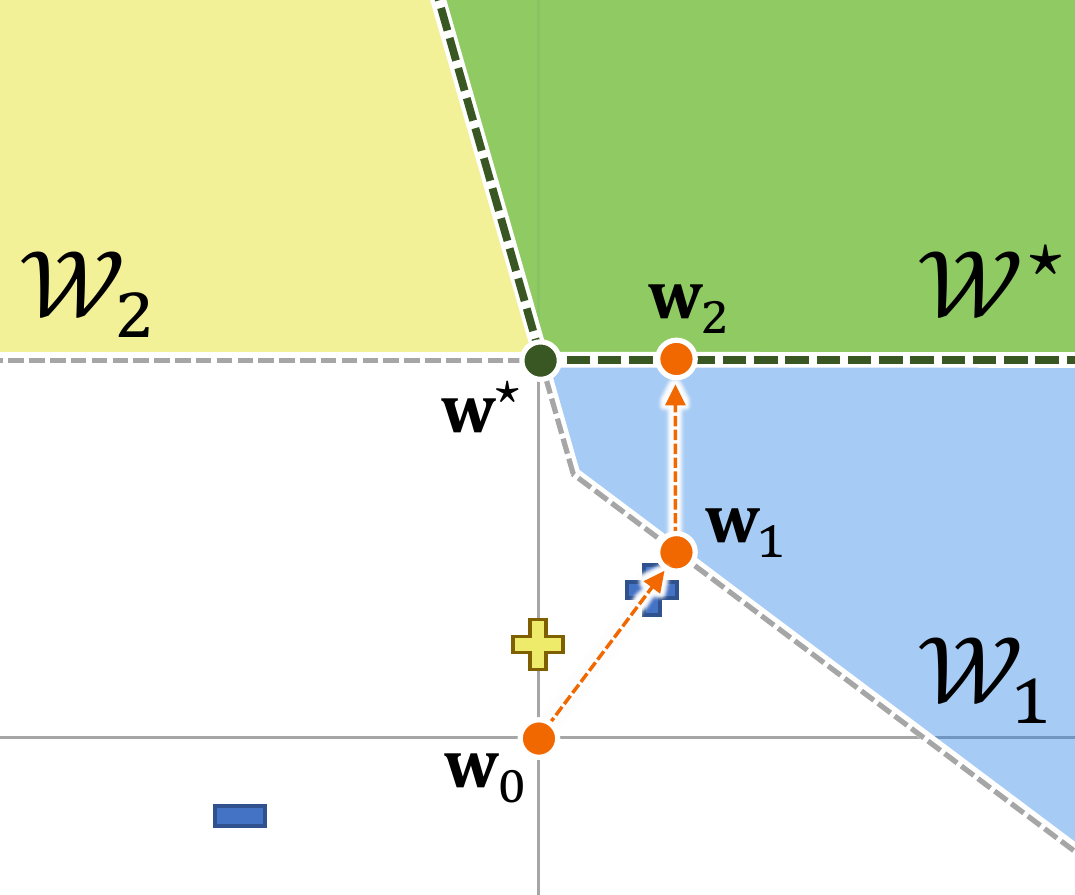}}
    \vspace{-.6em}
    \caption{
    \label{fig:feasibility_iterates}
    Learning the 1\textsuperscript{st} task projects ${\w_0\!=\!\0_{\dim}}$
    onto $\feasible_1$.
    The obtained iterate
    ${\w_1=\mP_1(\w_0)}$ is the max-margin solution of the 1\textsuperscript{st} task.
    Then, learning
    \linebreak
    the 2\textsuperscript{nd} task, projects $\w_1$ onto $\feasible_2$
    to obtain (in this case)
    an offline solution
    ${\w_2=\mP_2(\w_1)
    \in
    \teachers\subset\feasible_2}$.
    Notably, 
    $\w_2$ is \emph{not} the min-norm  $\teacher$
    (\secref{sec:convergence_to_min}).
    }
\end{subfigure}
\hspace{.5em}
\vspace{-.9em}
\end{figure*}

\section{Introduction}

Continual learning deals with learning settings where distributions, or tasks, change over time,
breaking traditional i.i.d.~assumptions.
While models trained sequentially are expected to accumulate knowledge and improve over time, practically they suffer from   
\emph{catastrophic forgetting} \citep{mccloskey1989catastrophic,goodfellow2013empirical},
\ie their performance on previously seen tasks deteriorates over time.

Much research in continual learning has focused on heuristic approaches to remedying forgetting. 
Recent approaches achieve impressive empirical performance, but often require storing examples from previous tasks
(\eg \citet{robins1995rehearsal,rolnick2019experience}), iteratively expanding the learned models 
(\eg \citet{yoon2018lifelong}), 
or lessening the plasticity of these models and harming their performance on new tasks
(\eg \citet{kirkpatrick2017overcomingEWC}).


We theoretically study the continual learning of a linear classification model on separable data with binary classes.
\linebreak
Even though this is a fundamental setup to consider, there are still very few analytic results on it,
since most of the continual learning theory thus far has focused on regression settings
(\eg \citet{bennani2020generalisationWithOGD,doan2021NTKoverlap,Asanuma_2021,lee2021taskSimilarity,evron2022catastrophic,goldfarb2023analysis,li2023fixedDesign}).
Even in the broader deep learning scope, theoreticians often start from very simple models and use them to gain insight into phenomena arising in more practical models 
(\eg \citet{belkin2018understand,woodworth2020kernel}).

Our paper reveals a surprising algorithmic bias of linear classifiers trained continually to minimize the exponential loss on separable data with weak regularization. 
Specifically, we prove that the weights converge in the same direction as the iterates of a Sequential Max-Margin scheme. 
This creates a bridge between the popular regularization methods
for continual learning 
\citep{kirkpatrick2017overcomingEWC,zenke2017continual}
and the well-studied POCS framework -- Projections Onto Convex Sets (also known as the convex feasibility problem or successive projections). 

Our results complement those of a recent paper
\citep{evron2022catastrophic}
that analyzed the worst-case performance of continual linear regression.
That paper showed that continually learning linear regression tasks with vanilla SGD,
implicitly performs sequential projections onto closed \emph{subspaces},
and connected that regime to the area of Alternating Projections \citep{vonNeumann1949rings,halperin1962product}.
In our paper, we draw comparisons between our continual \emph{classification} setting and their continual \emph{regression} setting
(summarized in App.~\ref{app:comparison}).
We point out inherent differences between these two settings, emphasizing the need for a proper and thorough analytical understanding dedicated to continual \emph{classification} settings.

\paragraph{Our Contributions}
Our analysis reveals the following:
\begin{itemize}[leftmargin=5mm]\itemsep.8pt
    \item Explicit regularization methods for continual learning of linear classification models are linked to the framework of Projections Onto Convex Sets (POCS).
    \item Each learned task
    brings the learner \emph{closer} to an ``offline'' feasible solution solving 
    \emph{all} tasks.
    However, there exist task sequences for which the learner stays arbitrarily far from offline feasibility, even after infinitely many tasks.
    \item When tasks recur cyclically or randomly, the learner converges to an offline solution with linear rates.
    \item 
    If we converge to an offline solution, it may not be the minimum-norm solution (in contrast to continual regression), but it still needs to be $2$-optimal (minimal).
    %
    \item Scheduling the regularization strength endangers convergence to an offline solution and optimality guarantees.
    \item Using popular regularization weighting schemes based on Fisher-information matrices, 
    does not necessarily prevent forgetting
    (in contrast to continual regression).
    %
    %
    \item Early stopping (without regularization) does \emph{not} yield the same solutions as weak regularization (unlike in stationary settings with a single task).
\end{itemize}


\section{Setting}
\label{sec:setting}

We consider $T\ge 2$ binary classification tasks.
Each task $m=1,\dots,T$ 
is defined by
a dataset $\dataset_m$ consisting of
tuples 
of {$\dim$-dimensional} samples and their binary labels,
\linebreak
\ie each tuple is
$(\x, y)
\!\in\!
{\,\reals^{\dim} \!\times\! \left\{-1,+1\right\}\,}$,
for a finite $\dim$.

\paragraph{Notation.}
Throughout the paper, we denote the (isotropic) Euclidean norm of vectors by $\norm{\vect{v}}$,
and the weighted norm by 
$\norm{\vv}_{\B}
\triangleq
\vv^\top \B \vv$,
for some $\B\succ\0_{\dim\times\dim}$.
We denote the set of natural numbers starting from $1$ by ${\naturals\triangleq \mathbb{N}\setminus\{0\}}$
and the natural numbers from $1$ to $n$ by $\cnt{n}$.
\linebreak
We define the distance of a vector $\w\in\reals^{\dim}$ from a closed set $\mathcal{C}\subseteq\reals^{\dim}$ as
${\dist(\w,\mathcal{C})\triangleq 
\min_{\vv\in\mathcal{C}}\norm{\w-\vv}}$.
\linebreak
Finally, we denote the maximal norm of any data point by ${R\triangleq \max_{m\in\cnt{T}}\max_{(\x,y)\in\dataset_m }\norm{\x}}$.

\pagebreak

Our main assumption in this paper is that the tasks are jointly-separable,
\ie they can be perfectly fitted simultaneously, as in ``offline'' non-continual settings. 
This can be formally stated as follows.

\begin{assumption}[Separability]
\label{asm:separability}
Each task $m\in\cnt{T}$ is separable, 
\ie it has a \emph{non-empty} feasible set
defined as
$$
\feasible_m
\,\!\triangleq\!\,
\left\{
\w\!\in\!\reals^{\dim}
\mid
y \w^\top \x\ge 1,
~
\forall (\x, y)\in \dataset_m \right\}\,.
$$
Moreover, the $T$ tasks are \emph{jointly}-separable ---
there exists a {non-empty} {offline}
feasible set:
$$\teachers
\,\!\triangleq\!\,
\feasible_1
\cap \dots \cap
\feasible_T
\neq\emptyset
\,.$$
\end{assumption}
A similar assumption was made in the continual regression setting \citep{evron2022catastrophic}. 
It is a reasonable assumption in overparameterized regimes, 
where feasible offline solutions often \emph{do} exist. 
Practically, this is commonly the case in modern deep networks.
Theoretically, with sufficient overparameterization \citep{du2019gradient} or high enough margin \citep{Ji2020Polylogarithmic}, 
it is often easy to converge to a zero-loss solution.

\bigskip

To facilitate our results and discussions,
we specifically define the minimum-norm offline solution.
This solution is traditionally linked to good generalization.

\begin{definition}[Minimum-norm offline solution]
\label{def:min-norm-solution}
We denote the offline solution with the minimal norm by
\begin{align*}        
\begin{aligned}
\teacher 
\triangleq
{\argmin}_{\w\in\teachers}
\norm{\w}\,.
\end{aligned}
\end{align*}
\end{definition}

\medskip

Figures~\ref{fig:one_task_feasibility}~and~\ref{fig:two_tasks_feasibility} illustrate our definitions.
Notice how both the feasible set $\feasible_m$ of each task and the offline feasible set $\teachers$ are closed, convex, and affine polyhedral cones.

\pagebreak

\section{Algorithmic Bias in Regularization Methods}
\label{sec:algorithmic_bias}

Regularization methods are highly influential in continual learning
\citep{kirkpatrick2017overcomingEWC,zenke2017continual,aljundi2018memory}.
In this section, we propose a novel analysis for such methods on separable datasets in the spirit of theoretical work on algorithmic biases outside the scope of continual learning.
Concretely, we discover that
weakly-regularized models, trained sequentially to minimize the exponential loss,%
\footnote{
It should be possible to extend our results to other losses with exponential tails, \eg cross-entropy,
as has been done 
in previous theoretical works on stationary settings
\citep{soudry2018journal}.
}
converge in direction to the iterates of a sequential projection scheme.

Specifically, we study \procref{proc:regularized}, in which the learner sequentially sees one task
(out of $T$)
at a time,
for $k$ iterations
($k\!>\!T$ implies repetitions).
Starting from $\w_0^{(\lambda)}\!\!=\!\!\0_{\dim}$,
at each iteration $t\!\in\!\cnt{k}$,
the learner minimizes
the exponential loss of the current task's dataset $\dataset_{t}$, 
while biasing towards the previous task's solution $\w_{t-1}^{(\lambda)}$ 
using the Euclidean norm.
The regularization strength is determined by a sequence of (possibly constant) positive scalars $\lambda_1,\smalldots,\lambda_{k}\!>\!0$.
The norms are possibly weighted by a sequence of positive-definite matrices $\B_1,\smalldots,\B_{k}\succ\0_{\dim\times\dim}$. 

\begin{algorithm}[ht!]
   \caption{Regularized Continual Learning
    \label{proc:regularized}}
\begin{algorithmic}
    \vspace{-.1em}
   \STATE {\algmargin\bfseries Initialization:} 
   $\w_{0}^{(\lambda)} = \0_{\dim}$
   \STATE {\algmargin{\bfseries Iterative update for each task $t\in[k] 
   $:}}
    \vspace{-.7em}
    \begin{align}
    \algmargin
    \label{eq:weakly_regularized_problem}
    \w_{t}^{(\lambda)}\!
    =
    \argmin_{\w \in \reals^{\dim}} \!\!
    {
    \sum_{(\x,y)\in \dataset_{t} }
    \!\!\!
    e^{-y \w^\top \x }
    +
    \frac{\lambda_t}{2}
    \norm{\w\!-\!\w_{t-1}^{(\lambda)}}^2_{\B_{t}}
    }
    \end{align}
    \vspace{-1.1em}
\end{algorithmic}
\end{algorithm}

\vspace{-.4em}
To clarify, each of the $k$ iterations 
corresponds to learning a \emph{whole} task (to convergence), and not to performing a single gradient step.

\vspace{.2em}

As a start, we first focus on 
\emph{constant} strengths $\lambda\!>\!0$
\linebreak
and on
``vanilla'' L$2$ regularization, \ie 
$\B_1 = \smalldots = \B_{k} = \I$.
Vanilla L$2$ regularization has recently been shown to be competitive 
with more popular norm-weighting schemes like EWC
\citep{lubana2022quadratic,smith2022closer}.
We address the more complicated setups in \secref{sec:extensions}.

For an arbitrary $\lambda$,
the objective in \eqref{eq:weakly_regularized_problem}
is hard to analyze,
even for the first task, where the regularization term reduces to the traditional unbiased term $\norm{\w}^2$.
\linebreak
Previous works were still able to perform non-trivial analysis
by examining the weakly-regularized case, 
\ie in the limit of $\lambda\to 0$
(\eg for linear models \citep{rosset2004margin} 
or homogeneous neural networks \citep{wei2019regularization}).

Under our continual setting,
we also take this approach and analyze 
the weakly-regularized model.
This enables us to gain analytical insights into regularization methods for continual learning.
We establish an equivalence between the weakly-regularized \procref{proc:regularized} and the following Sequential Max-Margin \procref{proc:adaptive},
and propose novel perspectives and techniques for analyzing regularization methods.

\begin{algorithm}[ht!]
   \caption{Sequential Max-Margin}
    \label{proc:adaptive}
\begin{algorithmic}
   \STATE {\algmargin\bfseries Initialization:} 
   $\w_0 = \0_{\dim}$
   \STATE {\algmargin{\bfseries Iterative update for each task $t\in[k] 
   $:}} 
    \vspace{-.3em}
    \begin{align}
    \algmargin
    \label{eq:adaptive_dynamics}
    \w_{t}
    =
    \mP_{t}\prn{\w_{t-1}}
    \triangleq
    \argmin_{\w \in \reals^{\dim}} 
    \enskip & 
    \!
    \norm{\w-{\w}_{t-1}}^2
    \\
    \suchthat
    \enskip & 
    \,
    y\w^\top \x \!\ge\! 1,
    \,
    \forall (\x,y)\!\in\! \dataset_{t}
    \nonumber
    \end{align}
    \vspace{-1.6em}
\end{algorithmic}
\end{algorithm}

\vspace{-0.3em}

We are now ready to state our fundamental result,
showing that 
the regularized continual iterates $\tprn{\w_t^{(\lambda)}}$  of \procref{proc:regularized}
converge in direction to the Sequential Max-Margin 
iterates
$\prn{\w_t}=\tprn{\mP_{t}\prn{\w_{t-1}}}$ of \procref{proc:adaptive},
obtained by successive projections onto closed convex sets, as depicted in \figref{fig:feasibility_iterates}.

\vspace{0.3em}

\begin{theorem}[Weakly-regularized Continual Learning converges to Sequential Max-Margin]
\label{thm:weakly_regularized}
Let $\lambda_t\!=\!\lambda\!>\!0$, 
and
\linebreak
$\B_{t}\!=\!\I$, $\forall t\!\in\!\cnt{k}$.
Then, 
for almost all separable datasets,%
\footnote{
This holds w.p.~$1$ for
separable datasets (Assumption~\ref{asm:separability}) sampled from \emph{any} absolutely continuous distribution.
It also holds even when the datasets are separable 
but not \emph{jointly}-separable.
\label{fn:almost_all}
}
\linebreak
in the limit of $\lambda\!\to\!0$,
it holds that
$\w_{t}^{(\lambda)} 
\to
{\ln\prn{\tfrac{1}{\lambda}}\w_{t}}$
with a residual of
$\tnorm{\w_{t}^{(\lambda)} 
\!-
{\ln\prn{\tfrac{1}{\lambda}}\w_{t}}}
=\bigO\prn{t\ln\ln\prn{\tfrac{1}{\lambda}}}
$. 
\linebreak
As a result, 
at any iteration 
$t=o\prn{\frac{\ln\prn{\nicefrac{1}{\lambda}}}{\ln \ln\prn{\nicefrac{1}{\lambda}}}}$,
we get
\vspace{-0.5em}
$$
\lim_{\lambda\to0}
\frac{\w_t^{(\lambda)}}{\tnorm{\w_t^{(\lambda)}}}
=
\frac{\w_t}{\tnorm{\w_t}}
\,.
$$
\end{theorem}
\vspace{-0.5em}

In \appref{app:algorithmic_bias}, we prove this theorem.
There, we also discuss the limitations of our analysis (\appref{sec:limitations}) and identify aspects where it can be improved.

\vspace{0.8em}

\begin{remark}[Important differences from existing analysis]
\label{rmk:analysis_differences}
Existing works (\eg \citet{rosset2004margin}) analyzed weakly-regularized models with an \emph{unbiased} regularizer $\norm{\w}_{p}^{p}$.
This allowed them to concentrate on the limit \emph{margin} which implies convergence in direction.
However, we show that to analyze the \emph{continual} regularizer in \eqref{eq:weakly_regularized_problem},
one must also take into account the \emph{scale} of the solutions $(\w_{t}^{(\lambda)})$.
Therefore, we analyze both the scale \emph{and} the direction of weakly-regularized solutions,
requiring more refined techniques.
\end{remark}


\section{Sequential Max-Margin Projections}\label{sec:adaptive_svm}

Now, we turn to exploit the connection we have established between the weakly-regularized continual 
Scheme~\ref{proc:regularized}
and the Sequential Max-Margin Scheme~\ref{proc:adaptive}.
Using tools from existing literature on Projections Onto Convex Sets (POCS), we gain valuable insights into 
the dynamics of training
models continually.
Specifically, we derive optimality guarantees and convergence bounds in several interesting settings.

\pagebreak

\subsection{Quantities of Interest}
\label{sec:quantities}
We have three quantities of interest.
While the first two are widely used in the POCS literature,
the latter is specific to our continual classification setting.

\begin{definition}[Quantities of interest]
\label{def:quantities}
Let $\w_{t}\in\reals^{\dim}$ be the $t$\nth iterate, 
obtained while continually learning $T$ jointly-separable tasks.
The following quantities are of interest: 
\vspace{-0.6cm}
\begin{enumerate}[leftmargin=0.36cm, itemindent=0cm, labelsep=0.1cm]\itemsep1pt
    \item \textbf{Distance to the offline feasible set:}
    \hfill
    $\dist\!\prn{\w_{t}, \teachers}$

    \item \textbf{Maximum dist.~to \emph{any} feasible set:}
    \hfill
    $
    \m@th\displaystyle
    \max_{m\in\cnt{T}}
    \!
    \dist\!\prn{\w_{t},\feasible_m}
    $ 
    
    \item \textbf{Forgetting:}
    We define the forgetting of a previously-seen task ${m\in\cnt{t}}$ as the maximal squared hinge loss on \emph{any} sample $(\x,y)$ of the task. 
    More formally,
    $$
    F_{m}
    (\w_{t})
    \triangleq
    {\max}_{(\x,y)\in\dataset_m}
    \!
    \left(
    \max\left\{0,\,1-y\w_{t}^\top \x\right\}
    \right)^2\,.
    $$ 
    Throughout our paper,
    we 
    analyze both the maximal and the average forgetting,
    \ie 
    \linebreak
    $\max_{m\in\cnt{t}}F_{m}(\w_{t})$
    and
    $\tfrac{1}{t}\sum_{m=1}^{t}F_{m}(\w_{t})$.
\end{enumerate}
\end{definition}

\paragraph{Explaining our forgetting.}
Previous works on continual linear \emph{regression} defined forgetting using the MSE on previously-seen tasks,
\linebreak
\ie 
$
{
\forall m<t\!:~
F_{m}(\w_{t})
=
\tfrac{1}{\abs{\dataset_{m}}}
\sum_{(\x,y)\in\dataset_{m}}
\prn{\w_{t}^{\top}\x-y}^2
}
$
(\eg \citet{doan2021NTKoverlap,evron2022catastrophic}).
Then, lower forgetting implies better \emph{training} loss on previous tasks.

In continual linear \emph{classification}, 
hinge losses capture similar properties.
First, per our definitions,
immediately after learning the $m$\nth task, 
the forgetting on it is $0$,
since $\forall \vv\!\in\!\reals^{\dim}\!:
\mP_m(\vv)\!\in\!\feasible_m$ 
and 
$\forall\vu\!\in\!\feasible_{m}\!:
F_m(\vu)=0$.
Moreover, 
a lower forgetting implies
better training margins and generalization performance on previous tasks.
Since the squared hinge loss is a surrogate for the 0-1 loss, $F_m \!\prn{\w_{t}}\!<\!1$
implies no prediction errors on the $m$\nth task.

\vspace{.55em}

\begin{remark}[Forgetting vs.~Regret]
\label{rmk:regret}
\textbf{Forgetting} is different from the \textbf{regret} used to analyze online learning algorithms \cite{crammer2006onlinePassAgg,shalev2012online,HOIonline}.
Forgetting quantifies the degradation on previous tasks in hindsight, 
while regret cumulatively captures the ability to predict future datapoints (or tasks). 
\end{remark}

A favorable property of the quantities we defined, is that they bound each other.
\begin{lemma}[Connecting quantities]
\label{lem:euclidean_to_hinge}
Recall our definition of
$
{R\triangleq 
\max_{m\in\cnt{T}}
\max_{(\x,y)\in\dataset_m }
\!\!\norm{\x}}$.
The quantities of \defref{def:quantities} are related
as follows,
$\forall\w\!\in\!\reals^{\dim},
\,m\!\in\!\cnt{T}$:
\begin{align*}
    F_m(\w)
    \le
    \dist^2(\w, \feasible_m)
    \!\!
    \max_{(\x,y)\in\dataset_m}
    \!\!\!\!
    \tnorm{\x}^2
    \le
    \dist^2(\w, \teachers) R^2.
\end{align*}
Moreover, for the Sequential Max-Margin iterates 
$(\w_t)$ of Scheme~\ref{proc:adaptive},
all quantities are upper bounded by the 
\linebreak
``problem complexity'',
\ie $\dist^2(\w_t, \teachers) R^2\le\norm{\teacher}^2\!R^2$.
\end{lemma}

\vspace{-0.1em}

The proof is given in \appref{app:successive-projections}.

\pagebreak

\begin{remark}[Problem complexity]
\label{rmk:complexity}
    Many of our bounds use $\norm{\teacher}^2\!R^2$,
    which can be seen as the problem complexity due to its links to sample complexities \citep{novikoff62perceptron}.
    \linebreak
    Moreover, 
    Assumption~\ref{asm:separability}, 
    implies that
    $\norm{\teacher}^2\!R^2\ge1$.
\end{remark}

The following is a known useful result from the POCS literature
(\eg Lemma~3 in \citet{gubin1967methodProjections}).
\begin{lemma}[Monotonicity of distances to offline feasibility]
    \label{lem:feasible_distance_monotonicity}
    Distances from the Sequential Max-Margin iterates 
    of \procref{proc:adaptive}
    to the offline feasible set
    are non-increasing, \ie
    \vspace{-0.1em}
    $$
    \dist(\w_{t}, \teachers) \le \dist(\w_{t-1}, \teachers),
    ~\,\forall t\in\cnt{k}
    \,.
    $$
\end{lemma}

\vspace{-1.2em}
    
\paragraph{Question}
An immediate question arises from \lemref{lem:feasible_distance_monotonicity}:
\linebreak
when learning infinite jointly-separable tasks
($k=T\!\to\!\infty$),
\linebreak
\emph{must we converge to the offline feasible set $\teachers$?}
\linebreak
Next, we answer this question in the negative.

\subsection{Adversarial Construction: 
Maximal Forgetting}
\label{sec:adversarial}
\begin{example}[Adversarial construction]
\label{exm:adversarial}
We present a construction of task sequences that seemingly exhibit arbitrarily bad continual performance.
Even after seeing $k=T\to\infty$ jointly-separable tasks, the learner stays afar from the offline feasible set $\teachers$, \emph{and} the forgetting of previously-seen tasks is maximal.
The learner \emph{fails} to successfully accumulate experience.
See further details in 
\appref{app:adversarial}.
\end{example}

\begin{figure}[ht!]
    \vspace{-.1cm}
    \centering
    \begin{subfigure}[t]{1\linewidth}
      \centering
      \begin{minipage}[t!]{0.55\linewidth}
        \caption{
        Our construction with $T\!=\!20$ tasks in $\dim\!=\!3$.
            Datapoints (each defining a task) have a norm of $R\!=\!1$
            and are spread uniformly on a plane, slightly elevated above the $xy$-plane.
            It holds that $\norm{\teacher}\!=\!10$.
            As $T\!\to\!\infty$, 
            angles between consecutive tasks
            and applied projections, get smaller.
            \label{fig:adversarial-data}
        }
      \end{minipage}
      \hfill
      \begin{minipage}[t!]{0.44\linewidth}
        \vspace{-.58cm}
        {\includegraphics[width=.99\linewidth]{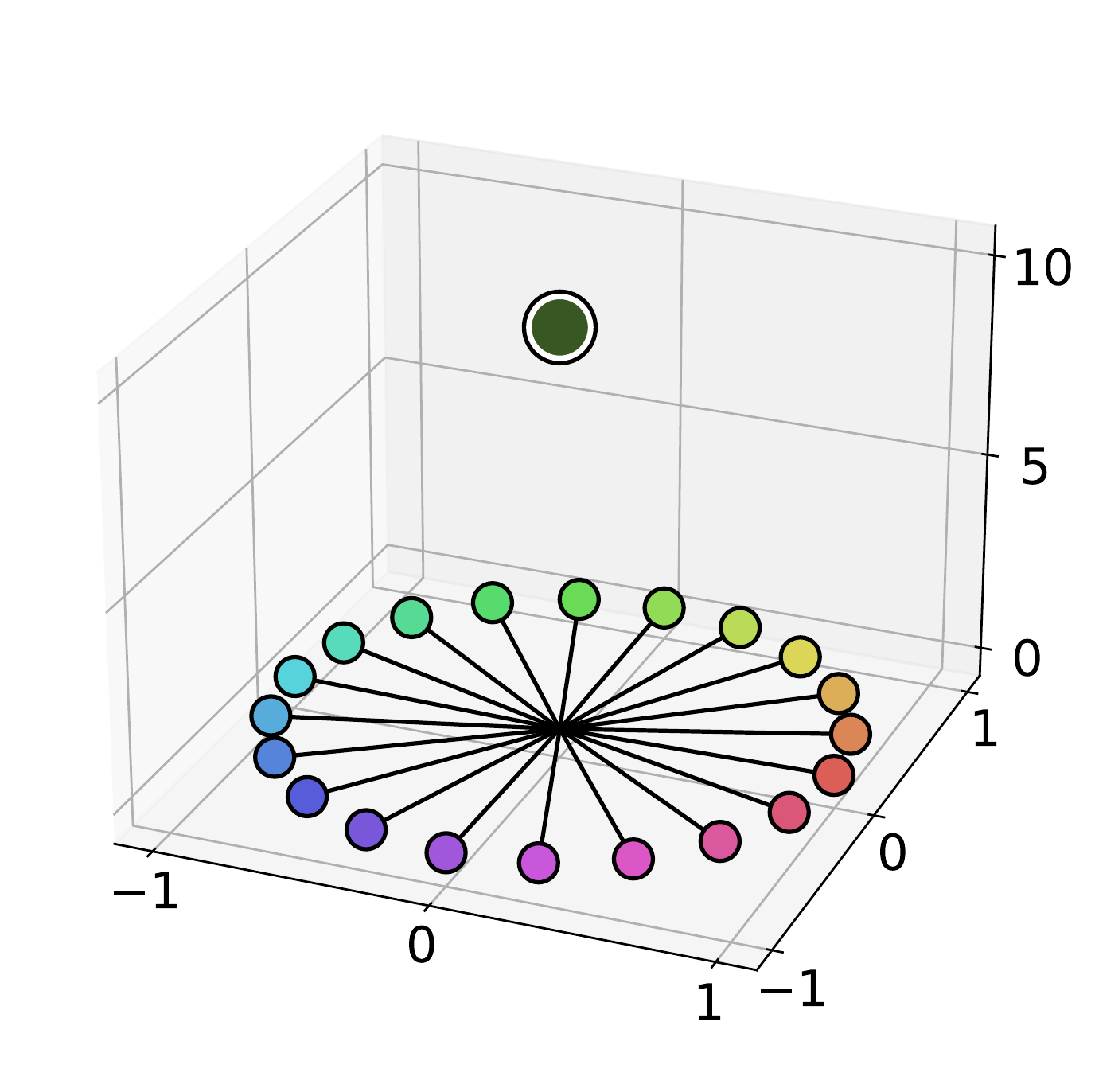}}
      \end{minipage}
    \end{subfigure}%

    \vspace{-.4em}
    
    \begin{subfigure}[t]{1\linewidth}
            \centering
            {\includegraphics[width=.99\linewidth]{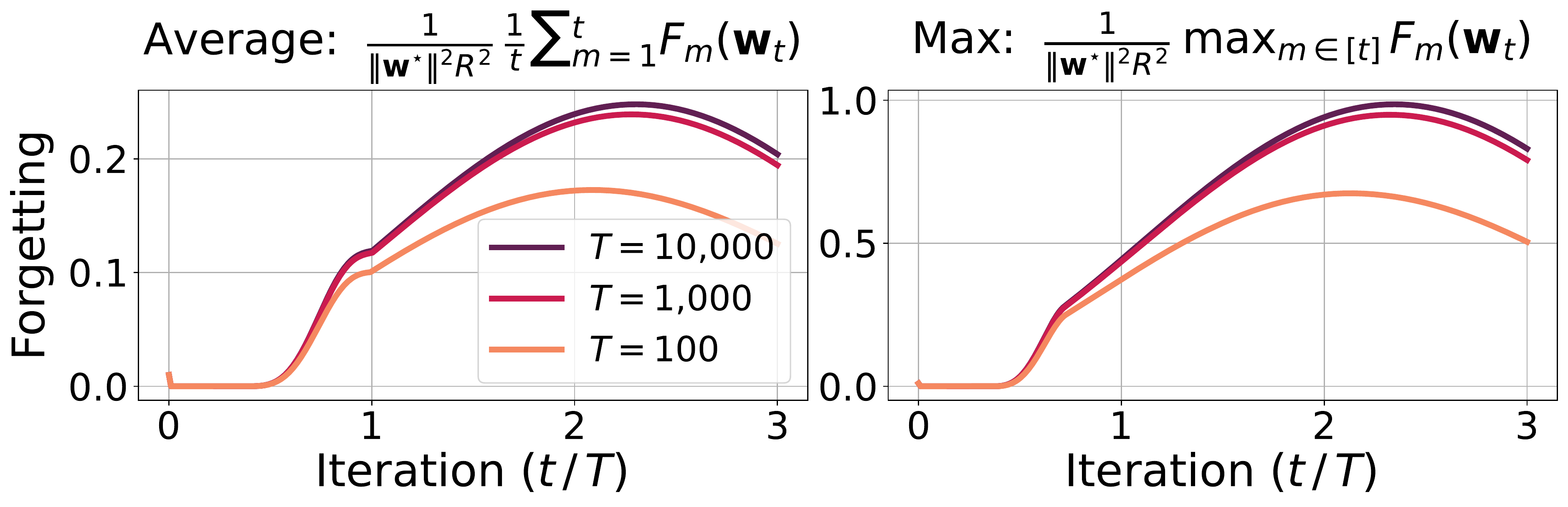}}
            \caption{
            The average and maximum forgetting for the adversarial construction
            for an increasing number of tasks $T$.
            Recall that the maximum forgetting lower bounds the distance to the offline feasible set (\lemref{lem:euclidean_to_hinge}).
            Notably, after learning $T\to \infty$ jointly separable tasks, 
            the quantities do \emph{not} decay but seemingly become arbitrarily bad (close to $\norm{\teacher}^2\!R^2$) at some point of learning.
            }
            \label{fig:adversarial}
    \end{subfigure}%
    \vspace{-.1cm}
    \caption{
    \label{fig:adaptive_illustrations}
    Illustrations of our adversarial construction.
    }
\end{figure}

\vspace{.5em}

\begin{remark}[Order of limits]
\label{rmk:limits_order}
Our paper analyzes the continual learning of $T$ tasks for $k$ iterations,
possibly taking ${k\!\to\!\infty}$
(\eg by repeating tasks).
We take $\lambda\!\to\!0$ \emph{after} fixing the number of iterations~$k$.
As a result, limit iterates hold
\vspace{-.3em}
$$
\frac{\wlim}{\tnorm{\wlim}}
\triangleq
\lim_{k\to \infty} \frac{\w_k}{\tnorm{\w_k}}
=
\lim_{k\to \infty} \lim_{\lambda\to 0}
\frac{\w_k^{(\lambda)}}{\tnorm{\w_k^{(\lambda)}}}
\,.
$$
\end{remark}

\pagebreak

\subsection{Convergence to the Minimum-Norm Solution}
\label{sec:convergence_to_min}
When solving feasibility problems for classification 
(\eg in hard-margin SVM), the minimum-norm solution is also the max-margin solution.
In turn, max-margin solutions are theoretically linked to better generalization performance.

In realizable continual linear regression 
(or generally, in alternating projections onto closed \emph{subspaces}),
it is known that if iterates converge to an offline (or globally feasible) solution,
then that solution must 
be the closest to $\w_0$,
\ie
have a minimum norm 
\citep{evron2022catastrophic,halperin1962product}.

In contrast, in separable continual linear classification settings like ours 
(or generally, in projections onto closed convex sets),
there is no such guarantee
and we can converge to a suboptimal offline solution,
as depicted in \figref{fig:feasibility_iterates}.

Nevertheless, the following optimality guarantee \emph{does} hold.

\begin{theorem}[Optimality guarantee]
\label{thm:optimality_guarantees}
Any iterate $\w_{t}$ obtained by \procref{proc:adaptive}
holds 
$\norm{\w_{t}}\le 2\norm{\teacher}$,
where $\teacher$ is the minimum-norm offline solution 
(\defref{def:min-norm-solution}).

If additionally, $\w_{t}$ is ``offline''-feasible, 
\ie $\w_{t} \in \teachers$,
then
\vspace{-.6em}
$$
\norm{\teacher} \le \norm{\w_{t}} \le 2\norm{\teacher}\,.
$$
\end{theorem}
\vspace{-.35em}
The proof is given in \appref{app:min_norm}.
Further comparisons to continual regression are drawn in \appref{app:comparison}.

\subsection{Recurring Tasks}
\label{sec:repetitions}

Naturally, in many practical continual problems,
certain concepts and experiences recur at different tasks
(\eg environments of an autonomous vehicle, levels of a computer game, trends of a search engine, etc.).

Several recent papers have observed empirically that  task repetitions mitigate catastrophic forgetting in continual learning
\citep{stojanov2019incremental,cossu2022class},
even when training is performed with vanilla SGD, without any forgetting-preventing method
\citep{lesort2022scaling}.

In this section, we analytically study the influence of repetitions. 
To accomplish this, we leverage the connection that we have established between continual learning and successive projection algorithms for convex feasibility problems.
\linebreak
Importantly, we \emph{do not} propose repetitions as a training method but rather aim to understand their effects on continual learning from a projection perspective.

The results of this section are summarized in Table~\ref{tbl:summary}.

\begin{table*}[b!]
\vspace{-0.9em}
\caption{Summary of our upper bounds for recurring orderings (\secref{sec:repetitions}). 
Upper bounds of random orderings apply to the \emph{expectations}.
}
\label{tbl:summary}
\vskip 0.1in
\begin{center}
\begin{small}
\begin{sc}
\begin{tabular}{l|l|c|c|c}
\toprule
Ordering & 
Iterate type 
&
$\m@th\displaystyle
\tfrac{1}{\tnorm{\teacher}^2 R^{2}}
{\max}_{m\in\cnt{T}} 
F_m (\w_k)
$
& 
$\m@th\displaystyle
\tfrac{1}{\tnorm{\teacher}^2}{\max}_{m\in\cnt{T}} \dist^2\prn{\w_k, \feasible_m}$
& 
$\m@th\displaystyle
\tfrac{1}{\tnorm{\teacher}^2} \dist^2\prn{\w_k, \teachers}$
\\
\midrule
Cyclic   
    &
    Last ($T=2$)
    &
    \multicolumn{2}{c|}{
    $\m@th\displaystyle
    \min\bigg\{
    \frac{1}{k+1},\,\,\,\,\,\,\,\,\,\,\,\,\,\,\,
    \exp
    \!\Bigprn{-\frac{k}{4\tnorm{\teacher}^2 R^2}}
    \Big\}
    $
    }
    &
    $\m@th\displaystyle
    \exp
    \!\Bigprn{-\frac{k}{4\tnorm{\teacher}^2 R^2}}
    $
\\ 
&
    Last ($T\ge3$)
    &
    \multicolumn{2}{c|}{
    $\m@th\displaystyle
    \min\bigg\{\,
    \frac{2T^2}{\sqrt{k}},
    \,\,\,4
    \exp
    \!\Bigprn{-\frac{k}{16T^2\tnorm{\teacher}^2 R^2}}
    \Big\}
    $
    }
    &
    $\m@th\displaystyle
    \exp
    \!\Bigprn{-\frac{k}{16T^2\tnorm{\teacher}^2 R^2}}
    $
\\
&
    Average ($T\ge2$)
& 
    \multicolumn{2}{c|}{
    $\hfrac{T^2}{k} $
    }
& 
    ---
\\
\midrule
Random   
    &
    Last
    &
    \multicolumn{2}{c|}{
    $\m@th\displaystyle
    \exp
    \!\Bigprn{-\frac{k}{4T\tnorm{\teacher}^2 R^2}}
    $
    }
    &
    $\m@th\displaystyle
    \exp
    \!\Bigprn{-\frac{k}{4T\tnorm{\teacher}^2 R^2}}
    $
\\
(i.i.d.)   
    &
    Average
    &
    \multicolumn{2}{c|}{
    $\hfrac{T}{k} $
    }
    &
    ---
\\
\bottomrule
\end{tabular}
\end{sc}
\end{small}
\end{center}
\vskip -0.1in
\end{table*}

\paragraph{Recurring tasks vs. Batch learning}
The recurring tasks setting should not be confused with standard batch training
(\ie by regarding each task as a batch). 
While in batch training, a \emph{single} gradient-descent step is made for each batch, 
in our continual learning setting
(\procref{proc:regularized})
each task is solved completely (to separation).

\pagebreak

The following is a key lemma in our paper.
Much of the research on POCS has focused on defining and applying linear regularity conditions, under which iterates $(\w_t)$ converge linearly 
(like $c^{t}$ for some $c\!\in\!\left[0,1\right)$) 
to the feasible sets' intersection $\teachers$
(\eg \citet{bauschke1993convergence}).
\linebreak
In realizable continual linear \emph{regression}, \citet{evron2022catastrophic} showed that no such regularity holds, and while their forgetting,
\ie ${F_m(\w_t)\triangleq \dist^2 (\w_t, \feasible_m),~\forall m\!\in\!\cnt{T}}$,
is upper bounded \emph{universally}, 
no universal bounds can be derived for 
$\dist^2 (\w_t, \teachers)$
even when $\norm{\teacher}$ and $R$ are bounded.
\linebreak
In contrast, we prove that in separable continual linear \emph{classification},
linear regularity \emph{does} hold
and nontrivial bounds \emph{can} be derived for $\dist^2 (\w_t, \teachers)$.

\vspace{0.2em}

\begin{lemma}[Linear regularity of Sequential Max-Margin]
\label{lem:regularity}
At the $t$\nth iteration, the distance to the offline feasible set is tied to the distance to the farthest feasible set of any \emph{specific} task.
Specifically, it holds that $\forall t\!\in\!\naturals$,
\begin{align*}
    \dist^2\!\prn{\w_{t}, \teachers}
    \le
    4\,\tnorm{\teacher}^2 R^2 
    \max_{m\in\cnt{T}} \dist^2 
    \bigprn{\w_{t}, \feasible_m}
    \,
    .
\end{align*}
The proofs for this section are given in \appref{app:repetitions_proofs}.
\end{lemma}

So far we assumed that at iteration $t$,
the learner solves a task, \ie a dataset $\dataset_{t}$ (out of $T$ possible datasets),
corresponding to a projection $\mP_{t}$.
To facilitate our next results for cases where tasks recur, we now define task ordering functions.
\begin{definition}[Task ordering]
\label{def:ordering}
    A task ordering is a function
    $$\tau:~\naturals\to\cnt{T}$$
    that maps an iteration to a learning task.
\end{definition}

\subsubsection{Cyclic Ordering}
\label{sec:cyclic}

Mathematically, 
when analyzing successive projections (like in our \procref{proc:adaptive}), 
it is common to first study a cyclic setting, where the projections form a clearer (cyclic) operator.
This setting has been at the center of focus in many theoretical papers (\eg \citet{agmon1954motzkin,halperin1962product,deutsch2006ratePOCS_I,borwein2014analysis,evron2022catastrophic}).

Practically, in continual learning, cyclic orderings are indeed less flexible.
However, we believe that they \emph{do} emerge naturally in real-world scenarios.
For instance, virtual assistants support the recurring daily routines of their customers
\linebreak
and should be able to continue learning from new experiences.

\begin{definition}[Cyclic task ordering]
\label{def:cyclic_ordering}
A cyclic ordering over $T$ tasks is defined as 
$$\tau(t)
\triangleq
1+\prn{\prn{t-1}\bmod T},
\enskip~
\forall t\in\naturals
\,,
$$
and induces a cyclic operator,
in the sense that
$$\w_{nT} 
=
\prn{\mP_{\tau(nT)}\circ\smallcdots\circ\mP_{\tau(1)}}
\!(\w_0)
=
(\mP_T\circ\smallcdots\circ\mP_1)^n
(\w_0).$$
\end{definition}

{We illustrate such orderings in the following \figref{fig:cyclic}.}

\begin{figure}[ht!]
    \vspace{.11cm}
    \centering
    \frame{\includegraphics[width=.69\linewidth]{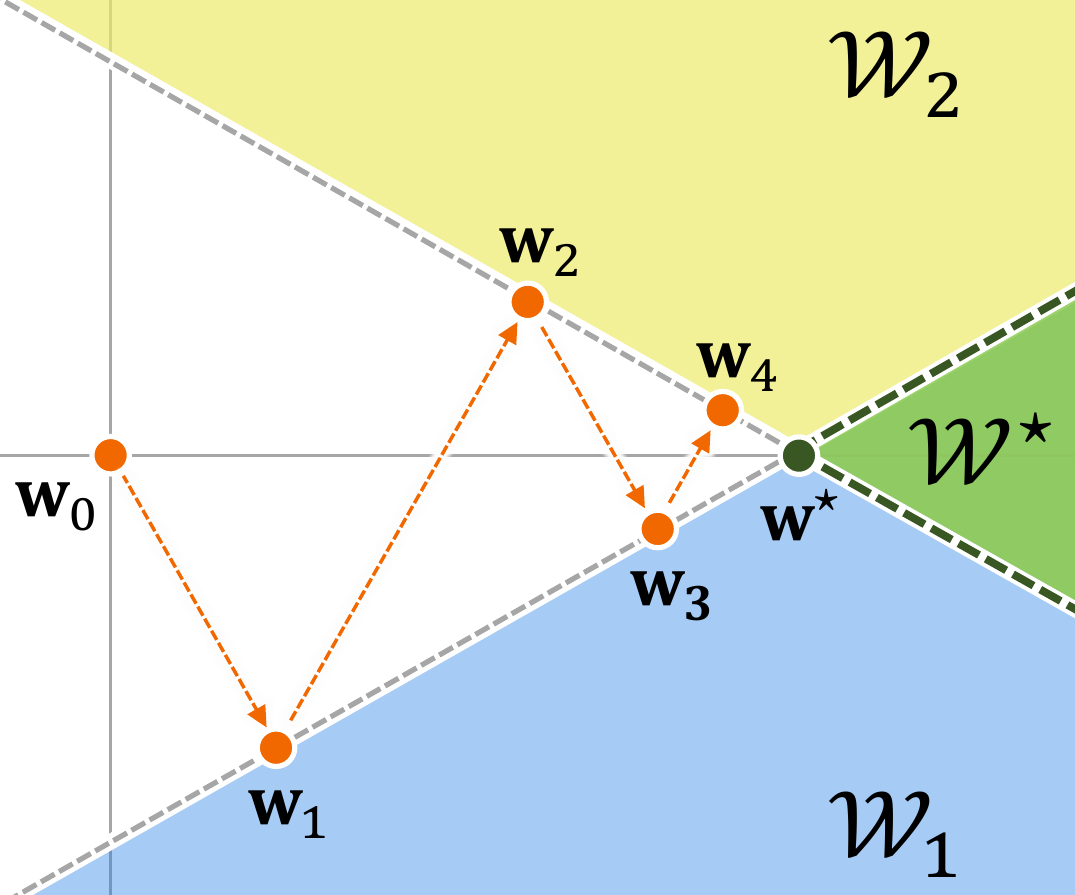}}
    \vspace{-.2cm}
    \caption{A cyclic setting with two tasks.
    Each time we solve a task, we project the previous iterate onto its corresponding convex~set.
    }
    \label{fig:cyclic}
    \vspace{-.1cm}
\end{figure}

\pagebreak

Using the linear regularity from \lemref{lem:regularity},
we prove linear convergence results for cyclic orderings.
Our proofs use tools and results from existing POCS literature.

\begin{lemma}[Limit guarantees for cyclic orderings]
\label{lem:cyclic_lim} 
Under~a~cyclic ordering
(and the separability assumption~\ref{asm:separability}), 
the iterates converge to a 2-optimal 
${\wlim\in\teachers}$.
That is,
$$
\lim_{k\to \infty} 
\dist(\w_k, \teachers)=0,
~~~
\norm{\teacher} \le \norm{\w_\infty} \le 2\norm{\teacher}
\,.
$$
\vspace{-.4cm}
\end{lemma}
The proofs for this section are given in \appref{app:cyclic_proofs}.

\bigskip

\begin{theorem}[Linear rates for cyclic orderings]
\label{thm:2_tasks_cyclic}
    For ${T\ge2}$
    jointly-separable tasks
    learned cyclically,
    after 
    ${k\!=\!nT}$ iterations ($n$ cycles),
    our quantities of interest 
    (\defref{def:quantities})
    converge linearly as
    \vspace{-.1cm}
    \begin{align*}
    \underbrace{
        \max_{m\in\cnt{T}}
        F_m (\w_k)}_{
        \substack{
        \text{Maximum forgetting}
        }
    }
    &\le
    \underbrace{
        \max_{m\in\cnt{T}}
        \dist^{2\!}\prn{\w_k, \feasible_{m}}
    }_{
        \substack{
        \text{Max.~dist.~to any feasible set}
        }
    }
    R^2
    \le 
    \\
    &\le 
    \underbrace{
    \dist^{2\!}\prn{\w_k, \teachers}
    }_{
    \substack{
    \text{Dist.~to}
    \\
    \text{offline feasible set}
    }
    }
    R^2
    \le
    g(k)
    \norm{\teacher}^2 \!R^2
    %
    \,,
    \end{align*}
    
    \vspace{-.35cm}
    where
    $$
    g(k)
    \triangleq
    \begin{cases}
    \enskip\,
    \exp
    \!\prn{-\tfrac{k}{4\tnorm{\teacher}^2 R^2}} & T=2
    \\
    4
    \exp\!\prn{-\tfrac{k}{16T^2\tnorm{\teacher}^2 R^2}} 
    & T\ge 3
    \end{cases}
    $$
    \linebreak
    and
    $\norm{\teacher}^2 \!R^2\!\ge\! 1$ is the problem complexity (Rem.~\ref{rmk:complexity}).
\end{theorem}

 \vspace{-0.45em}

\pagebreak

\paragraph{Detour: Universal rates for \emph{general} cyclic Projections onto Convex Sets (POCS) settings.}
We take a brief detour to derive universal bounds for general POCS settings, where problems do not necessarily hold regularity conditions like the ones in our \lemref{lem:regularity}.
While cyclic POCS settings have been studied for decades
\citep{agmon1954motzkin,deutsch2006ratePOCS_I,deutsch2006ratePOCS_II,deutsch2008ratePOCS_III},
most previous works either focused on settings with regularity assumptions or yielded problem-dependent rates that can be arbitrarily bad.

Our result below extends the universal results of a recent work that focused on closed \emph{subspaces} only 
\citep{evron2022catastrophic}; and is also related to a recent work that considered closed \emph{convex sets}  as well,
but dismissed the number of sets $T$ as $\bigO\prn{1}$
\citep{reich2022polynomial}.

\begin{proposition}[Universal rates for general cyclic POCS]
\label{prop:cyclic_universal}
    Let $\feasible_1, \dots, \feasible_{T}$ be closed convex sets with nonempty intersection $\teachers$.
    Let
    $\w_{k}=(\mP_T\circ\cdots\circ\mP_1)^{n}
    (\w_0)$
    be the iterate after $k\!=\!nT$ iterations ($n$ cycles)
    of cyclic projections onto these convex sets.
    Then, the maximal distance to any (specific) convex set,
    is upper bounded \emph{universally} as,
    \begin{align*}
        \text{For $T=2$: }~~
        &
        \max_{m\in\cnt{T}} 
        \dist^{2\!}\prn{\w_k, \feasible_{m}}
        \le
        \frac{1}{k+1}\,
        \dist^{2\!}\prn{\w_{0}, \teachers}
        \,,
        \\
        \text{For $T\ge3$: }~~
        &
        \max_{m\in\cnt{T}} 
        \dist^{2\!}\prn{\w_k, \feasible_{m}}
        \le
        \frac{2T^2}{\sqrt{k}}\,
        \dist^{2\!}\prn{\w_{0}, \teachers}
        \,.
    \end{align*}
\end{proposition}

\vspace{-0.5em}

Clearly, the above yields
universal rates for the forgetting as well
(see \lemref{lem:euclidean_to_hinge}),
but these are essentially worse than the linear rates we got (in the presence of regularity).

\subsubsection{Random Ordering}
\label{sec:random_ordering}
In this section, we consider a uniform i.i.d.~task ordering
over a set of $T$ tasks.
Random orderings are considered more realistic than cyclic ones
(\eg driverless taxis will likely encounter recurring routes/environments randomly, rather than in a certain cycle). 
Mathematically, they also require other analytical tools and often offer different convergence guarantees.
Much like cyclic orderings, random orderings have been studied in many related areas
(\eg \citet{nedic2010randomPOCS,needell2014paved,evron2022catastrophic}).

\pagebreak

\begin{definition}[Random task ordering]
\label{def:random_ordering}
A random (uniform) ordering over $T$ tasks is defined as 
$$\Pr\prn{\tau(t)\!=\!m}
\!=\!
\Pr\prn{\tau(t')\!=\!m'}\!,
\enskip
\forall t,t'\!\in\!\naturals\!,\,m,m'\!\in\!\cnt{T}
\!.
$$
\end{definition}

\begin{theorem}[Linear rates for random orderings]
\label{thm:random_rates}
 For $T$ jointly-separable tasks
learned in a random ordering,
our quantities of interest 
(\defref{def:quantities})
converge linearly as
\begin{flalign*}
\mathrlap{
\expectation_{\tau}\!
\big[
\underbrace{
    \max_{m\in\cnt{T}}
    F_m (\w_k)}_{
    \substack{
    \text{Maximum forgetting}
    }
}
\big]
\le
\expectation_{\tau}\!
\big[
\underbrace{
    \max_{m\in\cnt{T}}
    \dist^{2\!}\prn{\w_k, \feasible_{m}}
}_{
    \substack{
    \text{Max.~dist.~to any feasible set}
    }
}
\big]
R^2
\le 
}
\\
&&
\mathllap{
\le 
\expectation_{\tau}\!
\big[
\underbrace{
\dist^{2\!}\prn{\w_k, \teachers}
}_{
\substack{
\text{Dist.~to}
\\
\text{offline feasible set}
}
}
\big]
R^2
\le
\exp\!
\prn{
\!
-\tfrac{k}{4T\tnorm{\teacher}^2 R^2}
\!
}
\norm{\teacher}^2 \!R^2
%
}.
\end{flalign*}
where $\norm{\teacher}^2 \!R^2\!\ge\! 1$ is the problem complexity (Rem.~\ref{rmk:complexity}).
\end{theorem}

To show this theorem, we prove a property of \emph{our} convex sets
(\ie polyhedral cones; see \figref{fig:one_task_feasibility}) 
and apply
a result from the (random) POCS literature \citep{nedic2010randomPOCS}.
\linebreak
The proofs for this section are given in \appref{app:random_proofs}.

\begin{lemma}[Limit guarantees of random orderings]
\label{lem:random_limit}
Under a random ordering
 (and assumption~\ref{asm:separability}), 
 the iterates converge \emph{almost surely} to ${\teachers}$,
such that 
$$
\lim_{k\to \infty} 
\dist(\w_k, \teachers)=0,
~~~
\norm{\teacher} \le \norm{\w_\infty} \le 2\norm{\teacher}
\,.
$$
\end{lemma}

\begin{remark}[Beyond uniform distributions]
  The result from \citet{nedic2010randomPOCS}   that we used to prove \thmref{thm:random_rates}, 
  is applicable to any \iid~distribution
  with a nonzero
  \linebreak
  $p_{\min}
  \triangleq 
  \min_{m\in\cnt{T}} 
  \Pr\prn{\tau\prn{{\cdot}}\!=\!m}>0$.
  In such cases, the rate from \thmref{thm:random_rates} changes to
    $\exp\!
    \prn{
    \!
    -\tfrac{p_{\min} k}{4\tnorm{\teacher}^2 R^2}
    \!
    }$.
\end{remark}

\subsubsection{Average Iterate Analysis}
\label{sec:avg_iterate}
So far, we analyzed the convergence of
the last iterate $\w_{k}$.
\linebreak
Next, we analyze the average iterate 
\hfill
$
\m@th\displaystyle
    \avgitr_{k}\triangleq 
    \frac{1}{k}
    \tsum_{t=1}^{k} \w_t
$.
\linebreak
Such analysis often allows for stronger bounds
(\eg in SGD \citep{shamir2013stochastic}, Kaczmarz methods \citep{morshed2022sampling},
and continual regression \citep{evron2022catastrophic}).

\begin{proposition}[Universal rates for the average iterate]
    \label{prop:avearge_iterate}
    After $n$ cycles 
    under a cyclic ordering 
    ($k=nT$) we have
    $$
    \underbrace{
        \max_{m\in\cnt{T}}
        F_m (\avgitr_k)}_{
        \substack{
        \text{Maximum forgetting}
        }
    }
    \le
    \underbrace{
        \max_{m\in\cnt{T}}
        \dist^{2\!}\prn{\avgitr_k, \feasible_{m}}
    }_{
        \substack{
        \text{Max.~dist.~to any feasible set}
        }
    }
    \!
    R^2
    \le 
    \frac{T^2}{k}
    \!
    \norm{\w^{\star}}^2 \!R^2$$
    and after $k$ iterations under a random ordering we have
\begin{align*}
\expectation_{\tau}\!
    \Big[\!
    \underbrace{
        {\tfrac{1}{T}}\!\!
        \sum_{m=1}^T
        \!\!
        F_m (\avgitr_k)}_{
        \substack{
        \text{Average forgetting}
        }
    }
    \!\Big]
    \!\le
    \expectation_{\tau}\!
    \Big[\!\!\!
    \underbrace{
        {\tfrac{1}{T}}\!\!
        \sum_{m=1}^T
        \!\!
        \dist^{2\!}\prn{\avgitr_k, \feasible_{m}}
    }_{
        \substack{
        \text{Avg.~distance~to feasible sets}
        }
    }
    \!\!\!\Big]
    R^2
    \le
    \tfrac{\norm{\w^{\star}}^2 \!R^2}{k}
\end{align*}
(implying a $\frac{T}{k}\norm{\w^{\star}}^2 \!R^2$ bound on the expected \emph{maximum} forgetting and \emph{maximum} distance to any feasible set).
\end{proposition}

The proof is given in \appref{app:averaging}.

\begin{figure*}[t!]
    \centering
    \begin{subfigure}[t]{0.30\textwidth}
        \centering       
        
        \vspace{-4.3cm}
        \frame{\includegraphics[width=.95\linewidth]{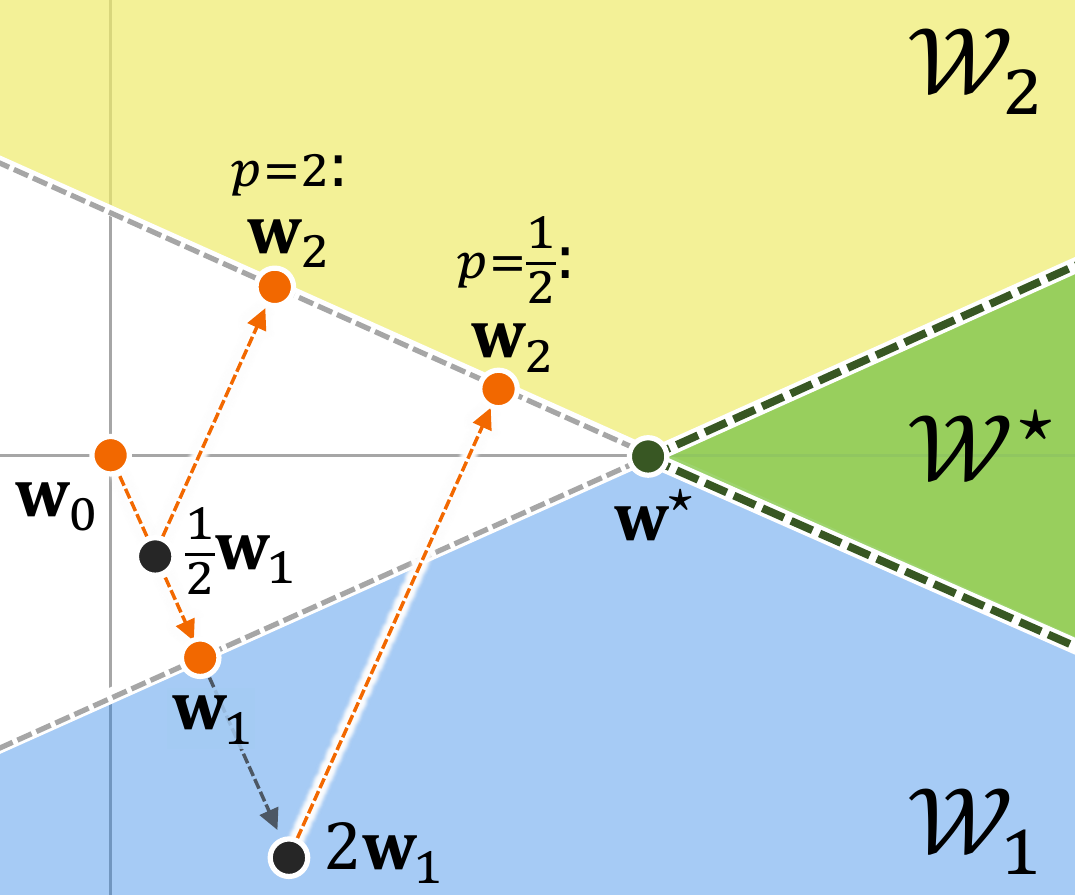}}
        \vspace{-.2em}
        
        \vspace{.25cm}
        \caption{
        Two iterations of the double exponential scheduler with two values $p=\nicefrac{1}{2}, 2$.
        \linebreak
        $p>1$ shrinks $\w_1$ before projecting it onto $\feasible_2$,
        while
        $p<1$ inflates it.
        }
        \label{fig:general_scheduling}
    \end{subfigure}%
    \hfill
    \begin{subfigure}[t]{0.32\textwidth}
        \centering            
        {\includegraphics[width=1\linewidth]{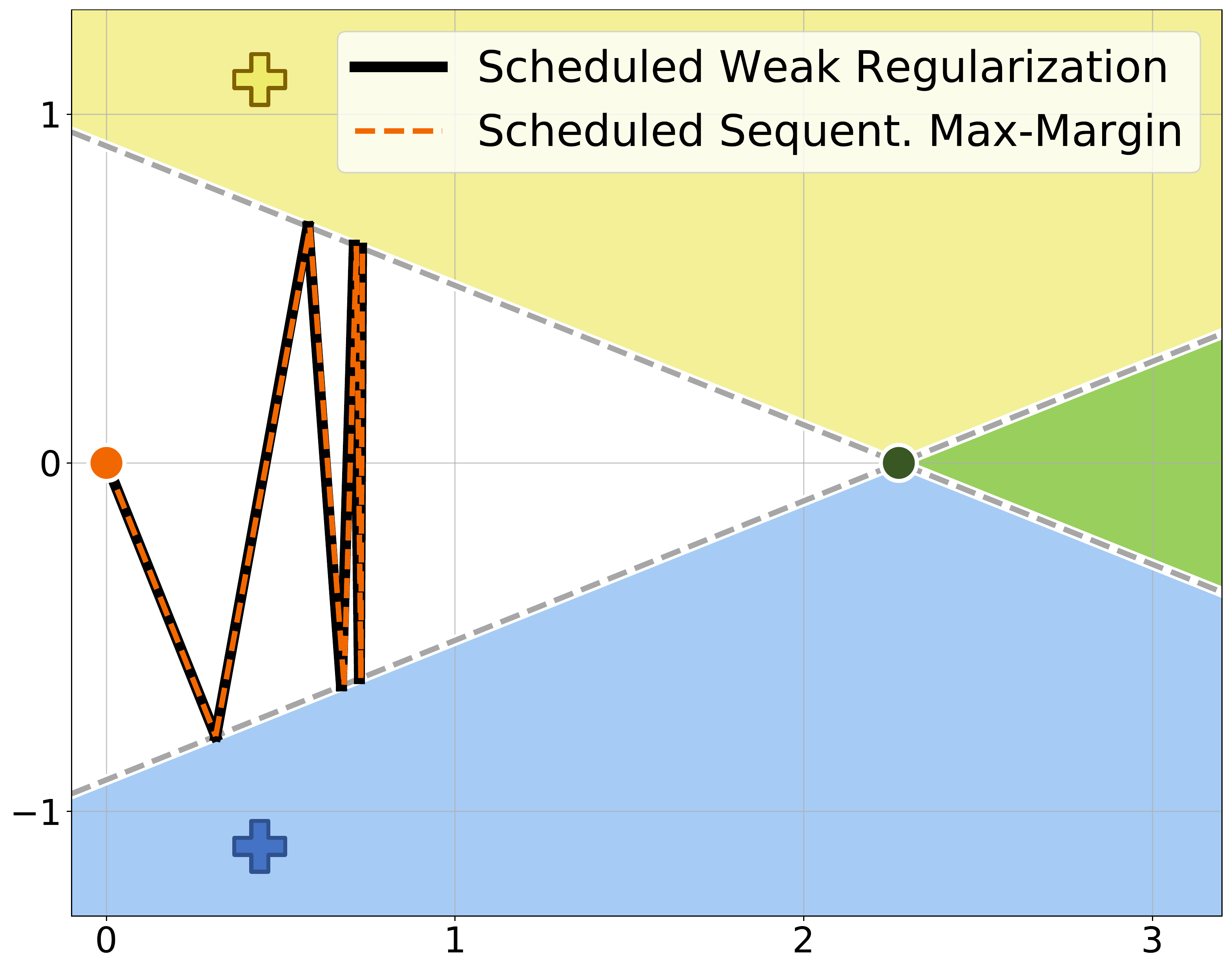}}
        \vspace{-1.4em}
        \caption{
        $p>1 \Longrightarrow $
        Might not converge to $\teachers$.
        \linebreak
        Here, ${p\!=\!2}$, ${\lambda_1\!=\!\lambda\!=\!{10}^{-32}}$, 
        ${\lambda_6\!
        \!=\!{10}^{-1024}}$. 
        Black~lines indicate changes in the \emph{directions} of the weakly-regularized solutions (their actual scale is much larger).
        }
        \label{fig:contraction_schedule}
    \end{subfigure}%
    \hfill
    \begin{subfigure}[t]{0.32\textwidth}
        \centering
        {\includegraphics[width=1\linewidth]{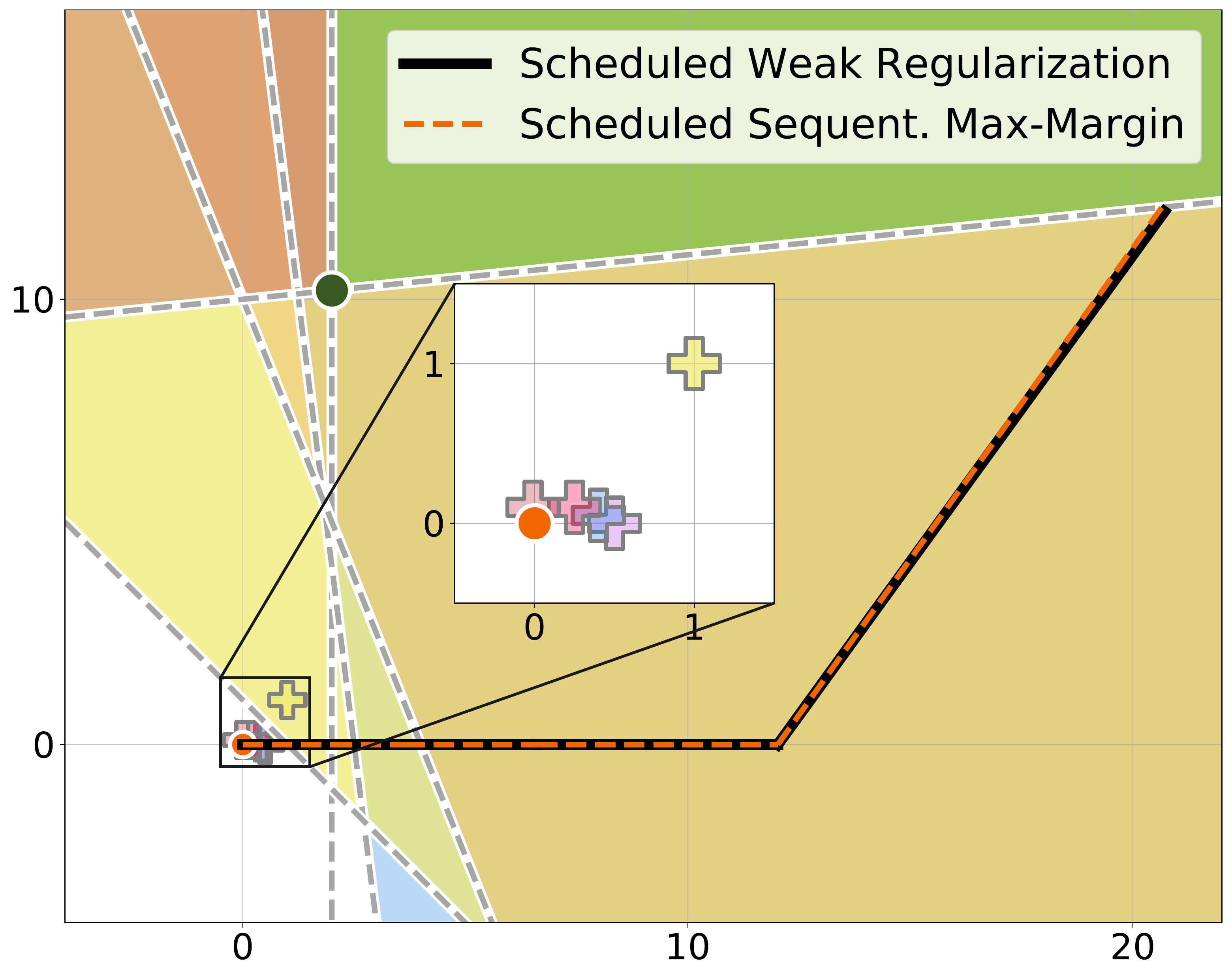}}
        \vspace{-1.4em}
        \caption{
        $p\!>\!1 \Longrightarrow $
        Possibly $\norm{\w_t}
        > 2\norm{\teacher}$
        (the guarantee from Thm.~\ref{thm:optimality_guarantees} does not hold).
        Here, $p=0.55$,
        ${\lambda_1=\lambda={10}^{-4096}}$
        and the rest
        $\lambda_2,\lambda_3,\lambda_4,\lambda_5$
        are
        $
        {10}^{-2253},
        {10}^{-1239},
        {10}^{-681},
        {10}^{-375}
        $.
        }
        \label{fig:inflation_schedule}
    \end{subfigure}%
    \vspace{-0.35em}
    \caption{
    \label{fig:extensions}
    Illustrations and experiments for
    the double exponential regularization scheduling (\eqref{eq:exponential_scheduler}
        in \secref{sec:scheduling}).
    }
    \vspace{-0.7em}
\end{figure*}

\pagebreak

\section{Extensions}
\label{sec:extensions}

\subsection{Regularization Strength Scheduling}
\label{sec:scheduling}

Up to this point, 
we assumed that the regularization strengths in  \procref{proc:regularized} are constant,
\ie $\lambda_1 = \smalldots = \lambda_{k} \triangleq \lambda$.
\linebreak
Alternatively, these strengths 
can \emph{vary}, as done in practice~in stationary settings \citep{lewkowycz2020regularizationScheduling}.
Related work in continual learning \citep{Mirzadeh2020regimes}
used a learning-rate decay scheme,
which can be seen as a form of varying regularization.
We now aim to understand the bias and implications of varying regularization strengths.

Given $\lambda>0$,
we parameterize the regularization strengths as 
$\lambda_{t}\!\triangleq\!\lambda_{t}(\lambda)\!>\!0$
for arbitrary functions
$\lambda_{t}\!:\reals_{>0}\!\to\!\reals_{>0}$
holding that
$\lim_{\lambda\to 0}\lambda_{t}(\lambda) \!=\! 0$
and
$
\lim_{\lambda\to 0}
\tfrac{\ln\lambda_{t-1}}{\ln\lambda_{t}}\!<\!\infty$ is well-defined
$\forall t\!\in\!\cnt{k}$.
We show that the weakly-regularized \procref{proc:regularized} 
(with $\lambda_1,\smalldots,\lambda_k$)
converges to the scheme below.

\begin{algorithm}[h!]
   \caption{Scheduled Sequential Max-Margin}
    \label{proc:scheduled_adaptive}
\begin{algorithmic}
   \STATE {\algmargin\bfseries Initialization:} 
   $\w_0 = \0_{\dim}$
   \STATE {\algmargin{\bfseries Iterative update $\forall t\!\in\![k] 
   $:}
   \hfill
   $
   \m@th\displaystyle
   \w_{t}
    \!=\!
    \mP_{t}
    {
    \Bigprn{
        \bigprn{
        \lim_{\lambda\to 0}
        \!
        \tfrac{\ln\lambda_{t-1}}{\ln\lambda_{t}}
        }
        \w_{t-1}
        }
    }$
   } 
    \vspace{-.2em}
\end{algorithmic}
\end{algorithm}

\begin{theorem}[Weakly-regularized models with scheduling]
\label{thm:weak_schedule}
For almost all {separable} datasets, %
when scheduling the regularization strength as described above, 
it holds that
\vspace{-0.2em}
$$
\lim_{\lambda\to0}
\frac{\w_t^{(\lambda)}}{\tnorm{\w_t^{(\lambda)}}}
=
\frac{\w_t}{\tnorm{\w_t}},
~~~
{\forall t\!\in\!\cnt{k}}
\,.
$$
\vspace{-0.8em}
\end{theorem}
The proofs for this section are given in \appref{app:extensions} \& \ref{app:scheduling}.

\medskip

Specifically, using a \textbf{double exponential scheduling} rule,
\begin{align}
    \label{eq:exponential_scheduler}
    \lambda_{t}=\lambda_{t-1}^{p}
    \,
    \tprn{
        =\lambda_{1}^{p^{t-1}}
        \!
        =\lambda^{p^{t-1}}
    }\,,\,\, 
    \text{ for some $p>0$ ,}
\end{align}
the update rule in Scheme~\ref{proc:scheduled_adaptive} becomes
%
$
{\w_{t}
\!=\!
\mP_{t}\bigprn{
    \tfrac{1}{p}\w_{t-1}
    }
}$.
\linebreak
We illustrate this in \figref{fig:general_scheduling}.

In this case, in contrast to the $p\!=\!1$ case, even when tasks recur we may not converge near the min-norm solution or the feasibility set, as demonstrated in  the next examples.

\vspace{0.1em}

\begin{example}[$p>1 \Longrightarrow $
Possibly $\wlim\notin \teachers$]
In~\figref{fig:contraction_schedule},
we cycle between two tasks using Schemes~\ref{proc:regularized}~and~\ref{proc:scheduled_adaptive} with $p=2$. 
The strengths \emph{decay} like
$\lambda, \lambda^2, \lambda^4, \lambda^8$, and so on.
Note that both schemes agree. 
Unlike the constant $\lambda$ case (Thm.~\ref{thm:2_tasks_cyclic}), 
we do \emph{not} converge to the offline set $\teachers$.
\end{example}
More generally, we prove that when $p>1$, 
the limit distance from the offline feasible set $\teachers$ can be arbitrarily bad.

\begin{proposition}
\label{prop:large_p_no_convergence}
There exists a construction of two jointly-separable tasks with $R=1$,
in which the iterates of the cyclic ordering do not converge to $\teachers$, for any $p>1$.
Specifically, for any ${p> 1,\norm{\teacher}>1}$,
it holds that
\begin{align*}
    \lim_{k\to\infty}
    \frac{\dist(\w_{k},\teachers)}{\tnorm{\teacher}}
    =
    \frac{\norm{\teacher}^2(p-1)}{
        2+\norm{\teacher}^2(p-1)
    }
    \sqrt{1-\tfrac{1}{\norm{\teacher}^2}}
    \,.
\end{align*}
\end{proposition}

\begin{example}[$p<1 \Longrightarrow$
No optimality guarantees on $\w_t$]
In \figref{fig:inflation_schedule},
we run both Schemes~\ref{proc:regularized}~and~\ref{proc:scheduled_adaptive} with $p=0.55$ for one ``pass'' over the tasks. 
The strengths \emph{increase} like
$\lambda, \lambda^{0.55}, \lambda^{0.303}, \lambda^{0.166}$, and so on
(however, our analysis still requires that $\lambda\to 0$; see Remark~\ref{rmk:validity}).
Note that the iterates of both schemes agree. 
Unlike the guarantees for the constant $\lambda$ case (\thmref{thm:optimality_guarantees}), 
here 
$\norm{\w_5} > 2\norm{\teacher}$.
\end{example}

Roughly speaking, adversarial placements of examples (\ie feasible sets)
can make the iterates $(\w_{t})$ grow like
$\bigO\prn{p^{-t}}$,
even in an almost orthogonal direction to the min-norm $\teacher$.

\vspace{0.2em}

\begin{remark}
We analyzed the double exponential rule (\ref{eq:exponential_scheduler}) 
\linebreak
since ``milder'' rules
(\eg $\lambda_t \!=\! c^t\lambda$ for $c\!=\!\bigO(1)$)
do not influence \procref{proc:scheduled_adaptive}, 
thus not affecting \procref{proc:regularized} 
when ${\lambda\!\to\!0}$.
Even a rule like ${\lambda_t \!=\! \lambda^{t}}$ implies
${\lim_{\lambda\to 0}
\!
\tfrac{\ln\lambda_{t-1}}{\ln\lambda_{t}}
\!=\!1\!-\!\frac{1}{t}}
$ and becomes insignificant as $t$ increases.
\end{remark}

\subsection{Weighted Regularization}
\label{sec:weighted}

So far, we analyzed \emph{unweighted} regularizers  in \procref{proc:regularized}.
Most regularization methods for continual learning employ \emph{weighted} norms 
(\eg \citet{kirkpatrick2017overcomingEWC,zenke2017continual,aljundi2018memory}).
Now, we focus on such norms and link them to Sequential Max-Margin when $\lambda\to 0$.

Given  $\B_1,\smalldots,\B_{k}\!\succ\!\0_{\dim\times\dim}$, we define the following~scheme.

\begin{algorithm}[h!]
   \caption{Weighted Sequential Max-Margin
    \label{proc:weighted_adaptive}}
\begin{algorithmic}
   \STATE {\algmargin\bfseries Initialization:} 
   $\w_0 = \0_{\dim}$
   \STATE {\algmargin{\bfseries Iterative update for each task $t\in\cnt{k} 
   $:}} 
    \vspace{-.8em}
    \begin{align*}
    \algmargin
    \w_{t}
    =
    {\mP_{t}}\prn{\w_{t-1}}
    \triangleq
    {\argmin}_{\w}
    \, 
    & 
    \!
    \norm{\w-{\w}_{t-1}}^2_{\B_{t}}
    \\
    \suchthat
    \,
    \enskip & 
    \,y\w^\top \x \!\ge\! 1,
    \,
    \forall (\x,y)\!\in\! \dataset_{t}
    \nonumber
    \end{align*}
    \vspace{-1.8em}
\end{algorithmic}
\end{algorithm}

\begin{theorem}[Weak weighted regularization]
\label{thm:weak_weighted}
For almost all separable datasets
and bounded weighting schemes (holding ${0\!<\!\mu\!\le\! 
\sigma_{\min}(\B_t) \!\le\!
\sigma_{\max}(\B_t)
\!\le\! M\!<\!\infty}$), %
in the limit of $\lambda\!\to\!0$,
Schemes~\ref{proc:regularized}~and~\ref{proc:weighted_adaptive} coincide.
\linebreak
That is,
$\forall t\!\in\!\cnt{k}$,
it holds that
$\m@th\displaystyle
\lim_{\lambda\to0}
\frac{\w_t^{(\lambda)}}{\tnorm{\w_t^{(\lambda)}}}
=
\frac{\w_t}{\tnorm{\w_t}}
$.
\end{theorem}
\vspace{-.4em}
The proofs for this section are given in \appref{app:extensions} \& \ref{app:weighted}.

Practically, most regularization methods use weighting matrices based on Fisher information (FI) 
\citep{benzing2021unifying_regularization}.
We present a novel result showing that in continual \emph{regression}, such weighting schemes \emph{prevent} forgetting 
(thus forming an ``ideal continual learner'' as defined by \citet{peng2023ideal}).

\begin{proposition}
\label{prop:regression_fisher}
Using a Fisher-information-based weighting scheme 
of $\B_{t}=\sum_{i=1}^{t-1}\sum_{\x\in \dataset_{i}}\x\x^{\top}$,
there is no forgetting in (realizable)
continual linear\footnote{A similar guarantee also applies more broadly, for 
linear networks of any depth and non-linear networks in the NTK regime.} regression.
\end{proposition}

In contrast, below we show a simple example where such weighting schemes, 
even when using the full FI matrices,
do not prevent 
forgetting in continual linear \emph{classification}.

\begin{figure}[h!]
\vspace{-1.1em}  
    \centering
  \begin{minipage}[t!]{0.56\linewidth}
    \caption{
A 2-dimensional setting with
solely two orthogonal examples $\x_1 \!\perp\! \x_2$
in the 1\first task. \linebreak
Here, the iterate $\w_1$ holds ${\w_1^\top\x_1\!=\!\w_1^\top\x_2\!=\!1}$,
implying a FI matrix of  
${\B_2\!\propto \!
\sum_{i\in\cnt{2}} \x_i \x_i ^\top
\!=\!\I}$.
Then, FI weighting \textcolor{purple}{(dotted)} reverts to vanilla  L2 regularization \textcolor{orange}{(dashed)} which does not prevent forgetting.
\\
In contrast, the FI matrices in high dimensions are often non-invertible. 
Then, the learner can move freely in directions orthogonal to previous data ($\norm{\cdot}_{\B_t}$ is no longer a ``proper'' norm), thus avoiding forgetting.
\\
See further examples in App.~\ref{app:weighted_examples}.
    }
    \label{fig:weighted}
  \end{minipage}
  \hfill
  \begin{minipage}[t!]{0.4\linewidth}
    \vspace{1em}  
    {\includegraphics[width=.99\linewidth]{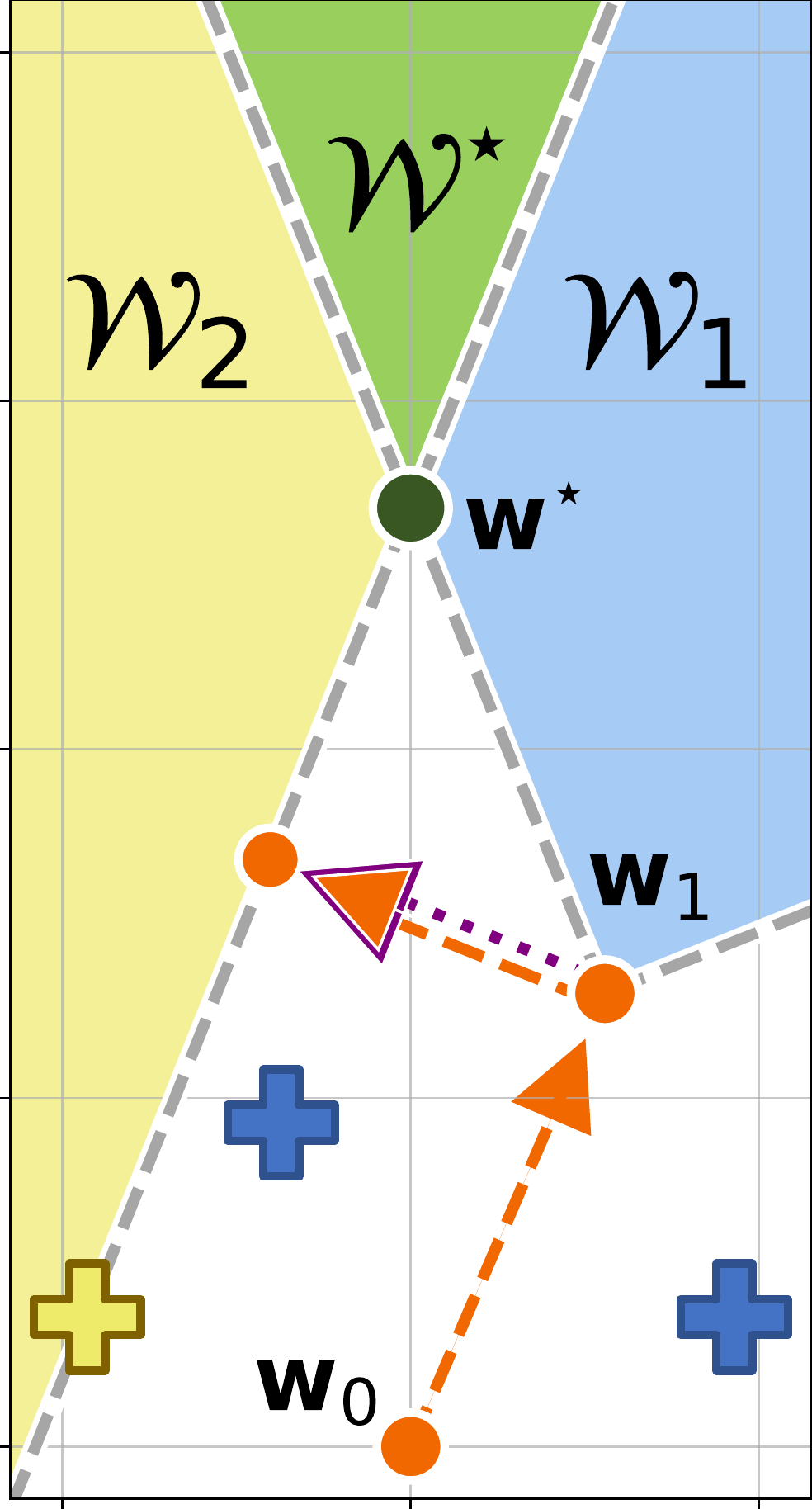}}
  \end{minipage}
\vspace{-.5em}  
\end{figure}

An interesting question arises: 
Are there general weighting schemes that can prevent forgetting in \procref{proc:weighted_adaptive}? 
We find this to be an exciting direction for future research.

\section{Related Work}

Many papers have investigated various aspects of catastrophic forgetting, including when it occurs \citep{evron2022catastrophic}, strategies to avoid it \citep{peng2023ideal}, the influence of task similarity \citep{lee2021taskSimilarity}, its impact on transferability \citep{chen2023forgetting}, and other related factors.
\linebreak
A sound understanding of forgetting can potentially  advance the continual learning field significantly.

Throughout our paper, we discussed many connections to other works from many fields.
Interestingly, our Sequential Max-Margin (SMM) scheme can be seen as a \emph{hard} variant of the Adaptive SVM algorithm \citep{yang2007adaptiveSVM},
which was previously used to \emph{practically} tackle domain adaptation and transfer learning.
(\eg in \citet{li2007regularizedAdaptation,pentina2015curriculum,tang2022IncrementalSVM}).
To the best of our knowledge, our paper is the first to highlight the connection between regularization methods for continual learning and Adaptive SVM.

For the special case where each task has one datapoint, the SMM scheme corresponds to the online Passive-Aggressive algorithm \citep{crammer2006onlinePassAgg}. 
If additionally, we consider only two tasks learned in a cyclic ordering, the regret after training task $t$ is related (but not identical) to the forgetting
(see Remark~\ref{rmk:regret}). 
Consequently, our forgetting bound in \propref{prop:cyclic_universal} matches the regret bound in their Theorem 2.

\section{Discussion and Future Work}
\subsection{Early Stopping}
\label{sec:early}

When training a \emph{single task}, it is well known that explicit regularization is related to early stopping. 
In linear classification with an exponential loss, \citet{rosset2004margin} showed that the regularization path converges 
(in the $\lambda\rightarrow 0$ limit)
to the hard-margin SVM, which is also the limit of the optimization path \cite{soudry2018journal}. 
In regression, \citet{AliKT19} connected early stopping to ridge regression.

Contrastingly, when learning a \emph{sequence of tasks}, explicit regularization and early stopping can behave differently. 
\linebreak
We first discuss a case where \emph{both} methods lead to SMM.
More generally, we demonstrate that solutions may differ.

Consider a task sequence with
$1$ sample $(\x_t, y_t)$ per task.\footnote{Equivalently, $2$ samples $(\x_t, y_t)$ and $(-\x_t, -y_t)$.
}
We minimize the exponential loss of
each task $t$ with gradient flow (GF) until 
$L_t(\w)\!\triangleq\!\exp(-y_t\x_t^{\top}\w)\!=\!\epsilon$ for
a 
fixed $\epsilon\!<\!1$, 
yielding a predictor $\w_t^{(\epsilon)}$.
Denote the 
``normalized'' predictor by ${\hat{\w}_t^{(\epsilon)}\!=\!{\w_t^{(\epsilon)}}/
\ln(\hfrac{1}{\epsilon})}$. 
Since GF stays in the data span, we have  
${\w_t^{(\epsilon)}\!=\!\w_{t-1}^{(\epsilon)}\!+\!\alpha_t y_t\x_t}$
for some $\alpha_t>0$.
The early stopping implies 
$y_t\x_t^{\top}{\w}_t^{(\epsilon)}\!=\!\ln(\hfrac{1}{\epsilon})$,
and thus 
$y_t\x_t^{\top}\hat{\w}_t^{(\epsilon)}\!=\!1$.
Hence, 
$\hat{\w}_t^{(\epsilon)}$
holds the KKT conditions of the SMM, 
\ie
${\hat{\w}_t^{(\epsilon)}\!=\!
\argmin_{\w}\!\bignorm{\w\!-\!\hat{\w}_{t-1}^{(\epsilon)}}^2 \,\suchthat\,\,y_t\x_t^{\top}\w\geq1}$
(where ${\hat{\alpha}_t \triangleq \prn{\alpha_t \,/ \ln(\hfrac{1}{\epsilon})} > 0}$ is the dual variable).\footnote{
If iteration $t$ begins and
we already have
$L_{t}({\w}_{t-1}^{(\epsilon)})\!<\!\epsilon$
and 
$y_t\x_t^{\top}\hat{\w}_{t-1}^{(\epsilon)}\!\ge\!1$,
both ES and SMM will not change their solutions.
}
In the
general case 
(\ie more than one sample per task), 
early stopping might \emph{not} agree with SMM, as depicted below.

\vspace{-.1em}

\begin{figure}[h!]
    \centering
    \begin{minipage}[t]{0.95\linewidth}
    \centering
    \includegraphics[width=.99\linewidth]{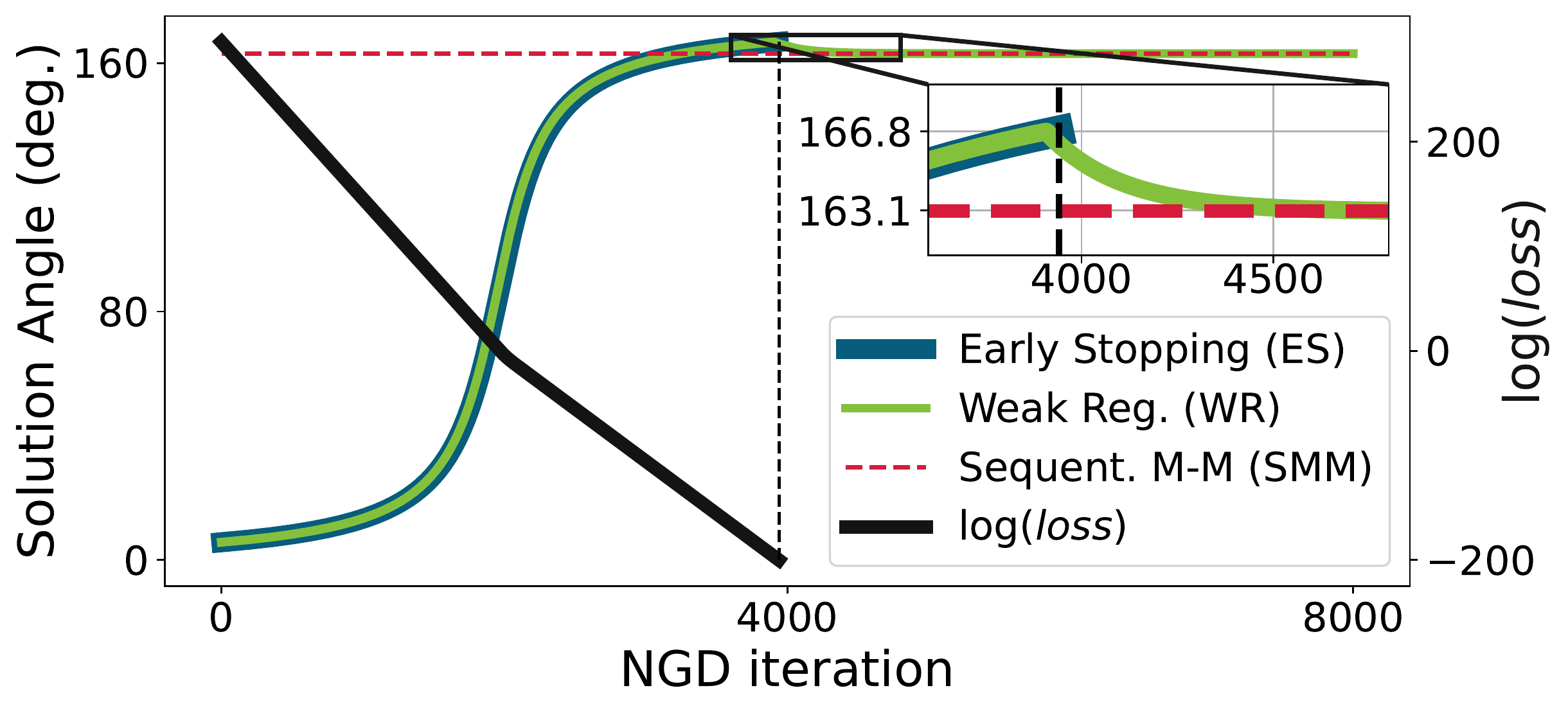}

    \vspace{-.8em}
    
    \caption{Comparing the weakly-regularized (WR) and early stopping (ES) schemes. 
    We train $2$ tasks: 
    $\dataset_1=\{(10,1)\}$, 
    $\dataset_2\!=\!\{(-10,1),(-15,0.5)\}$ 
    (where all points are labeled $+1$).
    \linebreak 
    We plot the angle of different predictors ($\angle(\w,\left[\substack{1\\0}]\right]) $) while training on the 2\nd task.
    For instance, the Sequential Max-Margin (SMM) solution for that task is $163.1^{\circ}$.
    For ES, we train both tasks with normalized GD (NGD) until their loss is ${\epsilon=e^{-200}}$. 
    Importantly, ES \emph{stops} when the loss of the 2\nd task, shown in the secondary y-axis, 
    is $e^{-200}$ and the angle is $166.8^{\circ}$.
    For WR, we set ${\lambda\!=\!e^{-200}}$ 
    and solve 
    \eqref{eq:weakly_regularized_problem} by running NGD to convergence. 
    For the 2\nd task, we initialize the NGD by $\w_1^{(\lambda)}$. 
    \linebreak
    Unlike ES, the WR solution \emph{does} converge to $163.1^{\circ}$ like SMM, as guaranteed by \thmref{thm:weakly_regularized}.
    Further~details in App.~\ref{app:early}.
    }
    \label{fig:early_stopping_vs_weakly_reg}
\end{minipage}
\end{figure}
\vspace{-.5em}

\subsection{Future Work}
There are several interesting avenues for future work, \eg extending our results 
to non-separable data (perhaps in the spirit of \citet{yang2007adaptiveSVM}),
multiclass classification with cross-entropy loss
(similarly to Appendix~4.1 in \citet{soudry2018journal}),
or non-linear models. 
One can also try to derive forgetting bounds for weighted regularization schemes and look for optimal weighting matrices.
Another challenging but rewarding avenue is to extend our analysis to \emph{finite} regularization strengths $(\lambda_t)$.
Finally, it is interesting to understand the exact algorithmic bias for early stopping in continual learning and its relation to explicit regularization.

\vspace{-.2em}

\section*{Acknowledgements}
We thank Lior Alon (MIT) for the fruitful discussions.
The research of DS was Funded by the European Union {(ERC, A-B-C-Deep, 101039436)}. Views and opinions expressed are however those of the author only and do not necessarily reflect those of the European Union or the European Research Council Executive Agency (ERCEA). Neither the European Union nor the granting authority can be held responsible for them. DS also acknowledges the support of Schmidt Career Advancement Chair in AI.
NS was partially supported by the Simons Foundation and NSF-IIS/CCF awards.

\clearpage

\bibliography{99_biblio}
\bibliographystyle{icml2023}

\clearpage

\appendix

\onecolumn

\section{Comparison: Continual Linear Classification vs. Continual Linear Regression}
\label{app:comparison}

\begin{table*}[ht!]
\vskip 0.1in
\begin{center}
\begin{small}
\begin{tabular}{|l|c|c|}
\toprule
\\[-7pt]
& 
\begin{sc}
\begin{tabular}{@{}c@{}}
Continual Linear Classification
\\
\textsc{(Ours)}
\end{tabular}
\end{sc}
&
\begin{sc}
\begin{tabular}{@{}c@{}}
Continual Linear Regression
\\
\textsc{\citep{evron2022catastrophic}}
\end{tabular}
\end{sc}
\\[8pt]
\midrule
\\[-5pt]
\begin{sc}
\begin{tabular}{@{}l@{}}
Fundamental mapping
\\
(proven algorithmic bias)
\end{tabular}
\end{sc} & 
\begin{tabular}{@{}c@{}}
   Continual linear classification with
   \\
   weakly-regularized exponential losses 
   \\
   converges to 
   \\
   a Sequential Max-Margin scheme
   \\
   (\thmref{thm:weakly_regularized})
\end{tabular}
&
\begin{tabular}{@{}c@{}}
   Continual linear regression 
   \\with unregularized losses, 
   \\solved sequentially with vanilla (S)GD, 
   \\implicitly performs successive projections
   \\
   (their Eq. (5))
\end{tabular}
\\[25pt]
\midrule
\\[-5pt]
\begin{sc}
Feasible sets
\end{sc} & 
\begin{tabular}{@{}c@{}}
   Closed, convex, and affine 
   \\
   polyhedral cones
   \\
   (Assumption~\ref{asm:separability})
\end{tabular}
&
\begin{tabular}{@{}c@{}}
   Closed affine subspaces
   \\
   (their Eq. (4))
\end{tabular}
\\[15pt]
\midrule
\\[-5pt]
\begin{sc}
Projection operators
\end{sc} & 
\begin{tabular}{@{}c@{}}
       Projection onto Convex Sets (POCS);
       \\
       Challenging to analyze
       \\
       (\eqref{eq:adaptive_dynamics})
 \end{tabular}
&
\begin{tabular}{@{}c@{}}
       Affine projections;
       \\
       Easy to analyze
       \\
       (their Eq. (5))
 \end{tabular}
\\[15pt]
\midrule
\\[-5pt]
\textsc{Optimality guarantees}
&
    \begin{tabular}{@{}c@{}}
    $\norm{\w_t} \le 2\norm{\teacher}$
    \\
    (\lemref{thm:optimality_guarantees})
    \end{tabular}
&
    \begin{tabular}{@{}c@{}}
    $\norm{\w_t} \le \norm{\teacher}$
    \\
    $\w_t\in\teachers\Longrightarrow\w_t=\teacher$
    \end{tabular}
\\[12pt]
\midrule
\\[-5pt]
\textsc{Adversarial construction}
    &
    \multicolumn{2}{c|}{
        \begin{tabular}{@{}c@{}}
       Arbitrarily bad quantities 
       \\
       $\m@th\displaystyle
       1 =
        \frac{1}{\norm{\teacher}^2 R^2}
       \lim_{T\to\infty}\max_{m\in\cnt{T}}F_m(\w_t)
       \le
       \frac{1}{\norm{\teacher}^2}
       \lim_{T\to\infty}\dist^2\prn{\w_t,\teachers}\le 1$
       \\
       (our \secref{sec:adversarial},
       their Section~4.2)
        \end{tabular}
    }
\\[20pt]
\midrule
\\[-5pt]
\textsc{Fisher-based regularization}
    &
    \begin{tabular}{@{}c@{}}
       Does not prevent forgetting
       \\
       (\figref{fig:weighted}, \appref{app:weighted_examples})
     \end{tabular}
    &
    \begin{tabular}{@{}c@{}}
       Prevents forgetting 
       \\
       (our \propref{prop:regression_fisher})
     \end{tabular}
\\[10pt]
\midrule
\\[-5pt]
\begin{sc}
\begin{tabular}{@{}l@{}}
       Linear regularity
 \end{tabular}
\end{sc}
    &
    \begin{tabular}{@{}c@{}}
       Holds
       \\
       $\m@th\displaystyle
       \dist^2 \!\prn{\w_t,\teachers} \le 
       4\!\norm{\w}^2 \!R^2 \!\max_{m\in\cnt{T}}
       \!
       \dist^2\!\prn{\w_t,\feasible_m}$
       \\
       (\lemref{lem:regularity})
     \end{tabular}
    &
    \begin{tabular}{@{}c@{}}
       Does not hold
     \end{tabular}
\\[20pt]
\midrule
\\[-.5em]
\begin{sc}
\begin{tabular}{@{}l@{}}
   Nontrivial bounds on 
   \\
   distance to the feasible set
   \\
   $\dist^2\prn{\w_t,\teachers}$
   depending on the
   \\
   problem complexity $\norm{\teacher}^2\! R^2$
    \\[10pt]
    \midrule
    \\[-.5em]
   Nontrivial bounds on 
   \\
   forgetting depending on the
   \\
   problem complexity $\norm{\teacher}^2\! R^2$
 \end{tabular}
\end{sc}
    &
    \begin{tabular}{@{}c@{}}
       Possible due to regularity;
       \\
       Exponential decrease (linear convergence)
       \\
       (Table~\ref{tbl:summary}, Theorems~\ref{thm:2_tasks_cyclic}~and~\ref{thm:random_rates})
     \end{tabular}
    &
    \begin{tabular}{@{}c@{}}
       \\
       Impossible
       \\
       (\eg their Section~5.1)
        \\[10pt]
        \midrule
        \\[-.5em]
        Possible universal bounds 
        \\
        (see below)
        \\
     \end{tabular}
%
%
%
\\[48pt]
\midrule
\\[-5pt]
\begin{sc}
\begin{tabular}{@{}l@{}}
       Universal bounds 
       \\
       on forgetting 
       \\
       under cyclic orderings
       \\ 
       (independent of $\norm{\teacher}^2\! R^2$)
 \end{tabular}
\end{sc} 
&
\begin{tabular}{@{}c@{}}
$\m@th\displaystyle
{\frac{2T^2}{\sqrt{k}}}$ 
\\
(Proposition~\ref{prop:cyclic_universal})
\end{tabular}
&
\begin{tabular}{@{}c@{}}
$\m@th\displaystyle
\min\prn{\frac{T^2}{\sqrt{k}}, \frac{T^2 \dim}{2k}}$
\\
(their Theorem 11)
\end{tabular}
\\[20pt]
\midrule
\\[-5pt]
\begin{sc}
\begin{tabular}{@{}l@{}}
       Universal bounds 
       \\
       on forgetting 
       \\
       under random orderings
       \\ 
       (independent of $\norm{\teacher}^2\! R^2$)
 \end{tabular}
\end{sc} 
&
Unclear if possible
&
\begin{tabular}{@{}c@{}}
$\m@th\displaystyle
\frac{9 \dim}{k}$
\\
(their Theorem 13)
\end{tabular}
\\[20pt]
\bottomrule
\end{tabular}
\end{small}
\end{center}
\end{table*}

\newpage




\newpage

\section{Proofs for Algorithmic Bias in Regularization Methods
(\secref{sec:algorithmic_bias})}
\label{app:algorithmic_bias}

\begin{remark}[Simplification]
\label{rmk:simplification}
For ease of notation, 
throughout our appendices, we redefine the samples so as to subsume their labels. 
That is, we handle only \emph{positive} labels by redefining each $y\x\longmapsto\x$.
This notation is common in theoretical papers 
(\eg \citet{soudry2018journal}).
\end{remark}

\medskip

\subsection{Auxiliary Lemmas}

We first present two auxiliary results that we need for our main proof.
\begin{lemma}
\label{lem:strong-convexity}
Let $f:\reals^{\dim}\to\reals$ be a $\mu$-strongly convex objective function
(holding $\nabla^2 f(\w)\succeq \mu\I$
for some $\mu>0$).
Then, for any $\w\in\reals^{\dim}$, the (Euclidean) distance between $\w$ and $\w^{\star}$ (the minimizer of the objective $f$), can be upper bounded by:
$$\norm{\w-\w^{\star}} \le \frac{1}{\mu} \norm{\nabla f(\w)}\,.$$
\end{lemma}

\begin{proof}
This lemma is a known convex optimization result
(\eg see Lemma~10 in \citet{sidford2019ms}).
We prove it here for the sake of completeness.

We make use of the following property of strongly convex functions.
\begin{quote}
\vspace{-.9cm}
\item
\paragraph{Property~(9.8) from \citet{boyd2004convex}.}
    Let $f:\reals^{\dim}\to\reals$ be a $\mu$-strongly convex objective function
    (holding $\nabla^2 f(\w)\succeq \mu\I$
    for some $\mu>0$).
    Then, for any $\x,\y\in\reals^{\dim}$, we have
    $$
    f(\y)
    \ge
    f(\x)
    +
    \nabla f(\x)^\top
    \prn{\y-\x}
    +
    \frac{\mu}{2}
    \norm{\y-\x}^2
    \,.
    $$
\end{quote}

Using the above property, we prove our lemma, which is a slightly stronger result than Property~(9.11) in \citet{boyd2004convex}.
Our proof follows the one by \citet{sidford2019ms},
and is brought here for completeness.

First, we set $\x=\w^{\star},\y=\w$ (from our lemma).
Since $\w^{\star}$ is a (global) minimizer, its gradient is zero,
\ie $\nabla f(\w^{\star})\!=\!\0_{\dim}$.
\linebreak
We thus get that
\begin{align*}
f(\w)
&
\ge
f(\w^{\star})
+
\underbrace{\nabla f(\w^{\star})^\top}_{=\0_{\dim}}
\prn{\w-\w^{\star}}
+
\frac{\mu}{2}
\norm{\w-\w^{\star}}^2
\\
f(\w)
-
f(\w^{\star})
&
\ge
\frac{\mu}{2}
\norm{\w-\w^{\star}}^2
\,.
\end{align*}
Using the optimality of $\w^{\star}$, we get that
\begin{align*}
    \forall \w\in\reals^{\dim}:~
    f(\w^{\star})
    =
    \min_\y
    f(\y)
    &\ge
    \min_\y
    \prn{
        f(\w)
        +
        \nabla f(\w)^\top
        \prn{\y-\w}
        +
        \frac{\mu}{2}
        \norm{\y-\w}^2    
    }
    \\
    &
    =
    f(\w)
    +
    \min_\y
    \prn{
        \nabla f(\w)^\top
        \prn{\y-\w}
        +
        \frac{\mu}{2}
        \norm{\y-\w}^2    
    }\,.
\end{align*}
Then, by plugging in the minimizer of the right term, we get
\begin{align*}
    \forall \w\in\reals^{\dim}:~
    f(\w^{\star})
    &
    \ge
    f(\w)
    +
    \nabla f(\w)^\top
    \Bigprn{\w-\frac{1}{\mu}\nabla f(\w)-\w}
    +
    \frac{\mu}{2}
    \Bignorm{\w-\frac{1}{\mu}\nabla f(\w)-\w}^2   
    \\
    &
    =
    f(\w)
    -\frac{1}{\mu}
    \nabla f(\w)^\top
    \nabla f(\w)
    +
    \frac{\mu}{2}
    \Bignorm{\frac{1}{\mu}\nabla f(\w)}^2
    =
    f(\w)
    -\frac{1}{\mu}
    \bignorm{\nabla f(\w)}^2  
    +
    \frac{1}{2\mu}
    \bignorm{\nabla f(\w)}^2   
    \\
    &
    =
    f(\w)
    -
    \frac{1}{2\mu}
    \bignorm{\nabla f(\w)}^2   
    \\
    \frac{1}{2\mu}
    \bignorm{\nabla f(\w)}^2   
    &
    \ge
    f(\w)
    -
    f(\w^{\star})\,.
\end{align*}
Overall, we showed that
$\forall\w\in\reals^{\dim}$, it holds that
\begin{align*}
    \frac{1}{2\mu}
    \bignorm{\nabla f(\w)}^2   
    \ge
    f(\w)
    -
    f(\w^{\star})
    \ge
    \frac{\mu}{2}
    \norm{\w-\w^{\star}}^2
    ~~~\Longrightarrow~~~
    \frac{1}{\mu^2}
    \bignorm{\nabla f(\w)}^2   
    \ge
    \norm{\w-\w^{\star}}^2\,,
\end{align*}
as required.
\end{proof}

\newpage

\begin{lemma}
\label{lem:kkt_w_tilde}
Let $c_1,\dots,c_k\in \reals_{\ge 0}$
and let $\B_1,\dots, \B_k \succ \0_{\dim\times \dim}$.
Consider sequentially solving $k$ separable
tasks $(\dataset_t=(\X_t, y_t=\1))_{t\in\cnt{k}}$ 
(recall the simplification in Remark~\ref{rmk:simplification})
using the following iterative update rule:
\begin{align}
\label{eq:iterative_general_rule}
    \w_{0} &=\0_{\dim}
    \nonumber
    \\
    \forall t\in\cnt{k}:~~~\w_{t}
    &=
    {\argmin}_{\w}
    \!
    \norm{\w-c_t\w_{t-1}}^2_{{\B_{t}}}
    \enskip~
    \suchthat
    \,
    \enskip 
    \,\w^\top \x \ge 1,
    \,
    \forall \x\in \X_{t}~.
\end{align}
Then, 
for almost all datasets sampled from $k$ absolutely continuous distributions, 
when $\w_t \neq c_{t}\w_{t-1}$,
the unique dual solution 
$\valpha_t\in\reals_{\ge0}^{\abs{\dataset_{t}}}$
satisfying the KKT conditions 
of \eqref{eq:iterative_general_rule},
holds:
\begin{align*}
    \B_t(\w_t - c_{t}\w_{t-1})
    =
    \sum_{\x\in\dataset_{t}}
    \x
    \alpha_t(\x),
    \quad
    \text{AND}
    \quad
    \Bigprn{
    \forall \x\in\dataset_{t}:~
    \prn{\alpha_t(\x)>0
    \wedge 
    \w_t^\top\x = 1
    }
    ~\text{OR}~
    \prn{\alpha_t(\x)=0
    \wedge 
    \w_t^\top\x > 1
    }
    }\,.
\end{align*}
(That is, there is no support vector for which $\alpha_t(\x)=0$.)
\end{lemma}

\bigskip

\begin{proof}[Proof for \lemref{lem:kkt_w_tilde}]
The proof follows the techniques of the proof of Lemma 12 in Appendix B of \citet{soudry2018journal}.

Here, we focus only on task sequences where $\w_t = \mP_t({c_{t}\w_{t-1}}) \neq c_{t}\w_{t-1},~\forall t\in\cnt{k}$
(we can always reduce to such cases by simply removing from the sequence any task $t$ for which $w_t=\w_{t-1}$).
Thus, $\w_t$ must lie on the boundary of the $t$\nth feasible set, as can be seen from the unique best approximation (or nearest point) property.
Consequently, the support set
${\supp_{t}\triangleq\left\{
\x \in \X_t \mid \w_t^\top \x = 1
\right\}}$ is nonempty.

For almost all datasets (all except measure zero), 
no more than $\dim$ datapoints will be on the same hyperplane (\eg the $\w_t^\top \x = 1$ hyperplane). 
Therefore, for any task $t\in\cnt{k}$ there can be at most $D$ support vectors, \ie $|\supp_{t}|\leq D$. 
Also, for almost all datasets, any set of at most $D$ vectors (\eg $\supp_{t}$) is linearly independent.

Let $\valpha_t\in\reals^{\abs{\dataset_t}}_{\ge0}$ be the dual solution satisfying the KKT conditions 
of \eqref{eq:iterative_general_rule} in the $t$\nth task.
We denote the matrix whose columns are the support vectors by $\X_{\supp_{t}}\in\reals^{\dim\times\abs{\supp_{t}}}$
and the dual solution, restricted to its components corresponding to support vectors, by $\valpha_{\supp_{t}}$.
Due to the complementary slackness,
datapoints outside of the support set (\ie $\w_t^\top\x>1$), must have a corresponding zero dual variable $\alpha_t(\x)=0$.
We now recall that $\B_t$ is invertible, and rewrite the stationarity condition as,
\begin{align}
\label{eq:kkt}
\B_t(\w_t - c_{t}\w_{t-1})
    &=
    \sum_{\x\in\dataset_{t}}
    \x
    \alpha_t(\x)
    =
    \sum_{\x\in\supp_{t}}
    \x
    \alpha_t(\x)
    +
    \!\!\!\!
    \cancel{
    \sum_{\x\in\dataset_{t}:\,\w_t^\top\x>1}
    \!\!\!
    \x
    \underbrace{
    \alpha_t(\x)
    }_{
    =0
    }
    }
    =
    \X_{\supp_{t}}\valpha_{\supp_{t}}
    \nonumber
    \\
    \w_{t}-c_{t}\w_{t-1}	
 &=\B_{t}^{-1}\X_{\supp_{t}}\valpha_{\supp_{t}}\,.
\end{align}

Multiplying from the left by the full row rank matrix $\X_{\supp_{t}}^{\top}$, we get on the one hand that $\valpha_{\supp_{t}}$ is uniquely defined as,
\begin{align}
\underbrace{\X_{\supp_{t}}^{\top}\w_{t}}_{=\1}-c_{t}\X_{\supp_{t}}^{\top}\w_{t-1}
&
=\X_{\supp_{t}}^{\top}\B_{t}^{-1}\X_{\supp_{t}}\valpha_{\supp_{t}}
\nonumber
\\
\valpha_{\supp_{t}}	
&=
\left(\X_{\supp_{t}}^{\top}\B_{t}^{-1}\X_{\supp_{t}}\right)^{-1}\1-c_{t}\left(\X_{\supp_{t}}^{\top}\B_{t}^{-1}\X_{\supp_{t}}\right)^{-1}\X_{\supp_{t}}^{\top}\w_{t-1}
~,
\label{eq:alpha_kkt}
\end{align}
where $\X_{\supp_{t}}^{\top}\B_{t}^{-1}\X_{\supp_{t}}$ is invertible because $\B_t,\B_t^{-1}$ are positive definite invertible matrices and $\X_{\supp_{t}}$ is full column rank as explained above (to see this, notice that $\B_t^{-1}$ has a Cholesky decomposition such that $\sqrt{\B_t^{-1}}$ preserves the full column rank of $\X_{\supp_{t}}$).
On the other hand, by substituting \eqref{eq:alpha_kkt} back into the condition in \eqref{eq:kkt}, 
we get that
$\forall t\in\cnt{k}$,
\begin{align}
\w_{t}	
&
=
\B_{t}^{-1}\X_{\supp_{t}}\valpha_{\supp_{t}}+c_{t}\w_{t-1}
\nonumber
\\
&
=\B_{t}^{-1}\X_{\supp_{t}}\left(\left(\X_{\supp_{t}}^{\top}\B_{t}^{-1}\X_{\supp_{t}}\right)^{-1}\1-
c_{t}\left(\X_{\supp_{t}}^{\top}\B_{t}^{-1}\X_{\supp_{t}}\right)^{-1}\X_{\supp_{t}}^{\top}\w_{t-1}\right)+c_{t}\w_{t-1}
\nonumber
\\
&
=\B_{t}^{-1}\X_{\supp_{t}}\left(\X_{\supp_{t}}^{\top}\B_{t}^{-1}\X_{\supp_{t}}\right)^{-1}\1+
c_{t}\left(\I-\B_{t}^{-1}\X_{\supp_{t}}\left(\X_{\supp_{t}}^{\top}\B_{t}^{-1}\X_{\supp_{t}}\right)^{-1}\X_{\supp_{t}}^{\top}\right)\w_{t-1}~.
\label{eq:w_recursive}
\end{align}

\pagebreak

Importantly, the above implies that, recursively, $\w_{t-1}$ is a rational function in the components of $\X_1, \ldots, \X_{t-1}$
(where $c_1,\ldots,c_{t}$ and $\B_1,\dots,\B_{t}$ are \emph{given}).
Clearly, the same holds for 
$\valpha_{\supp_{t}}$ (\eqref{eq:alpha_kkt})
which entirely depends on $\X_t,\B_t$, and $\w_{t-1}$.
\linebreak
Hence, its entries can be expressed as
$\m@th\displaystyle{(\valpha_{\supp_{t}})_n=
\hfrac{p^{(t)}_n(\X_1,\ldots,\X_t)}{q^{(t)}_n(\X_1,\ldots,\X_t)}}$
for some polynomials $p^{(t)}_n,q^{(t)}_n$.
\linebreak
Now, similarly to \citet{soudry2018journal},
we notice that ${(\valpha_{\supp_{t}})_n=0}$ only if $p^{(t)}_n(\X_1,\ldots,\X_t)=0$,
\ie the components of $\X_1,\ldots,\X_t$ must constitute a root of the polynomial $p^{(t)}_n$.
However, the roots of any polynomial have measure zero, unless that polynomial is the zero polynomial, \ie $p^{(t)}_n(\X_1,\ldots,\X_t)=0,~\forall \X_1,\ldots,\X_t$.

\paragraph{Our goal now.}
To prove that our polynomials $(p^{(t)}_n)_t$ cannot be zero polynomials, it is sufficient to construct a specific task sequence for which they are not zero.
Then, we will be able to conclude that the event in which $\valpha_{\supp_{t}}$ has a zero entry, is measure zero.

For the sake of readability, we divide our proof into two cases.

\begin{enumerate}
\item 
\textbf{When $\B_1=\dots=\B_k=\I_{\dim\times\dim}$}:
The expressions from \eqref{eq:alpha_kkt}~and~\eqref{eq:w_recursive} become, $\forall t\in\cnt{k}$,
\begin{align}   
\begin{split}
\valpha_{\supp_{t}}	
&
=
\left(\X_{\supp_{t}}^{\top}\X_{\supp_{t}}\right)^{-1}\1-c_{t}\left(\X_{\supp_{t}}^{\top}\X_{\supp_{t}}\right)^{-1}\X_{\supp_{t}}^{\top}\w_{t-1}
\\
\w_{t}	
&
=\X_{\supp_{t}}\left(\X_{\supp_{t}}^{\top}\X_{\supp_{t}}\right)^{-1}\1+
c_{t}\left(\I-\X_{\supp_{t}}\left(\X_{\supp_{t}}^{\top}\X_{\supp_{t}}\right)^{-1}\X_{\supp_{t}}^{\top}\right)\w_{t-1}~.
\end{split}
\label{eq:identity_w_recursive}
\end{align}

\paragraph{Detailed construction of the task sequence.}
We are given an arbitrary dimensionality $\dim$. 
Let $\e_i$ be the $i$\nth standard unit vector in $\reals^\dim$.
Let $N_{1},\dots, N_{k}\in\left[\dim\right]$ be an arbitrary sequence of the support sets' sizes.

We define the following datasets:
$$
\underbrace{\X_{t}}_{D\times N_{t}}=\beta_{t}\left[\begin{array}{ccc}
\mathbf{e}_{1} & \cdots & \mathbf{e}_{N_{t}}\end{array}\right],\quad\quad\forall t\in\left[k\right]
$$
for some sequence $\left(\beta_{t}\right)$ of strictly positive numbers which we will define later.
Under this construction (after choosing appropriate $(\beta_t)$ that will ensure projections at every iteration), the support set of each of the datasets will be identical to the dataset itself,
hence we use $\X_t = \X_{\supp_{t}}$ and $\valpha_t = \valpha_{\supp_{t}}$ interchangeably for the rest of the construction.

Notice that under our construction it holds that
$$\underbrace{\X_{t}^{\top}\X_{t}}_{N_{t}\times N_{t}}	=\beta_{t}^{2}\I_{N_{t}\times N_{t}},
\quad\quad
\underbrace{\X_{t}\X_{t}^{\top}}_{D\times D}=\beta_{t}^{2}\left[\begin{array}{cc}
\I_{N_{t}\times N_{t}}\\
 & \0_{\left(D-N_{t}\right)\times\left(D-N_{t}\right)}
\end{array}\right],
\quad\quad
\X_{t}\1_{N_{t}}=\beta_{t}\left[\substack{\1_{N_{t}}\\
\0_{D-N_{t}}}
\right]\,.
$$

Using the expression in \eqref{eq:identity_w_recursive} and recalling that $\w_{0}=\0_{D}$, we have:
\begin{align*}
\w_{t}	&=\X_{t}\left(\X_{t}^{\top}\X_{t}\right)^{-1}\1+c_{t}\left(\I-\X_{t}\left(\X_{t}^{\top}\X_{t}\right)^{-1}\X_{t}^{\top}\right)\w_{t-1}
=
\frac{1}{\beta_{t}^{2}}\X_{t}\1_{N_{t}}+c_{t}\underbrace{\left(\I-\frac{1}{\beta_{t}^{2}}\X_{t}\X_{t}^{\top}\right)}_{\triangleq\mathbf{M}_{t}}\w_{t-1}
\\
&	
 =
 \frac{1}{\beta_{t}^{2}}\X_{t}\1_{N_{t}}+c_{t}\mathbf{M}_{t}\w_{t-1}
	=\frac{1}{\beta_{t}^{2}}\X_{t}\1_{N_{t}}+c_{t}\mathbf{M}_{t}\left(\frac{1}{\beta_{t-1}^{2}}\X_{t-1}\1_{N_{t-1}}+c_{t-1}\mathbf{M}_{t-1}\w_{t-2}\right)	
\\
&=\frac{1}{\beta_{t}^{2}}\X_{t}\1_{N_{t}}+\frac{1}{\beta_{t-1}^{2}}c_{t}\mathbf{M}_{t}\X_{t-1}\1_{N_{t-1}}+c_{t}c_{t-1}\mathbf{M}_{t}\mathbf{M}_{t-1}\w_{t-2}
\\
\explain{\text{recursively}}
&
=\frac{1}{\beta_{t}^{2}}\X_{t}\1_{N_{t}}+\sum_{i=1}^{t-1}
c_{t}\cdots c_{i+1}
\frac{1}{\beta_{i}^{2}}
\mathbf{M}_{t}\cdots\mathbf{M}_{i+1}
\X_{i}\1_{N_{i}}\,.
\end{align*}

Plugging in the diagonal $\mathbf{M}_{i}=\I_{D\times D}-\frac{1}{\beta_{i}^{2}}\X_{i}\X_{i}^{\top}=\left[\begin{array}{cc}
\0_{N_{i}\times N_{i}}\\
 & \I_{\left(D-N_{i}\right)\times\left(D-N_{i}\right)}
\end{array}\right]$,
we get
\begin{align*}
    \w_{t}	
    &=
    \frac{1}{\beta_{t}^{2}}\X_{t}\1_{N_{t}}
    +
    \sum_{i=1}^{t-1}\biggprn{\frac{1}{\beta_{i}^{2}}
    \prod_{j=i+1}^{t}c_{j}\left[\begin{array}{cc}
\0_{N_{j}\times N_{j}}
\\
 & \I_{\left(D-N_{j}\right)\times\left(D-N_{j}\right)}
\end{array}\right]
\,\cdot\!\!\!\!
\underbrace{\X_{i}\1_{N_{i}}}_{
=\beta_i \left[\substack{
\1_{N_{i}}
 \\
\0_{D-N_{i}}    
}\right]
}
}
\\
&=
\frac{1}{\beta_{t}^{2}}\X_{t}\1_{N_{t}}+\sum_{i=1}^{t-1}\left(\left(\frac{1}{\beta_{i}}\prod_{j=i+1}^{t}c_{j}\right)
\begin{bmatrix}\mathbb{I}\left[1\le N_{i}\right]\prod_{j=i+1}^{t}\mathbb{I}\left[1>N_{j}\right]\\
\vdots\\
\mathbb{I}\left[D\le N_{i}\right]\prod_{j=i+1}^{t}\mathbb{I}\left[D>N_{j}\right]
\end{bmatrix}\right)	
\\
&=\frac{1}{\beta_{t}}
\begin{bmatrix}
\1_{N_{t}}\\
\0_{D-N_{t}}
\end{bmatrix}
+\sum_{i=1}^{t-1}\left(\left(\frac{1}{\beta_{i}}\prod_{j=i+1}^{t}c_{j}\right)
\begin{bmatrix}\mathbb{I}\left[N_{i}\ge1>\max\left(N_{i+1},\dots,N_{t}\right)\right]\\
\vdots\\
\mathbb{I}\left[N_{i}\ge D>\max\left(N_{i+1},\dots,N_{t}\right)\right]
\end{bmatrix}\right)\,,
\end{align*}
and since multiplying by $\X_{t}^{\top}$ from the left “trims” the last $D-N_{t}$ rows of any vector, we get:
\begin{align*}
\underbrace{\X_{t}^{\top}\w_{t-1}}_{N_t \times 1}
&=\X_{t}^{\top}\left(\frac{1}{\beta_{t-1}}
\begin{bmatrix}
\1_{N_{t-1}}\\
\0_{D-N_{t-1}}
\end{bmatrix}
+\sum_{i=1}^{t-2}\left(\left(\frac{1}{\beta_{i}}\prod_{j=i+1}^{t-1}c_{j}\right)
\begin{bmatrix}\mathbb{I}\left[N_{i}\ge1>\max\left(N_{i+1},\dots,N_{t-1}\right)\right]\\
\vdots\\
\mathbb{I}\left[N_{i}\ge D>\max\left(N_{i+1},\dots,N_{t-1}\right)\right]
\end{bmatrix}
\right)\right)
\\
&=
\beta_{t}\left(\frac{1}{\beta_{t-1}}
\begin{bmatrix}\1_{\min\left(N_{t-1},N_{t}\right)}\\
\0_{\max\left(N_{t}-N_{t-1},0\right)}
\end{bmatrix}
+\sum_{i=1}^{t-2}\left(\left(\frac{1}{\beta_{i}}
\prod_{j=i+1}^{t-1}c_{j}\right)
\begin{bmatrix}\mathbb{I}\left[N_{i}\ge1>\max\left(N_{i+1},\dots,N_{t-1}\right)\right]\\
\vdots\\
\mathbb{I}\left[N_{i}\ge N_{t}>\max\left(N_{i+1},\dots,N_{t-1}\right)\right]
\end{bmatrix}\right)\right)
\end{align*}

We then plug in the above into $\bm{\alpha}_{t}$ again (using \eqref{eq:alpha_kkt}):
\begin{align*}
\underbrace{\bm{\alpha}_{t}}_{N_{t}\times 1}	\!
&=
\underbrace{\left(\X_{t}^{\top}\X_{t}\right)^{-1}}_{N_{t}\times N_{t}}
\underbrace{
\left(\1_{N_{t}}-c_{t}\X_{t}^{\top}\w_{t-1}\right)
}_{N_t \times 1}
\\
&=	
\frac{1}{\beta_{t}^{2}}
\!
\left(\1_{N_{t}}\!-c_{t}\beta_{t}
\!
\left(\frac{1}{\beta_{t-1}}
 \!\begin{bmatrix}\1_{\min\left(N_{t-1},N_{t}\right)}\\
\0_{\max\left(N_{t}-N_{t-1},0\right)}
\end{bmatrix}
\!
+
\sum_{i=1}^{t-2}
\!
\left(\frac{1}{\beta_{i}}
\prod_{j=i+1}^{t-1}\!c_{j}\right)
\!\!\!
\begin{bmatrix}\mathbb{I}\left[N_{i}\ge1>\max\left(N_{i+1},\dots,N_{t-1}\right)\right]\\
\vdots\\
\mathbb{I}\left[N_{i}\ge N_{t}>\max\left(N_{i+1},\dots,N_{t-1}\right)\right]
\end{bmatrix}
\right)\!\right)
\end{align*}

    Then we get the following elementwise formula $\forall n\in\left[N_{t}\right]$:
\begin{align*}
\beta_{t}^{2}\left(\bm{\alpha}_{t}\right)_{n}	
=
1-\underbrace{c_{t}\beta_{t}\frac{1}{\beta_{t-1}}}_{\ge0}
\underbrace{\mathbb{I}\left[n\le N_{t-1}\right]}_{\le 1}
-
\underbrace{c_{t}\beta_{t}}_{\ge0}\sum_{i=1}^{t-2}
\underbrace{
\mathbb{I}\left[N_{i}\ge n>\max\left(N_{i+1},\dots,N_{t-1}\right)\right]
}_{\le 1}
\underbrace{\frac{1}{\beta_{i}}\prod_{j=i+1}^{t-1}c_{j}}_{\ge0}\,.
\end{align*}
Denoting $C\triangleq\max\left\{ 1,\,\max_{t\in\cnt{k}}\!c_{t} \right\}$, we lower bound the above as:
\begin{align*}
\beta_{t}^{2}\left(\bm{\alpha}_{t}\right)_{n}	
&
\ge1-c_{t}\beta_{t}\frac{1}{\beta_{t-1}}-c_{t}\beta_{t}\sum_{i=1}^{t-2}\frac{1}{\beta_{i}}\prod_{j=i+1}^{t-1}c_{j}=1-\beta_{t}\sum_{i=1}^{t-1}\frac{1}{\beta_{i}}\prod_{j=i+1}^{t}c_{j}
	\ge
 1-\beta_{t}\sum_{i=1}^{t-1}\frac{1}{\beta_{i}}C^{t-i}\
\\
\explain{C\ge 1}
&
\ge
1-\beta_{t}C^{t}\sum_{i=1}^{t-1}\frac{1}{\beta_{i}}\ge1-\left(t-1\right)\underbrace{\beta_{t}C^{t}\max_{i\in\left[t-1\right]}\frac{1}{\beta_{i}}}_{>0}>1-t\beta_{t}C^{t}\max_{i\in\left[t-1\right]}\frac{1}{\beta_{i}}\stackrel{\text{require}}{\ge}0
\\
\frac{\min_{i\in\left[t-1\right]}\beta_{i}}{tC^{t}}&\ge\beta_{t}
\end{align*}
Thus, to hold the above, we can choose $\left(\beta_{t}\right)$ to be a decreasing sequence as follows:
\begin{align*}
\beta_{1}=1,\quad
\beta_{t}=
\frac{\min_{i\in\left[t-1\right]}\beta_{i}}{tC^{t}}
=
\frac{1}{tC^{t}}\beta_{t-1}=\frac{1}{t!C^{2+3+\dots+t}}=\frac{1}{t!C^{\left(t+2\right)\left(t-1\right)/2}}\,.
\end{align*}

\paragraph{Summary of the first case.}
We showed a construction where at \emph{any} iteration $t\in\cnt{k}$,
\emph{all} $N_t$ entries of the corresponding dual vector $\valpha_{\supp_t}(=\valpha_t)$ are strictly positive.
As explained above, this implies that when $\w_t\neq c_{t}\w_{t-1}$, the polynomial $p_n^{(t)}$ is \emph{not} a zero polynomial ${\forall t\in\cnt{k}},{\forall n\in N_t}$
and thus becomes zero only in a finite number of measure zero roots.
Using the union bound on the countable number of iterations $k$ 
and all possible choices of support sizes $N_1,\smalldots,N_k\in\cnt{\dim}$, we still remain with a measure zero event.

\medskip


\item
\textbf{General $\B_1,\ldots,\B_t$}:
We can extend the techniques above and define the polynomials over the choices of $(\B_t)$ as well,
\ie 
$\m@th\displaystyle{(\valpha_{\supp_{t}})_n=
\hfrac{p^{(t)}_n(\X_1,\ldots,\X_t, \B_1,\ldots,\B_t)}{q^{(t)}_n(\X_1,\ldots,\X_t, \B_1,\ldots,\B_t)}}$
for some polynomials $p^{(t)}_n,q^{(t)}_n$.
\linebreak
Then, our first case above, where $\B_t = \I,\forall t$, shows that the roots of these updated polynomials are of measure zero.
Again, employing the union bound over
the countable number of iterations $k$ 
and all possible choices of support sizes $N_1,\smalldots,N_k\in\cnt{\dim}$,
shows that ${\forall t\in\cnt{k}:~\w_t \neq c_t \w_{t-1}\Longrightarrow\valpha_{\supp_{t}}}\succ\0$ (elementwise) almost surely.
\end{enumerate}

\vspace{-1.1em}

\end{proof}

\vspace{-0.4em}

After showing that either $\valpha_t=\0$
(when $\w_t = c_{t}\w_{t-1}$) or $\valpha_{\supp_{t}}\succ\0$ (elementwise),
we can conclude the following. 
\begin{corollary}
\label{cor:finite_w_tilde}
    Under the conditions and iterative process in \lemref{lem:kkt_w_tilde},
    when
    $\w_t\neq c_{t}\w_{t-1}$, 
    there almost surely exists a finite $\tilde{\w}_t$ such that
    $$
    \tsum_{\x\in\supp_{t}}
    \x
    \exp\prn{-\tilde{\w}_{t}^\top \x}
    =
    \B_t(\w_t - c_{t}\w_{t-1})
    =
    \tsum_{\x\in\supp_{t}}
    \x
    \underbrace{\alpha_t(\x)}_{>0}
    \,.
    $$
\end{corollary}

\vspace{-1.5em}

\begin{proof}
    The support $\supp_{t}$ is linearly independent a.s., 
    so we simply require ${\X^\top_{\supp_{t}} \tilde{\w}_t \!=\!-\ln{\valpha_{\supp_{t}}}}$ for a full row rank $\X^\top_{\supp_{t}}$.
\end{proof}

\vspace{.7em}

\newpage

\subsection{Main Result}

We are now ready to prove \thmref{thm:weakly_regularized}.
In Remark~\ref{rmk:analysis_differences}, we explained why existing analysis techniques 
(\eg from \citet{rosset2004margin,wei2019regularization}) are less suitable for our setting, in which the scale of iterates is also of great importance (and not only their direction).
More related tools are the ones that were used in \citet{soudry2018journal},
which analyzed the convergence of the gradient descent iterates under \emph{unregularized} problems with exponential losses to the max-margin solution as the number of gradient steps $t\to\infty$. 
Here, however, we analyze the convergence of the (unique) minimizer of the \emph{regularized} problem to a max-margin solution, as the regularization strength $\lambda\to 0$.
While there are some technical similarities, there are also differences and challenges of a different nature.

\bigskip

\begin{recall}[\thmref{thm:weakly_regularized}]
Let $\lambda_t\!=\!\lambda\!>\!0$, 
and 
$\B_{t}\!=\!\I$, $\forall t\!\in\!\cnt{k}$.
Then, 
for almost all separable datasets,\footref{fn:almost_all} %
in the limit of $\lambda\!\to\!0$,
it holds that
$\w_{t}^{(\lambda)} 
\to
{\ln\prn{\tfrac{1}{\lambda}}\w_{t}}$
with a residual of
$\tnorm{\w_{t}^{(\lambda)} 
\!-\!
{\ln\prn{\tfrac{1}{\lambda}}\w_{t}}}
=\bigO\prn{t\ln\ln\prn{\tfrac{1}{\lambda}}}
$. 
\linebreak
As a result, 
at any iteration 
$t=o\prn{\frac{\ln\prn{\nicefrac{1}{\lambda}}}{{\ln \ln\prn{\nicefrac{1}{\lambda}}}}}$,
we get
$$
\lim_{\lambda\to0}
\frac{\w_t^{(\lambda)}}{\tnorm{\w_t^{(\lambda)}}}
=
\frac{\w_t}{\tnorm{\w_t}}
\,.
$$
\end{recall}

\medskip

\begin{proof}
We will prove by induction on $t\ge0$
that 
the scale of the residual 
$
\m@th\displaystyle
\residual_{t}^{(\lambda)}
\triangleq
{\w_{t}^{(\lambda)} 
\!-\!
{\ln\prn{\tfrac{1}{\lambda}}\w_{t}}}
$
at each iteration is 
$\bigO\prn{t\ln\ln\prn{\tfrac{1}{\lambda}}}
$;
and that consequently (since $\ln\prn{\tfrac{1}{\lambda}}\w_t$ grows faster), 
the iterates are either identical
(\ie $\w_{t}^{(\lambda)}=\w_{t}$)
or converge in the same direction when $\lambda\to0$,
\ie $
\lim_{\lambda\to0}
\frac{\w_t^{(\lambda)}}{\tnorm{\w_t^{(\lambda)}}}
=
\frac{\w_t}{\tnorm{\w_t}}
$.

\paragraph{For $t=0$:}
By the conditions of the theorem, it trivially holds that 
$\w_0^{(\lambda)}=\w_0=\0_{\dim}$ and
$\residual_{0}^{(\lambda)}
=
\w_{0}^{(\lambda)} 
-{\ln\prn{\tfrac{1}{\lambda}}\w_{0}}=\0_{\dim}$.

\bigskip

\paragraph{For $t\ge1$:}

The solved optimization problem
(recall Remark~\ref{rmk:simplification})
is:
%
$$\displaystyle
\w^{(\lambda)}_{t}
=
\argmin_{\w \in \reals^{\dim}} 
\lamloss(\w)
\triangleq
\argmin_{\w \in \reals^{\dim}} 
{
\sum_{\x\in \dataset_{t}}
e^{-\w^\top \x }
+
\frac{\lambda}{2}
\norm{\w\!-\!\w^{(\lambda)}_{t-1}}^2
}\,.
$$


\paragraph{Proof's idea.}
Notice that the objective above is $\lambda$-strongly convex, 
since its Hessian matrix is
${\nabla^2 \lamloss(\w)
=
\lambda\I + \sum_{\x} e^{-\w^\top \x } \x\x^\top
\succeq \lambda\I \succ \0}$.
We are going to define an $\bigO\prn{1}$
vector $\tilde{\w}_{t}$
and a sign $\wsign_t\in\{-1,+1\}$,
and employ the triangle inequality and \lemref{lem:strong-convexity} to show that
\begin{align}
\begin{split}
\label{eq:proofs_idea}
\bignorm{\residual_{t}^{(\lambda)}}
&\triangleq
\bignorm{
\w_{t}^{(\lambda)}-\ln\prn{\tfrac{1}{\lambda}}\w_{t}
}
\\
&
=
\bignorm{
\w_{t}^{(\lambda)}
-
\Bigprn{
\prn{\ln\prn{\tfrac{1}{\lambda}}
+
\wsign_t\ln\ln\prn{\tfrac{1}{\lambda}}}\w_{t}
+
\tilde{\w}_{t}
}
+
\wsign_t\ln\ln\prn{\tfrac{1}{\lambda}}
\w_{t}
+
\tilde{\w}_{t}
}
\\
&\le 
\underbrace{
\frac{1}{\lambda}
\norm{
\nabla
\lamloss\Bigprn{
    \prn{\ln\prn{\tfrac{1}{\lambda}}
    + \wsign_t\ln\ln\prn{\tfrac{1}{\lambda}}}\w_{t}
    +
    \tilde{\w}_{t}
}
}
}_{
=\bigO\prn{\prn{t-1}\ln\ln\prn{\tfrac{1}{\lambda}}}
}
+
\ln\ln\prn{\tfrac{1}{\lambda}}
\underbrace{\bignorm{\w_{t}}}_{=\bigO\prn{1}}
+
\underbrace{\bignorm{\tilde{\w}_{t}}}_{=\bigO\prn{1}}
=
\bigO\prn{t\ln\ln\prn{\tfrac{1}{\lambda}}}\,.
\end{split}
\end{align}
Then, since $\bignorm{
\w_{t}^{(\lambda)}-\ln\prn{\tfrac{1}{\lambda}}\w_{t}
}=\bigO\prn{t\ln\ln\prn{\tfrac{1}{\lambda}}}$
it will immediately follow that the weakly-regularized solution $\w_{t}^{(\lambda)}$
converges in direction to the Sequential Max-Margin solution $\w_{t}$.

\newpage

\paragraph{Back to the proof.}
First, we compute the gradient of $\lamloss$, normalized by $\lambda$:
\begin{align*}
    &\frac{1}{\lambda}
    \nabla
    \lamloss(\w)
    =
    \frac{1}{\lambda}
    \prn{
    \sum_{\x\in\dataset_{t}}-\x\exp\left(-\w^{\top}\x\right)
    +\lambda\w-\lambda\w^{(\lambda)}_{t-1}}
    =
    -\frac{1}{\lambda}
    \sum_{\x\in\dataset_{t}}\x\exp\left(-\w^{\top}\x\right)+
    \w-\w^{(\lambda)}_{t-1}\,.
\end{align*}
Then, we plug in
$\w
=
\prn{\ln\prn{\tfrac{1}{\lambda}}+\wsign_{t}\ln\ln\prn{\tfrac{1}{\lambda}}}\w_{t}
+\tilde{\w}_{t}
=
\ln\prn{\tfrac{1}{\lambda}\ln^{\wsign_{t}\!}\prn{\tfrac{1}{\lambda}}}
\w_{t}
+\tilde{\w}_{t}$,
for some sign $\wsign_{t}\in\{-1,+1\}$ and
a vector $\tilde{\w}_{t}$ with a norm independent of $\lambda$ (both will be defined below).
\begin{align*}
    \frac{1}{\lambda}
    \nabla
    \lamloss\Bigprn{
    \prn{\ln\prn{\tfrac{1}{\lambda}}
    +
    \wsign_{t}\ln\ln\prn{\tfrac{1}{\lambda}}
    }
    \w_{t}
    +
    \tilde{\w}_{t}}
    &=
    -\frac{1}{\lambda}
    \sum_{\x\in\dataset_{t}}
    \x
    e^{
    \ln\prn{\lambda\ln^{-\wsign_t\!}\prn{\tfrac{1}{\lambda}}}\w_{t}^\top 
    \x}
    e^{-\tilde{\w}_{t}^\top \x}
    +
    \ln\Bigprn{\tfrac{1}{\lambda}
    \ln^{\wsign_t\!}\prn{\tfrac{1}{\lambda}}}\w_{t}
    +\tilde{\w}_{t}
    -\w^{(\lambda)}_{t-1}
    \\
    &=
    -\frac{1}{\lambda}
    \sum_{\x\in\dataset_{t}}
    \x
    \prn{\lambda\ln^{-\wsign_t\!}\prn{\tfrac{1}{\lambda}}}^{\w_{t}^\top \x}
    e^{-\tilde{\w}_{t}^\top \x}
    +
    \ln\Bigprn{\tfrac{1}{\lambda}
    \ln^{\wsign_t\!}\prn{\tfrac{1}{\lambda}}}\w_{t}
    +
    \tilde{\w}_{t}
    \!-\!\w^{(\lambda)}_{t-1}
    \,.
\end{align*}
Now, denoting the set of support vectors by
$\supp_{t}\triangleq\left\{ \x\in \dataset_{t}\mid\w_{t}^{\top}\x=1\right\}$
(which might be empty for $t\ge 2$),
and using the inductive assumption
that
$
\w_{t-1}^{(\lambda)}=
\residual_{t-1}^{(\lambda)}
+
\ln\prn{\tfrac{1}{\lambda}}\w_{t-1}
$
(where $\bignorm{\residual_{t-1}^{(\lambda)}} = \bigO\prn{\prn{t-1}\ln\ln\prn{\tfrac{1}{\lambda}}}$)
, we get 
\begin{align*}
&
\frac{1}{\lambda}
    \nabla
    \lamloss\Bigprn{
    \prn{\ln\prn{\tfrac{1}{\lambda}}
    +
    \wsign_{t}\ln\ln\prn{\tfrac{1}{\lambda}}
    }
    \w_{t}
    +
    \tilde{\w}_{t}}
    \\
    &=
    -
    \frac{1}{\lambda}
    \prn{\lambda\ln^{-\wsign_t\!}\prn{\tfrac{1}{\lambda}}}^{1}
    \sum_{\x\in\supp_{t}}
    \x
    e^{-\tilde{\w}_{t}^\top \x}
    -
    \frac{1}{\lambda}\!
    \sum_{\x\notin\supp_{t}}
    \x
    \prn{\lambda\ln^{-\wsign_t\!}\prn{\tfrac{1}{\lambda}}}^{\w_{t}^\top \x}
    e^{-\tilde{\w}_{t}^\top \x}
    +
    \ln \prn{\tfrac{1}{\lambda}}\w_{t}
    -
    \ln \prn{\tfrac{1}{\lambda}} \w_{t-1}
    +
    \\
    &
    \hspace{3em}
    +
    \ln\ln^{\wsign_t\!}\prn{\tfrac{1}{\lambda}}
    \w_{t}
    +
    \tilde{\w}_{t}
    -
    \residual_{t-1}^{(\lambda)}
    \\
    &=
    -
    \ln^{-\wsign_t\!}\prn{\tfrac{1}{\lambda}}
    \sum_{\x\in\supp_{t}}
    \x
    e^{-\tilde{\w}_{t}^\top \x}
    -
    \sum_{\x\notin\supp_{t}}
    \x
    {\lambda}^{\w_{t}^\top \x-1}
    \prn{\ln\prn{\tfrac{1}{\lambda}}}^{-\wsign_t\w_{t}^\top \x}
    e^{-\tilde{\w}_{t}^\top \x}
    +
    \ln \prn{\tfrac{1}{\lambda}}
    \prn{\w_{t}-\w_{t-1}}
    +
    \\
    &
    \hspace{3em}
    +
    \ln\ln^{\wsign_t\!}\prn{\tfrac{1}{\lambda}}
    \w_{t}
    +
    \tilde{\w}_{t}
    -
    \residual_{t-1}^{(\lambda)}
    \,.
\end{align*}

By the triangle inequality 
and since $\bignorm{\ln\ln^{\wsign_t\!}\prn{\tfrac{1}{\lambda}}}
=
\bignorm{\wsign_t\ln\ln\prn{\tfrac{1}{\lambda}}}
=
\overbrace{\abs{\wsign_t}}^{= 1}
\ln\ln\prn{\tfrac{1}{\lambda}}$
, we have,
\begin{align*}   
&\norm{\frac{1}{\lambda}
\nabla
\lamloss\Bigprn{\prn{\ln\prn{\tfrac{1}{\lambda}}+
    \wsign_{t}\ln\ln\prn{\tfrac{1}{\lambda}}}\w_{t}+\tilde{\w}_{t}}
}
\\
&
\le
\,
\underbrace{\biggnorm{
    \ln \prn{\tfrac{1}{\lambda}}
    \prn{\w_{t}\!-\!\w_{t-1}}
    -
    \ln^{-\wsign_t\!}\prn{\tfrac{1}{\lambda}}
    \sum_{\x\in\supp_{t}}\!
    \x
    e^{-\tilde{\w}_{t}^\top \x}
}
}_{\triangleq\vect{a}_1(\lambda)}
+
\underbrace{
\biggnorm{
    \sum_{\x\notin\supp_{t}}
    \x
    {\lambda}^{\w_{t}^\top \x-1}
    \!
    \prn{\ln\prn{\tfrac{1}{\lambda}}}^{-\wsign_t\w_{t}^\top \x}
    e^{-\tilde{\w}_{t}^\top \x}
}
}_{\triangleq\vect{a}_2(\lambda)}
+
\\
&
\hspace{1cm}
+
\ln\ln\prn{\tfrac{1}{\lambda}}
\Bignorm{
    \w_{t}
}
+
\Bignorm{
    \tilde{\w}_{t}
}
+
\Bignorm{
    \residual_{t-1}^{(\lambda)}
}
\end{align*}

\pagebreak

We now distinguish between two cases in which
the behavior of 
the sequential max-margin differs greatly 
(see \eqref{eq:adaptive_dynamics}).
\begin{enumerate}
    \item \textbf{When 
    $\w_{t}\triangleq \mP_{t}\prn{\w_{t-1}}
    \neq \w_{t-1}$ 
    (and necessarily 
    $|{\supp_{t}}|\ge1$):}
    We choose $\wsign_{t}=-1$ and $\vect{a}_1(\lambda)$ becomes:
    \begin{align*}
    \vect{a}_1(\lambda)
    &
    =
    \bignorm{
    \ln \prn{\tfrac{1}{\lambda}}
    \prn{\w_{t}
    -\w_{t-1}}
    -
    \ln^{-\wsign_t\!}\prn{\tfrac{1}{\lambda}}
    \tsum_{\x\in\supp_{t}}
    \x
    e^{-\tilde{\w}_{t}^\top \x}
    }
    =
    \ln \prn{\tfrac{1}{\lambda}}
    \bignorm{
    \prn{\w_{t}
    -\w_{t-1}}
    -
    \tsum_{\x\in\supp_{t}}
    \x
    e^{-\tilde{\w}_{t}^\top \x}
    }\,.
    %
    \end{align*}
    We thus wish to choose 
    $\tilde{\w}_{t}$
    so as to \emph{zero} $\vect{a}_1(\lambda)$.
    Combined with the KKT conditions of
    \eqref{eq:adaptive_dynamics},
    we require
    $$
    \sum_{\x\in\supp_{t}}
    \x
    e^{-\tilde{\w}_{t}^\top \x}
    =
    \w_{t}
    -\w_{t-1}
    \stackrel{\text{KKT}}{=}
    \sum_{\x\in\supp_{t}}
    \x
    \alpha(\x)
    \,,
    $$
    where $\valpha\in\reals_{\ge0}^{\abs{\dataset_{t}}}$ 
    is the dual solution of the sequential max-margin
    problem (\eqref{eq:adaptive_dynamics}).
    In \lemref{lem:kkt_w_tilde} and 
    \corref{cor:finite_w_tilde} (applied with $\B=\mathbf{I}$ and $\mathbf{v}=\w_{t-1}$) we show that such a vector $\tilde{\w}_{t}$ almost surely  exists.

    Furthermore, 
    since 
    $\lim_{\lambda\to 0} 
    \prn{\lambda^{c-1}
    \ln^{c} \prn{\nicefrac{1}{\lambda}}
    }=0,~\forall c>1$,
    it holds that
    $\vect{a}_2(\lambda)$ becomes
    \begin{align*}   
    \vect{a}_2(\lambda)
    &=
    \biggnorm{
    \sum_{\x\notin\supp_{t}}
    \underbrace{
        \x
        e^{-\tilde{\w}_{t}^\top \x}
    }_{={\bigO\prn{1}}}
    \underbrace{
        \lambda^{\overbrace{\w_{t}^\top \x-1}^{>0}}
        \prn{\ln \prn{\nicefrac{1}{\lambda}}}^{
        {\w_{t}^\top \x}
        }
    }_{\to 0}
    }
    \xrightarrow[]{\lambda\to 0} 0\,.
    \end{align*}

    In conclusion, we can choose $\tilde{\w}_t$ and $\wsign_{t}$
    such that $\norm{\tilde{\w}_t}=\bigO\prn{1}$,
    $\vect{a}_1(\lambda)=0$,
    and $\vect{a}_2(\lambda)\to0$.

    \medskip
    
    \item
    \textbf{When 
    $\w_{t}\triangleq \mP_{t}\prn{\w_{t-1}}
    = \w_{t-1}$ 
    (and possibly $\supp_{t} = \emptyset$):}
    We choose $\tilde{\w}_t=
    \0_{\dim}
    $ and $\wsign_{t}=1$.
    It follows that
    \begin{align*}
    \vect{a}_1(\lambda)
    &
    =\bignorm{\,
    \ln^{-\wsign_t\!}\prn{\tfrac{1}{\lambda}}
    \sum_{\x\in\supp_{t}}
    \x\,
    }
    =
    \underbrace{
    \ln^{-1\!}\prn{\tfrac{1}{\lambda}}
    }_{\to 0}
    \underbrace{
    \bignorm{\,
    \tsum_{\x\in\supp_{t}}
        \x\,
    }
    }_{={\bigO\prn{1}}}
    \xrightarrow[]{\lambda\to 0} 0\,,
    \\
    \vect{a}_2(\lambda)
    &=
    \biggnorm{
    \sum_{\x\notin\supp_{t}}
    \underbrace{
        \x
    }_{={\bigO\prn{1}}}
    \underbrace{
        \lambda^{\overbrace{\w_{t}^\top \x-1}^{>0}}
    }_{\to 0}
    \underbrace{
        \prn{\ln \prn{\tfrac{1}{\lambda}}}^{
        {\overbrace{-\w_{t}^\top \x}^{<-1}}
        }
    }_{\to 0}
    }
    \xrightarrow[]{\lambda\to 0} 0\,.
    \end{align*}
\end{enumerate}

\bigskip

As explained in our proof's idea (\ref{eq:proofs_idea}), we now use the $\lambda$-strong convexity of our objective
and \lemref{lem:strong-convexity}
to bound the distance to the optimum 
by
\begin{align*}
    \bignorm{\residual_{t}^{(\lambda)}}
    &\triangleq
    \bignorm{
    \w_{t}^{(\lambda)}
    -
    \ln\prn{\tfrac{1}{\lambda}}\w_{t}
    }
    =
    \bignorm{
    \w_{t}^{(\lambda)}
    -
    \prn{
    \prn{\ln\prn{\tfrac{1}{\lambda}}
    +
    \wsign_t\ln\ln\prn{\tfrac{1}{\lambda}}}\w_{t}
    +
    \tilde{\w}_{t}
    }
    +
    \wsign_t\ln\ln\prn{\tfrac{1}{\lambda}}
    \w_{t}
    +
    \tilde{\w}_{t}
    }
    \\
    \explain{\text{triangle ineq.}}
    &
    \le
    \bignorm{
    \w_{t}^{(\lambda)}
    -
    \prn{
    \prn{\ln\prn{\tfrac{1}{\lambda}}
    +
    \wsign_t\ln\ln\prn{\tfrac{1}{\lambda}}}\w_{t}
    +
    \tilde{\w}_{t}
    }
    }
    +
    \bignorm{
    \wsign_t\ln\ln\prn{\tfrac{1}{\lambda}}
    \w_{t}
    }
    +
    \bignorm{
    \tilde{\w}_{t}
    }
    \\
    \explain{\text{\lemref{lem:strong-convexity}}}
    &
    \le 
    \frac{1}{\lambda}
    \norm{
    \nabla
    \lamloss\Bigprn{
        \prn{\ln\prn{\tfrac{1}{\lambda}}
        + \wsign_t\ln\ln\prn{\tfrac{1}{\lambda}}}\w_{t}
        +
        \tilde{\w}_{t}
    }
    }
    +
    \ln\ln\prn{\tfrac{1}{\lambda}}
    \bignorm{\w_{t}}
    +
    \bignorm{\tilde{\w}_{t}}
    \\
    \explain{\text{all the above}}
    &
    \le
    \underbrace{
        \vect{a}_1(\lambda)+\vect{a}_2(\lambda)
    }_{\to 0}
    +
    2\ln\ln\prn{\tfrac{1}{\lambda}}
    \underbrace{\bignorm{\w_{t}}
    }_{=\bigO\prn{1}}
    +
    \underbrace{
        2\bignorm{
            \tilde{\w}_{t}
        }
    }_{=\bigO\prn{1}}
    \,\,\,\,+\!\!\!\!\!\!\!\!\!
    \underbrace{
    ~
    \bignorm{\residual_{t-1}^{(\lambda)}}
    ~
    }_{=\bigO\prn{\prn{t-1}\ln\ln\prn{\tfrac{1}{\lambda}}}}
    \!\!\!\!\!\!\!.
    %
\end{align*}
Finally, we conclude that
$
    \bignorm{\residual_{t}^{(\lambda)}}
    =\bigO\prn{t \ln\ln\prn{\tfrac{1}{\lambda}}}$.
\end{proof}

\newpage

\subsection{Limitations of our Analysis}
\label{sec:limitations}
We proved that \lemref{lem:kkt_w_tilde} 
and \corref{cor:finite_w_tilde},
which are pivotal for our proofs of Theorems~\ref{thm:weakly_regularized},~\ref{thm:weak_schedule},~and~\ref{thm:weak_weighted},
hold for almost all datasets and weighting schemes.
We took the approach of \citet{soudry2018journal} who analyzed a simpler non-continual single-task case 
and derived similar results from the perspective of the roots of some polynomials.
Below, we discuss the limitations of this analytical approach in our case.
\vspace{-0.2em}
\begin{enumerate}
\item \textbf{Task recurrence.}
In our proof of \lemref{lem:kkt_w_tilde}, we employed a construction where tasks do not recur. When tasks recur, the constructed polynomials have a higher-order dependence on the elements of $\X_1,\dots,\X_T$ ($T$ is the number of possible tasks, in contrast to the number of iterations $k$). 
Then, the analysis becomes more subtle 
due to some additional constraints on the number of support vectors at each iteration $(N_t)$ 
(\eg when the same task is seen at iterations $t$ and $\prn{t+1}$, it must hold that $N_t=N_{t+1}$).
Finding a construction for a general sequence of support sizes $(N_t)$ is thus more challenging because not any sequence is attainable.
Without such a general construction, it remains possible that task recurrence leads to
a collapse into the measure-zero scenarios
where the lemma does not hold.

On the other hand, we \emph{are} able to prove that the lemma holds for some cases where tasks recur.
For instance, we can design constructions that allow task recurrence, as long as
the support sizes are fixed and hold ${N_t=N\le\dim/2,~\forall t\in \cnt{k}}$.
Specifically, we can construct a task sequence of $k$ iterations over $T\le k$ tasks (using a task ordering $\tau$ as in \defref{def:ordering}).
We construct the datasets so as to form a ``cyclic'' sequence with two types of tasks.
That is, the columns of $\X_{\tau(t)}=\X_{\supp_{\tau(t)}}\in \reals^{\dim\times N}$ are 
$
\x^{\tau(t)}_{n}
=
\begin{cases}
+\cos\prn{\theta} \e_n + \sin\prn{\theta} \e_{D-n}  
& \tau(t) = \tau(k)
\\
-\cos\prn{\theta} \e_n + \sin\prn{\theta} \e_{D-n} 
& \tau(t) \neq \tau(k)
\end{cases}$.
\linebreak
Under this construction, which is easy to analyze 
(as in \appref{app:scheduling}), 
one can show that at the $k$\nth iteration, if $\w_k\neq c_k \w_{k-1}$ then
$\valpha_{\supp_k}=\gamma\1_{N} \succ \0$, for some $\gamma > 0$.
This suffices for showing that the polynomials are nonzero (when ${N_t=N\le\dim/2,~\forall t\in \cnt{k}}$, for an arbitrary length $k$, with any form of recurrence of the task seen at iteration $k$),
and conclude that recurrence does not necessarily collapse to measure zero events where the lemma does not hold.

\item \textbf{Weighting schemes.} 
We initially proved our lemma for the isotropic weighting scheme where $\B_t=\I$ (notably, this corresponds exactly to the case in our main result in \thmref{thm:weakly_regularized}).
Subsequently, we used these isotropic weighting schemes to establish that our lemma applies to almost all weighting schemes as well.
However, common weighting schemes, such as Fisher-information-based schemes, rely on the data observed in previously encountered tasks.
Again, this makes the construction of a general task sequence where $\valpha_{\supp_{t}}\succ\0$ more complicated. 
Thus, it is possible that such weighting schemes will collapse into the measure zero cases where the lemma does not hold.
\end{enumerate}

We hypothesize that the reservations we expressed above are merely limitations of the analytical tools we utilized.
\linebreak
Various simulations we conducted demonstrated an agreement between the weakly-regularized iterates and the Sequential Max-Margin iterates. Closing these gaps in our analysis will likely require an alternative analytical approach.

\bigskip

\subsubsection{Example of a Measure Zero Case}
Finally, we briefly demonstrate a measure zero case where \lemref{lem:kkt_w_tilde} 
and \corref{cor:finite_w_tilde} do not hold.

Let the first task be $\X_{1}=\left[\e_{1}\right]\in\reals^{\dim\times 1}$. Then, $\w_{1}=\e_{1}$.

The second task is $\X_{2}=\left[\e_{1},\e_{2}\right]\in\reals^{\dim\times 2}$.
Then,
\begin{align*}
\w_{2}	&=\arg\min\left\Vert \w-\w_{1}\right\Vert \,\text{s.t. }\left(\e_{1}^{\top}\w\ge1\right)\wedge\left(\e_{2}^{\top}\w\ge1\right)
\\
	&=\e_{1}+\e_{2}\,.
\end{align*}

 Both $\e_1,\e_2$ are support vectors of the 2\nd task, but from the stationarity condition 
 (plug in $\B_2 = \I$, $c_2=1$ into \eqref{eq:kkt}), it holds that,
$$\alpha_2(\e_{1})\,\e_{1}+\alpha_2(\e_{2})\,\e_{2}=\w_{2}-\w_{1}=\e_{2}\,,$$
 thus requiring that $\alpha_2(\e_{1})=0$.
 Crucially, this prevents the existence of a finite $\tilde{\w}_2$ that holds $\exp\prn{-\tilde{\w}_2^\top\e_1}=\alpha_2(\e_1)$.

\newpage

\section{Proofs for the Sequential Max-Margin Projections Scheme (\secref{sec:adaptive_svm})}
\label{app:successive-projections}

\begin{recall}[\lemref{lem:euclidean_to_hinge}]
Recall our definition of
$
\m@th\displaystyle
{R\triangleq 
\max_{m\in\cnt{T}}
\max_{(\x,y)\in\dataset_m }
\!\!\!\norm{\x}}$.
The quantities of \defref{def:quantities} are related
as follows:
\begin{align*}
    \forall\w\!\in\!\reals^{\dim},\,
    m\!\in\!\cnt{T}:
    ~~
    F_m(\w)
    \le
    \dist^2(\w, \feasible_m)
    %
    \!
    \max_{(\x,y)\in\dataset_m}
    \!\!\!
    \!
    \tnorm{\x}^2
    \le
    \dist^2(\w, \teachers) R^2
    \,.
\end{align*}
Moreover, for the Sequential Max-Margin iterates 
$(\w_t)$ of Scheme~\ref{proc:adaptive},
all quantities are upper bounded by the ``problem complexity'',
\ie $\dist^2(\w_t, \teachers) R^2\le\norm{\teacher}^2\!R^2$.
\end{recall}

\medskip

\begin{proof}
    In our proof, we use the simplifying mapping from Remark~\ref{rmk:simplification} 
    ($y\x\longmapsto\x$).

    Using simple algebra and the Cauchy-Schwarz inequality, 
    we have 
    ${\forall m\in{\cnt{T}},\,
    \x\in\dataset_{m},\,
    \w\in\reals^{\dim}}$,
    \begin{align*}
        1-\w^{\top}\x
        &=
        1-\underbrace{\left(\mP_{m}\left(\w\right)\right)^{\top}\x}_{\ge1\text{ \,(see below)}}
        +
        \left(\mP_{m}\left(\w\right)\right)^{\top}\x-\w^{\top}\x
        \le
        \left(\mP_{m}\left(\w\right)-\w\right)^{\top}\x
        \\
        \explain{\text{Cauchy–Schwarz}}
        &
        \le
        \left\Vert \w-\mP_{m}\left(\w\right)\right\Vert 
        \tnorm{\x}
        =
        \dist(\w, \feasible_m)
        \tnorm{\x}
        \,,
    \end{align*}
    where the inequality in the underbrace stems from the fact that 
    $\mP_{m}(\w)\in\feasible_m$,
    and that $\feasible_m$ is defined (\asmref{asm:separability}) 
    as the set of solutions with a zero
    hinge loss over the samples in $\dataset_{m}$ (such as $\x$).

    Since $\teacher$ is the intersection of all ${\feasible_m}$,
    we get that \hfill 
    $\dist(\w, \feasible_m)\le \dist(\w, \teachers)$.

    Overall, it follows that
    \begin{align*}
        F_m (\w) 
        \triangleq
        \!
        \!
        \max_{(\x,y)\in\dataset_m}
        \!\!\!
        \prn{
        \max\left\{0,\,1-\w^\top \x\right\}
        }^2
        \!
        \,&
        \le
        \max_{\x\in\dataset_m}
        \!
        \Bigprn{
        \dist(\w, \feasible_m)
        \tnorm{\x}
        }^2
        \!
        =
        \dist^{2}(\w, \feasible_m)
        \max_{\x\in\dataset_m}
        \tnorm{\x}^2
        %
        %
        \le\!
        \dist^{2}(\w, \teachers)
        R^2
        .
    \end{align*}

Finally, due to the monotonicity from \lemref{lem:feasible_distance_monotonicity},
the iterates of \procref{proc:adaptive} hold that,
\begin{align*}
    \dist^{2}(\w_t, \teachers)
    \le 
    \dist^{2}(\w_0, \teachers)
    =\norm{\teacher}^2,
    ~~~~\forall t\in\naturals\,.
\end{align*}
\end{proof}

\newpage

\subsection{Adversarial Construction: Additional Discussion and Illustrations (\secref{sec:adversarial})}
\label{app:adversarial}

\vspace{-.5cm}

\begin{figure}[ht!]
    
      \centering
      \begin{minipage}[t!]{0.6\linewidth}
        \caption{
        Elaboration on the construction in \figref{fig:adversarial-data}.
        We always use $\dim\!=\!3$ dimensions.
        Here, we demonstrate our construction with only $T\!=\!20$ tasks.
        Each task consists of a single datapoint $\x_m$ which is positively labeled (\ie ${y_m=+1}$) and has a norm of ${\norm{\x_m} = 1 = R}$.
        Datapoints are uniformly spread on a plane slightly elevated above the $xy$-plane.
        We also plot the plane induced by a specific datapoint $\x_m$ (highlighted in orange), which is the boundary of its feasible set 
        ${\feasible_m
        \,\!\triangleq\!\,
        \left\{
        \w\!\in\!\reals^{\dim}
        \mid
        \w^\top \x_m\ge 1 \right\}}
        $.
        All such planes intersect at $\teacher$. 
        The elevation of the datapoints above the $xy$-plane determines the magnitude of $\norm{\teacher}$ (lower elevation implies a worse minimum margin and a larger $\norm{\teacher}$).
        In our experiment, we set ${\norm{\teacher} = 10}$. As the number of tasks ${T\!\to\!\infty}$, the uniform angles between consecutive tasks and the applied projections become smaller.
        }
      \end{minipage}
      \hfill
      \begin{minipage}[t!]{0.34\linewidth}
        {\includegraphics[width=.99\linewidth]{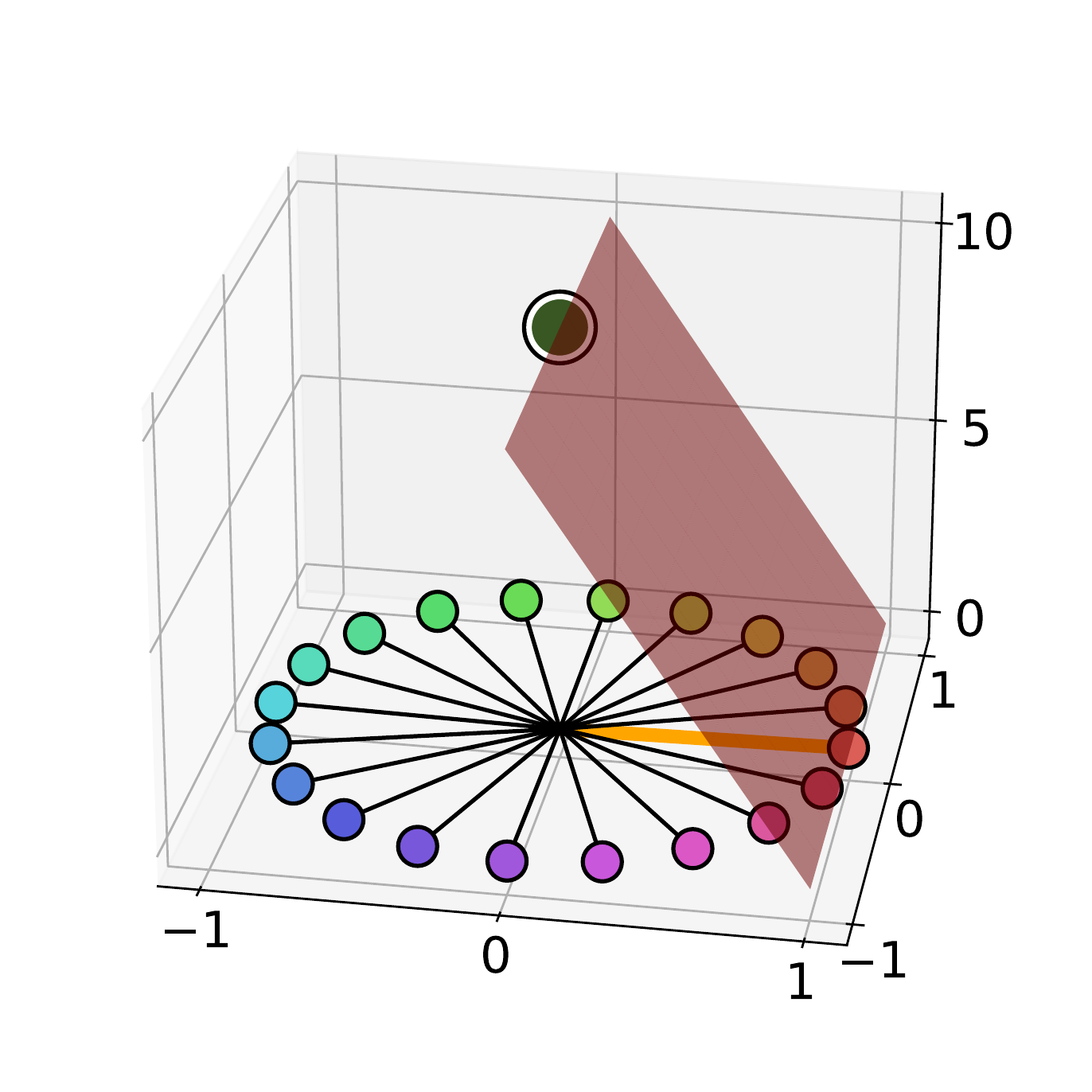}}
      \end{minipage}
\end{figure}

\vspace{-1em}

\begin{figure}[ht!]
    \centering
    {\includegraphics[width=.95\linewidth]{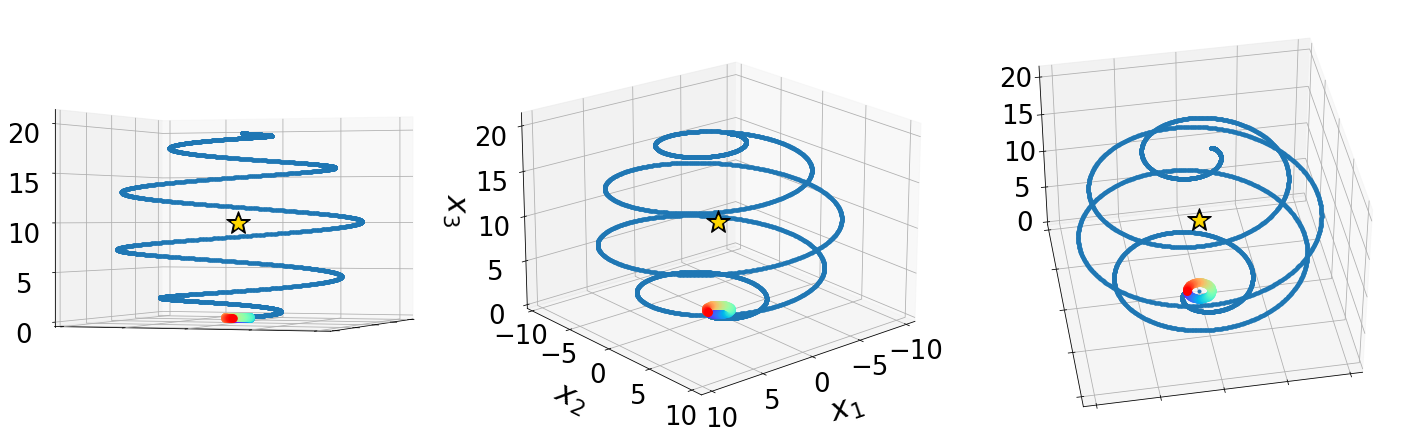}}
    \vspace{-.5em}
    \caption{Plotting the path of the iterates $(\w_t)$ during several cycles on $T=1,000$ tasks (forming a ``denser'' circle than the one in \figref{fig:adversarial-data}).
    Notice how the learner ``misses'' the minimum norm solution $\teacher=(0, 0, 10)$, which is indicated by a star.
    Instead, the iterates seem to converge
    near $2\teacher$, still holding the guarantee in \thmref{thm:optimality_guarantees}.
    %
    %
    Moreover, since at the end of learning it approximately holds that ${\norm{\w_t - \teacher} \approx \norm{2\teacher - \teacher} = 2\norm{\teacher}=20}$, then by the monotonicity that we prove in 
    \corref{cor:monotonicity},
    \ie that ${\norm{\w_t - \teacher} \le \norm{\w_{t'} - \teacher} \le \norm{\w_0 - \teacher},~
    \forall t' \le t}$,
    we understand that the iterates remain at a distance of approximately 10 throughout the entire learning process, thus approximately residing on a sphere centered at $\teacher$ with a radius of $\norm{\teacher}$.
    }
\end{figure}

\vspace{2em}

\begin{figure}[ht!]
    \centering
    {\includegraphics[width=.85\linewidth]{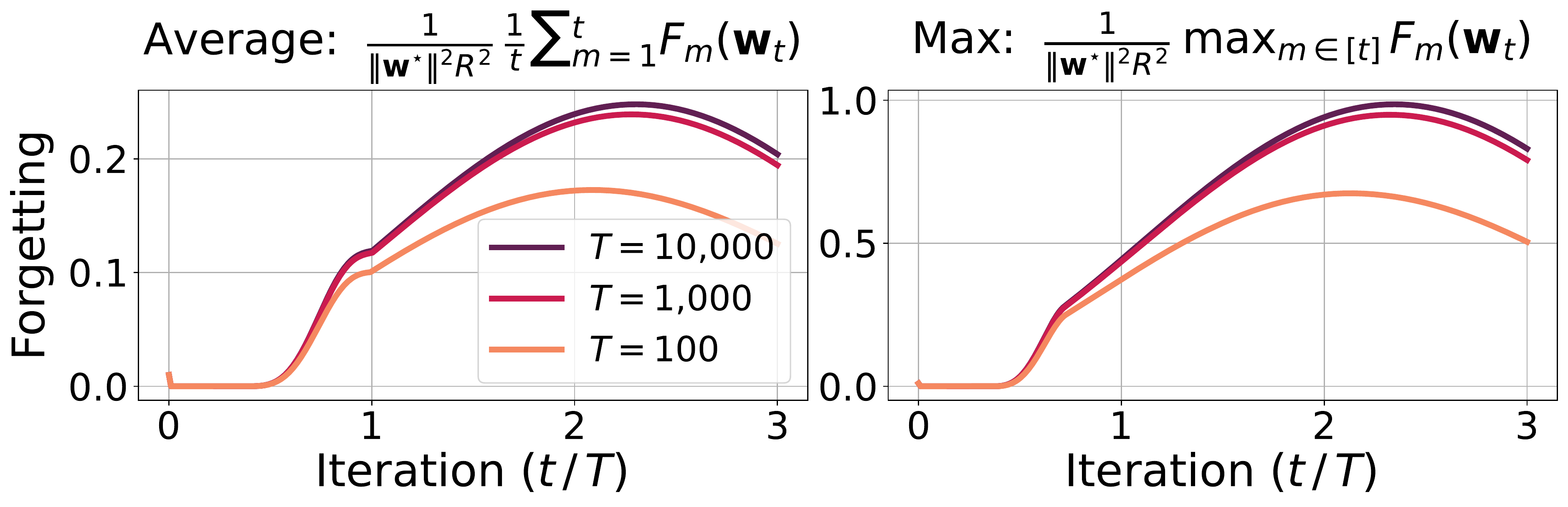}}
                \caption{
            This is a larger version of \figref{fig:adversarial}. The average and maximum forgetting for the adversarial construction. In this experiment, we trained on $3$ cycles of the same $T$ tasks, 
            gradually increasing $T$.
            Notably, after learning $T\to \infty$ jointly separable tasks, the quantities do \emph{not} decay.
            Recall that the maximum forgetting lower bounds the distance to the offline feasible set, which is upper bounded by $\norm{\teacher}^2\!R^2$ (\lemref{lem:euclidean_to_hinge}).
            Hence, the quantities of interest appear to become \emph{arbitrarily bad} at some point during the 3rd cycle of learning, in the sense that they reach $\norm{\teacher}^2\!R^2$.
            }
\end{figure}

\newpage

\subsection{Convergence to the Minimum-Norm Solution (\secref{sec:convergence_to_min})}
\label{app:min_norm}

\begin{recall}[\thmref{thm:optimality_guarantees}]
Any iterate $\w_{t}$ obtained by \procref{proc:adaptive}
holds 
$\norm{\w_{t}}\le 2\norm{\teacher}$,
where $\teacher$ is the minimum-norm offline solution 
(\defref{def:min-norm-solution}).

If additionally, $\w_{t}$ is ``offline''-feasible, 
\ie $\w_{t} \in \teachers$,
then
$$
\norm{\teacher} \le \norm{\w_{t}} \le 2\norm{\teacher}\,.
$$
\end{recall}

\medskip

\begin{proof}
To prove the first part of the theorem, we 
use the triangle inequality,
the monotonicity from \corref{cor:monotonicity} (below),
and the fact that $\w_{0}=\0_d$,
to
show that
\begin{align*}
    \forall t\in\naturals:
    \enskip
    \norm{\w_{t}}
    =
    \norm{\w_{t}-\w_{0}}
    =
    \norm{\w_{t}-\teacher+\teacher-\w_{0}}
    \le 
    {\norm{\w_{0}-\teacher}}
    +
    \norm{\w_{t}-\teacher}
    \stackrel{\ref{cor:monotonicity}}{\le}
    2\norm{\w_{0}-\teacher}
    =
    2\norm{\teacher}
    \,.
\end{align*}

The second part of the theorem follows immediately from the minimality of $\teacher\in \teachers$ 
(\defref{def:min-norm-solution}).
\end{proof}

\newpage

\section{Proofs for Recurring Tasks
(\secref{sec:repetitions})}
\label{app:repetitions_proofs}

\subsection{General Properties}
\label{sec:general-prop}

We start by stating a few general properties of projections and norms that will help us throughout the appendices.

\bigskip

\begin{property}[Projection properties]
\label{prop:projections}

Let $\mP: \reals^{\dim} \to \reals^{\dim}$ be a projection  operator onto a nonempty closed convex set $\mathcal{C}\subseteq\reals^{\dim}$.
\linebreak
Then, $\mP$ holds the following properties: 
\begin{enumerate}
    \item 
    \textbf{Geometric definition.}
    $\mP(\vv)=
    \argmin_{\w\in\mathcal{C}}\norm{\vv-\w}$
    for any $\vv \in \reals^{\dim}$~;
    \item \textbf{Idempotence.}
    $\mP^2\triangleq\mP \circ \mP =\mP$~;
    \item \textbf{Contraction (Non-expansiveness).} 
    For all $\vv, \vu \in \reals^{\dim}$,
    it holds that
    $\norm{\mP(\vv) - \mP(\vu)} \le \norm{\vv-\vu}$.
    \linebreak
    Consequently, when $\0_{\dim}\in \mathcal{C}$, it holds that 
    $\norm{\mP(\vv)} \le \norm{\vv}$
    (see Fact~1.9 in \citet{deutsch2006ratePOCS_II})~;
    \item
    \textbf{The operator $\I - \mP$.}
    Let $\I(\vv)=\vv$ be the identity operator
    and define the operator $(\I-\mP)(\vv)=\vv-\mP(\vv)$.
    \linebreak
    Then, for all $\vv, \vu \in \reals^{\dim}$,
    it holds that
    $\norm{(\I-\mP)(\vv)-(\I-\mP)(\vu)}^2
    \le 
    \norm{\vv-\vu}^2
    -
    \norm{\mP(\vv)-\mP(\vu)}^2
    $.
    \linebreak
    Consequently, when $\0_{\dim}\in \mathcal{C}$,
    it holds that 
    $\norm{\vv-\mP(\vv)}^2 \le 
    \norm{\vv}^2-\norm{\mP(\vv)}^2\le\norm{\vv}^2$
    for any $\vv\in\reals^{\dim}$ 
    \linebreak
    (see Propositions~4.2~and~4.8 in \citet{bauschke2011convexBook}
    and Fact~1.7 in \citet{deutsch2006ratePOCS_II}).
\end{enumerate}
\end{property}

\bigskip
\bigskip

As a result of the non-expansiveness property above, we get the following monotonicity result 
(stronger than \ref{lem:feasible_distance_monotonicity}).
\begin{corollary}
\label{cor:monotonicity}
Let $\w\in\teachers$ be an arbitrary offline solution and $\mP_{t}$ be the projection onto the feasible set of the ${t}$\nth task.
The following monotonicity holds for the iterates $(\w_t)$ of \procref{proc:adaptive}:
$$
\tnorm{\w_{t}-\!\underbrace{\w}_{\in\feasible_t}}
=
\tnorm{\mP_{t}\prn{\w_{t-1}} - \mP_{t}\prn{\w}}
\le
\norm{\w_{t-1}-\w},~~~\forall t\in\naturals\,.
$$
Specifically, this holds for the minimum-norm solution $\teacher$:
\hfill
$
\tnorm{\w_{t}-\teacher}
\le
\norm{\w_{t-1}-\teacher},~~~\forall t\in\naturals\,.
$
\end{corollary}

\bigskip
\bigskip

Next, we state a known property of squared Euclidean norms, stemming from their convexity and from Jensen's inequality.
\begin{claim}
\label{clm:square_ineq}
 For any $m$ vectors $\vv_{1},\dots,\vv_{m}\!\in\!\reals^{\dim}$, it holds that $\left\Vert \vv_{1}\!+\!\dots\!+\!\vv_{m}\right\Vert ^{2}
 \le
 m\!\left(\left\Vert \vv_{1}\right\Vert ^{2}+\cdots+\left\Vert \vv_{m}\right\Vert ^{2}\right)$.
\end{claim}

\bigskip

\begin{lemma}
\label{lem:residual_is_convex}
The residual from task $m$, defined by
$
    \left\Vert
    \vv - \mP_m(\vv)
    \right\Vert ^{2}
    =
    \left\Vert
    (\I-\mP_m)(\vv)
    \right\Vert ^{2}$,
is a convex function in $\vv\in\mathbb{R}^{\dim}$.
\end{lemma}

\begin{proof}
Let $\vu,\vv\in\mathbb{R}^{\dim}$, than for every $\alpha\in[0,1]$ it holds that
\begin{align*}
\left\Vert
(\I-\mP_{m})(\alpha\vu+(1\!-\!\alpha)\vv)
\right\Vert^{2}
&=
\left\Vert
\alpha\vu+(1\!-\!\alpha)\vv - \mP_{m}(\alpha\vu+(1\!-\!\alpha)\vv)
\right\Vert^{2}
\\
\explain{\text{projection properties}}
&\triangleq
\!
\min_{\w\in\feasible_{m}}
\!
\norm{\alpha\vu+(1\!-\!\alpha)\vv - \w}^2
\le
\Big\Vert
\alpha\vu+(1\!-\!\alpha)\vv - 
\underbrace{(\alpha
\overbrace{\mP_{m}(\vu)}^{\in\feasible_{m}}
+
(1\!-\!\alpha)
\overbrace{\mP_{m}(\vv)}^{\in\feasible_{m}}
)}_{
\in \feasible_m\text{, due to its convexity}
}
\Big\Vert^{2}
\\
&
=
\left\Vert
\alpha(\vu-\mP_{m}(\vu))+(1\!-\!\alpha)(\vv-\mP_{m}(\vv))
\right\Vert^{2}
\\
\explain{\text{squared norm is convex}}
&
\le
\alpha
\left\Vert
\vu-\mP_{m}(\vu)
\right\Vert^{2}
+
(1\!-\!\alpha)
\left\Vert
\vv-\mP_{m}(\vv)
\right\Vert^{2}\,.
\end{align*}
\end{proof}

\newpage

\subsection{Linear Regularity of Classification Tasks}

We start by proving an important property of our setting --
that the feasible set has a nonempty interior.
\begin{lemma}
\label{lem:feasible_interior}
Recall our definition of
$
{R\triangleq 
\max_{m\in\cnt{T}}
\max_{(\x,y)\in\dataset_m }
\!\!\norm{\x}}$.
Let $\teacher$ be the minimum-norm solution defined in \defref{def:min-norm-solution}.
Then, the vector $2\teacher$ has a \emph{feasible} ball of radius $\frac{1}{R}$ around it.
More formally,
$$\mathcal{B}\prn{2\teacher, \nicefrac{1}{R}}
\triangleq
\left\{
\w\in\reals^{\dim}
\,\mid\,
\norm{\w-2\teacher}\le \nicefrac{1}{R}
\right\}
\subset
\teachers
\,.$$
\end{lemma}

\begin{proof}
By the definition of feasibility, 
$\forall m\in\cnt{T},\,
\forall (\x,y)\in \dataset_m\!:\,
y {\teacher}^{\top} \x \ge 1
$.
Let $\w$ be an arbitrary vector in $\mathcal{B}\prn{2\teacher, \nicefrac{1}{R}}$.
Clearly, $\w$ can be instead denoted as $\w=2\teacher +\vv$ for some $\vv$ such that $\norm{\vv}\le\nicefrac{1}{R}$.
We then conclude that $\w$ is feasible:
$$
\forall m\in\cnt{T},\,
\forall (\x,y)\in \dataset_m\!:\,\,
y\w^\top \x = 
2\underbrace{y{\teacher}^{\top} \x}_{\ge 1} + y\vv^\top \x 
\ge
2 - \abs{\vv^\top \x }
\ge
2 - \underbrace{\norm{\vv}}_{\le\nicefrac{1}{R}}\underbrace{\norm{\x}}_{\le R}
\ge
1
~~\Longrightarrow~~
\w \in \teachers\,.
$$

\end{proof}

\bigskip

We are now ready to prove the linear regularity, using \lemref{lem:feasible_interior} and techniques from 
\citet{gubin1967methodProjections}
(in their proof of Lemma~5)
and \citet{nedic2010randomPOCS}
(in their proof of Proposition~8).
To be more exact, we prove \emph{bounded} linear regularity (as defined in \citet{bauschke1993convergence,deutsch2008ratePOCS_III}),
which suffices for our needs in this paper.

\begin{recall}[\lemref{lem:regularity}]
At the $t$\nth iteration, the distance to the offline feasible set is tied to the distance to the farthest feasible set of any \emph{specific} task.
Specifically, it holds that $\forall t\!\in\!\naturals$,
\begin{align*}
    \dist^2\!\prn{\w_{t}, \teachers}
    \le
    4\tnorm{\teacher}^2 R^2 
    \max_{m\in\cnt{T}} \dist^2 
    \bigprn{\w_{t}, \feasible_m}
    \,
    .
\end{align*}
\end{recall}

\begin{proof}
We start by showing (bounded) regularity for an arbitrary $\w\in\reals^{\dim}$ (not necessarily an iterate).
\linebreak
We define $\varepsilon \triangleq \max_{m\in\cnt{T}} d \prn{\w, \feasible_m}$
and consider a convex combination between $\w$ and $2\teacher$:
\begin{align*}
\y &= 
\frac{\nicefrac{1}{R}}{\varepsilon + \nicefrac{1}{R}}\w
+
\frac{\varepsilon}{\varepsilon + \nicefrac{1}{R}}2\teacher\,
= 
\frac{\nicefrac{1}{R}}{\varepsilon + \nicefrac{1}{R}}\w
+
\frac{\varepsilon}{\varepsilon + \nicefrac{1}{R}}2\teacher
-
\frac{\nicefrac{1}{R}}{\varepsilon + \nicefrac{1}{R}}\mP_m(\w)
+
\frac{\nicefrac{1}{R}}{\varepsilon + \nicefrac{1}{R}}\mP_m(\w)
\\
&
\triangleq
\frac{\varepsilon}{\varepsilon + \nicefrac{1}{R}}\vect{z}_m
+
\frac{\nicefrac{1}{R}}{\varepsilon + \nicefrac{1}{R}}\mP_m(\w)
\,,
\end{align*}
where we denoted 
$\vect{z}_m \triangleq
2\teacher
+
\frac{\nicefrac{1}{R}}{\varepsilon}
\prn{
    \w
    -
    \mP_m(\w)
}$, for any $m\in\cnt{T}$.

We notice that $\vect{z}_m \in \mathcal{B}\prn{2\teacher,\nicefrac{1}{R}}\stackrel{\text{\ref{lem:feasible_interior}}}{\subset}\teachers\subseteq \feasible_m$,
since $\norm{\vect{z}_m - 2\teacher}
=
\frac{\nicefrac{1}{R}}{\varepsilon}
\underbrace{
\norm{\w - \mP_m (\w)}
}_{
=\dist(\w, \feasible_m)\le \varepsilon
}
\le \nicefrac{1}{R}$.
Therefore, $\y$ is a convex combination of $\vect{z}_m,\mP_m\prn{\w}\in \feasible_m$ and is therefore also contained in $\feasible_m$.
Since this is true $\forall m\in\cnt{T}$, we get that $\y\in\teachers$.
Then, a bounded linear regularity property follows, since
\begin{align*}
\dist(\w, \teachers)
&\le
\norm{\w - \y}
=
\overbrace{\frac{\varepsilon}{\varepsilon+\nicefrac{1}{R}}
}^{ \le R\varepsilon }
\norm{\w - 2\teacher}
\le
\norm{\w - 2\teacher}
R
\varepsilon
=
\norm{\w - 2\teacher}
R
\max_{m\in\cnt{T}} \dist
\bigprn{\w, \feasible_m}\,.
\end{align*}
When $\w=\w_{t}$ is an iterate of the sequential \procref{proc:adaptive}, we can use the monotonicity from \corref{cor:monotonicity} 
(notice that $2\teacher\in\teachers$)
to show that
$\norm{\w_{t} - 2\teacher}\le
\norm{\w_{0} - 2\teacher}
=\norm{2\teacher}
=2\norm{\teacher}$,
and finally, conclude that
\begin{align*}
\dist(\w_{t}, \teachers)
&
\le
2
\norm{\teacher}R
\max_{m\in\cnt{T}} \dist
\bigprn{\w_{t}, \feasible_m}
\,.
\end{align*}
\end{proof}


\newpage

\subsection{Proofs for Cyclic Orderings
(\secref{sec:cyclic})}
\label{app:cyclic_proofs}

\begin{recall}[\lemref{lem:cyclic_lim}] 
Under~a~cyclic ordering
(and the separability assumption~\ref{asm:separability}), 
the iterates converge to a 2-optimal 
${\wlim\in\teachers}$.
That is,
$$
\lim_{k\to \infty} 
\dist(\w_k, \teachers)=0,
~~~
\norm{\teacher} \le \norm{\w_\infty} \le 2\norm{\teacher}
\,.
$$
\end{recall}

\begin{proof}
First, notice that we work in a finite-dimensional Euclidean space $\reals^{\dim}$ and that from Assumption~\ref{asm:separability} we have that
the offline feasibility set,
\ie the convex sets' intersection
$\teachers=\feasible_1\cap\dots\cap\feasible_T$, is nonempty.
Under these conditions, the (weak and strong) convergence of the cyclic iterates $(\w_{t})$ to a vector $\wlim\in\teachers$
can be deduced from the rates we derive (independently of the lemma here) in
Proposition~\ref{prop:cyclic_universal}.
The second part of our corollary stems directly from
our \thmref{thm:optimality_guarantees}.

Similar optimality guarantees for cyclic settings have also been proved in previous papers (\eg Theorem~1 in \citet{gubin1967methodProjections}
and Theorem~2.7 in \citet{deutsch2006ratePOCS_I}).
\end{proof}

\newpage

\begin{recall}[\thmref{thm:2_tasks_cyclic}]
    For ${T\ge2}$
    jointly-separable tasks
    learned cyclically,
    after 
    ${k\!=\!nT}$ iterations ($n$ cycles),
    our quantities of interest 
    (\defref{def:quantities})
    converge linearly as
    \begin{align*}
    \underbrace{
        \max_{m\in\cnt{T}}
        F_m (\w_k)}_{
        \substack{
        \text{Maximum forgetting}
        }
    }
    \le
    \underbrace{
        \max_{m\in\cnt{T}}
        \dist^{2\!}\prn{\w_k, \feasible_{m}}
    }_{
        \substack{
        \text{Maximum distance}
        \\
        \text{to any feasible set}
        }
    }
    R^2
    \le 
    \underbrace{
    \dist^{2\!}\prn{\w_k, \teachers}
    }_{
    \substack{
    \text{Dist.~to}
    \\
    \text{offline feasible set}
    }
    }
    R^2
    &\le
    g(k)
    \norm{\teacher}^2 \!R^2
    %
    \,,
    \end{align*}
    where
    $
    g(k)
    \triangleq
    \begin{cases}
    \enskip\,
    \exp
    \!\prn{-\tfrac{k}{4\tnorm{\teacher}^2 R^2}} & T=2
    \\
    4
    \exp\!\prn{-\tfrac{k}{16T^2\tnorm{\teacher}^2 R^2}} 
    & T\ge 3
    \end{cases}
    $
    and
    $\norm{\teacher}^2 \!R^2\!\ge\! 1$ is the problem complexity (Rem.~\ref{rmk:complexity}).
\end{recall}

\begin{proof}

Here, we exploit the regularity of our setting,
in which the convex sets are defined by the intersections of halfspaces.

\begin{enumerate}
    \item \textbf{Cyclic orderings of $T=2$ tasks.}

    Given the regularity of our problems
    (\lemref{lem:regularity}),
    the required upper bound on $\dist^2\!\prn{\w_k,\teachers}$ 
    is given from existing results on projection algorithms with two closed convex sets (specifically from Theorem~3.12 and Corollary~3.14 in \citet{bauschke1993convergence} 
    and Theorem~2.10 in \citet{bauschke2001projection}).
    However, the indexing in some of these previous works may be confusing. Therefore, for the sake of completeness, we prove it here as well.
    
    Assume w.l.o.g.~that $2\mid k$. Define the projection operator $\mP_{1\cap 2}$ onto the intersection
    $\teachers=\feasible_1\cap \feasible_2$.
    It follows that,
    \begin{align*}
    \frac{1}{4\tnorm{\teacher}^2 R^2}
    \dist^2\!\prn{\w_k,\teachers}
    &\!\!\stackrel{\text{\lemref{lem:regularity}}}{\le}\!\!\!
    \max_{m\in\{1,2\}} \dist^2\!\prn{\w_k,\feasible_m}
    \!\stackrel{\text{cyclic}}{=}\!
    \norm{\w_k-\mP_1 (\w_k)}^2
    \\
    &
    =
    \norm{\w_k-\mP_1 (\w_k) 
    -
    \mP_{1\cap 2}(\w_k)
    +
    \mP_{1\cap 2}(\w_k)
    }^2
    \\
    &
    =
    \norm{\w_k-\mP_1 (\w_k) 
    -
    \mP_{1\cap 2}(\w_k)
    +
    \mP_1 (\mP_{1\cap 2}(\w_k))}^2
    \\
    &
    =
    \norm{
    \prn{\I-\mP_1}(\w_k)
    -
    \prn{\I-\mP_1}(\mP_{1\cap 2}(\w_k))}^2
    \\
    \explain{\text{\propref{prop:projections}}}
    &
    \le
    \norm{\w_k-\mP_{1\cap 2}(\w_k) }^2
    -
    \norm{\mP_1 (\w_k) - \mP_1 (\mP_{1\cap 2}(\w_k))}^2
    \\
    &
    =
    \underbrace{\norm{\w_k-\mP_{1\cap 2}(\w_k) }^2}_{=\dist^2\!\prn{\w_k,\teachers}}
    -
    \underbrace{\norm{\w_{k+1} - \mP_{1\cap 2}(\w_k)}^2}_{
    \ge\,\dist^2\!\prn{\w_{k+1},\teachers} 
    }
    \le
    \dist^2\!\prn{\w_k,\teachers}
    -
    \dist^2\!\prn{\w_{k+1},\teachers}
    \\
    \dist^2\!\prn{\w_{k+1},\teachers}
    &
    \le
    \prn{1-\frac{1}{4\tnorm{\teacher}^2 R^2}}
    \dist^2\!\prn{\w_k,\teachers}
    \le
    \dots
    \le
    \prn{1-\frac{1}{4\tnorm{\teacher}^2 R^2}}^{k+1}
    \underbrace{\dist^2\!\prn{\w_0,\teachers}}_{
    \triangleq \tnorm{\teacher}^2
    }\,.
    \end{align*}
    Finally, we conclude the $T=2$ case by using
    the ordering property of our quantities of interest
    (\lemref{lem:euclidean_to_hinge}),
    and using the algebraic identity stating that
    $\forall z\!\in\!\prn{0,1},k\!>\!0:~
    \prn{1-z}^k \le \exp\prn{-kz}$.
    
    \begin{remark}
    A similar upper bound on
    $\dist^2\!\prn{\w_{k},\teachers}$,
    with a slightly worse rate,
    could be deduced from the stochastic upper bound in
    \thmref{thm:random_rates},
    by using the fact that for $T=2$ tasks, random orderings lead to exactly the same projections as cyclic ones, but slower 
    (the idempotence of the projections implies that applying the same projection many consecutive times in a random ordering has the same effect as applying it once like in a cyclic ordering). 
    \end{remark}

    \item \textbf{Cyclic orderings of $T\ge3$ tasks.}

    Given the regularity of our problems
    (\lemref{lem:regularity}),
    the required upper bound is an almost immediate 
    corollary from    
    Theorem~3.15 in \citet{deutsch2008ratePOCS_III}),
    stating that in our case we have
    \begin{align*}
        \norm{\w_{k}-\wlim}^2
        \le 
        \prn{
        1 - \frac{1}{4T\cdot 4 \norm{\teacher}^2 R^2 }
        }^{k/T}
        \norm{\w_{0} - \wlim}^2
        =
        \prn{
        1 - \frac{1}{16 T\norm{\teacher}^2 R^2 }
        }^{k/T}
        \norm{\wlim}^2\,.
    \end{align*}
    We use 
    \lemref{lem:cyclic_lim} 
    ($\wlim\in\teachers, \norm{\wlim}\le2\norm{\teacher}$)
    and
    the aforementioned algebraic identity,
    and get
    \begin{align*}
        \dist^{2\!}\prn{\w_{k},\teachers}
        \le
        \norm{\w_{k}-\wlim}^2
        \le 
        4
        \norm{\teacher}^2
        \exp\prn{-\frac{k}{16 T^2\norm{\teacher}^2 R^2 }}\,.
    \end{align*}
    \vspace{-2em}
\end{enumerate}
\end{proof}

\newpage

\subsubsection{Detour: Universal bounds for general cyclic Projections Onto Convex Sets settings (POCS)}
\label{app:universal}

First, we establish a lemma equivalent to Lemma~22 in \citet{evron2022catastrophic} for our POCS setting.
They used \emph{linear} projections solely while we analyze more general projection operators.
Thus, while we \emph{closely} follow their statements and proofs, we are required to perform several adjustments to capture the wider family of convex operators.

\bigskip

\begin{lemma}[{Cyclic-case auxiliary bounds}]
\label{lem:dimension_independent}
Let $\mQ_{1},\dots,\mQ_{T}: \reals^{\dim}\to\reals^{\dim}$
be $T$ projection operators
onto the nonempty closed convex sets {$\mathcal{C}_1,\dots, \mathcal{C}_T$} (respectively)
such that $\0_{\dim} \in \mathcal{C}_1 \cap \dots \cap \mathcal{C}_T$.
%
Let $\M=\mQ_{T}\circ\cdots\circ\mQ_{1}: \reals^{\dim}\to\reals^{\dim}$
be the cyclic operator formed by these projections.
Moreover, let $\vv\in \reals^{\dim}$ be an arbitrary vector.
Then:
\begin{enumerate}[label=(\textbf{Lemma~\ref*{lem:dimension_independent}\alph*)},leftmargin=2.5cm]\itemsep6pt
    \item 
    \label{lem:bound_on_projection_to_m_task}
    For any $m\in\cnt{T-1}$, it holds that
    \hfill
    $
    \dist^2\!\prn{\vv, \mathcal{C}_m}
    =
    \left\Vert
    \vv-\mQ_{m}(\vv)
    \right\Vert ^{2}
        \le 
        m
        \left(\left\Vert \vv\right\Vert ^{2}-
        \left\Vert \M(\vv)\right\Vert ^{2}\right)~;$
    \item 
    \label{lem:bound_of_one_cycle}
    It holds that
    \hfill
    $\left\Vert \vv - \M(\vv)\right\Vert ^{2}
        \le 
        T
        \!
        \left(\left\Vert \vv\right\Vert ^{2}-\left\Vert \M(\vv)\right\Vert ^{2}\right)~\!;$
    \item 
    \label{lem:first_task_forgetting}
    After $n\ge 1$ cycles it holds that
    \hfill
    $
        \left\Vert \M^{n}(\vv)\right\Vert ^{2}
        -
        \left\Vert \M^{n+1}(\vv)\right\Vert ^{2}
            \le 
        2 \norm{\M^{n}(\vv) - \M^{n+1}(\vv)}
        \norm{\vv}$~;
    \item \label{lem:m_task_forgetting}
    For any $m\in\cnt{T-1}$,
    after $n\ge 1$ cycles it holds that
    \begin{align*}
        \dist^2\!\prn{\M^{n}(\vv), \mathcal{C}_m}
        =
        \norm{\M^{n}(\vv)-\mQ_m(\M^{n}(\vv))}^2 
        \le 
        2m 
        \norm{\M^{n}(\vv) - \M^{n+1}(\vv)}
        \norm{\vv}
        ~;
    \end{align*}
    \item 
    \label{lem:T_over_n_norm}
    For any number of cycles $n\ge 1$, it holds that
    \hfill
    $\norm{\M^{n-1}(\vv) - \M^{n}(\vv)} 
    \le \sqrt{\nicefrac{T}{n}} \norm{\vv}$~.
\end{enumerate}
\end{lemma}

\bigskip

We will prove this lemma after stating a few additional results including our universal bounds.

\bigskip
\bigskip

\begin{remark}[Translates of convex sets]
\label{rmk:translations}
For simplicity, many of our results (\eg \lemref{lem:dimension_independent} above)
are derived using
nonempty closed convex sets having $\0_{\dim}$ in their intersection.
We wish to apply these results to our feasible sets 
$\feasible_1,\dots,\feasible_{T}$ from Assumption~\ref{asm:separability},
but their nonempty intersection $\teachers$
does not contain $\0_{\dim}$.
As a remedy, it is common to translate the entire space by some feasible solution (\eg see \citet{deutsch2006ratePOCS_I}).

We take this approach and translate our space by $-\teacher$ (where $\teacher\in\feasible$ is the minimum-norm offline solution defined in \defref{def:min-norm-solution}).
In turn, the resulting feasible sets 
$C_{m}\triangleq \feasible_{m}\!-\!\teacher,~\forall m\in\cnt{T}$, now have an equivalent intersection $\teachers\!-\!\teacher$ that \emph{does} contain $\0_{\dim}$.
Importantly, this simple translation does not change the properties of the space we work in or the ones of any projection that we apply. 
In particular, 
we can apply either ``type'' of these projections interchangeably, in the sense that:
$$
\mP_{m}(\w)
\triangleq
\mP_{\feasible_m}(\w)
=
\mP_{\mathcal{C}_m}(\w\!-\!\teacher)+\teacher
\triangleq
\mQ_{m}(\w-\teacher)+\teacher\,.$$
This also holds recursively,
\eg
\hfill
$
\w_2 = \mP_{2}(\mP_{1}(\w_0))
=
\mP_{2}(\mQ_{1}(\w_0-\teacher)+\teacher)
=
\mQ_{2}(\mQ_{1}(\w_0-\teacher))+\teacher$.
\linebreak
Moreover, our translation preserves all distances,
such that for instance 
$$\forall \w\!:~
\norm{\w-\mP_m(\w)}
=
\dist\!\prn{\w, \feasible_m}
=
\dist\!\prn{\w\!-\!\teacher, \mathcal{C}_m}
=
\norm{\w-\teacher-\mQ_m(\w-\teacher)}\,.
$$
\end{remark}

\bigskip
\bigskip
\bigskip

We will also need the following lemma to prove the universal result for the cyclic $T=2$ case.

\medskip

\begin{lemma}
\label{lem:two_tasks_monotonicity}
Let $\mathcal{C}_1, \mathcal{C}_2$
be two closed convex subsets of $\reals^{\dim}$ with a nonempty intersection,
and let $\prn{\w_t}$ be a sequence of iterates induced by cyclic projections onto  $\mathcal{C}_1, \mathcal{C}_2$
starting from an arbitrary $\w_0\in\reals^{\dim}$.
    Then, the projection ``residuals'' are monotonically decreasing. 
    That is, For any iteration $t\in\naturals$ we have
\begin{align*}
\left\Vert \w_{t+1}-\w_{t}\right\Vert \le\left\Vert \w_{t}-\w_{t-1}\right\Vert \,.
\end{align*}
\end{lemma}

\begin{proof}
Assume w.l.o.g. that $2\mid t$. 
By the definition of projections, 
\begin{align*}
\norm{\w_{t+1}-\w_{t}} 
=
\norm{\mP_1(\w_t)-\w_{t}}
=
\min_{\w\in{\feasible_{1}}}
\Vert\w-\!\underbrace{\w_{t}}_{\in{\feasible_{2}}}\!\Vert
\le
\Vert
\!\underbrace{\w_{t-1}}_{\in{\feasible_{1}}}\!
-
\w_{t}
\Vert
\,.
\end{align*}
\end{proof}

\newpage

Now, we recall and prove our main result for this section (using the above lemmas). 
Afterward, we will prove \lemref{lem:dimension_independent}.
Our result here generalizes important parts of Theorems 10 and 11 in \citet{evron2022catastrophic} and extends these results from projections onto closed subspaces to projections onto general convex sets.

\begin{recall}[Proposition~\ref{prop:cyclic_universal}]
    Let $\feasible_1, \dots, \feasible_{T}$ be closed convex sets with nonempty intersection $\teachers$.
    \linebreak
    Let
    $\w_{k}=(\mP_T\circ\cdots\circ\mP_1)^{n}
    (\w_0)$
    be the iterate after $k\!=\!nT$ iterations ($n$ cycles)
    of cyclic projections onto these convex sets.
    Then, the maximal distance to any (specific) convex set,
    is upper bounded \emph{universally} as,
    \begin{align*}
        \text{For $T=2$: }~~
        &
        \max_{m\in\cnt{T}} 
        \dist^{2\!}\prn{\w_k, \feasible_{m}}
        \le
        \frac{1}{k+1}\,
        \dist^{2\!}\prn{\w_{0}, \teachers}
        \,,
        \\
        \text{For $T\ge3$: }~~
        &
        \max_{m\in\cnt{T}} 
        \dist^{2\!}\prn{\w_k, \feasible_{m}}
        \le
        \frac{2T^2}{\sqrt{k}}\,
        \dist^{2\!}\prn{\w_{0}, \teachers}
        \,.
    \end{align*}
\end{recall}

\begin{proof}
We divide our proof into two parts.

\begin{enumerate}
    \item \textbf{Cyclic two tasks ($T=2$)}
The proposition deals with cases where $k=nT=2n$ (thus $2\mid k$). 
We get 
\begin{align*}
\max
\Big\{ 
\dist\left(\w_{k},{\feasible_{1}}\right),
\dist\big(\underbrace{\w_{k}}_{\in{\feasible_{2}}},{\feasible_{2}}\big)
\Big\}
=
\dist\left(\w_{k},{\feasible_{1}}\right)
=
\min_{\w\in{\feasible_{1}}}\left\Vert \w-\w_{k}\right\Vert \triangleq\left\Vert \w_{k+1}-\w_{k}\right\Vert \,.
\end{align*}

We thus focus on $\left\Vert \w_{k+1}-\w_{k}\right\Vert$  and bound it using all of the above, using ideas from \citet{evron2022catastrophic}:
\begin{align*}
\left\Vert \w_{k+1}-\w_{k}\right\Vert ^{2}	
\stackrel{\text{\lemref{lem:two_tasks_monotonicity}}}{\le}
&
\frac{1}{k+1}
\sum_{t=0}^{k}\left\Vert \w_{t}-\w_{t+1}\right\Vert ^{2}
=
\frac{1}{k+1}
\sum_{t=0}^{k}
\left\Vert \w_{t}-\w_{t+1} - \teacher+\teacher\right\Vert ^{2}	
 \\
 =
 ~~~
 &
\frac{1}{k+1}
\sum_{t=0}^{k}
\tnorm{
\prn{\I-\mP_{t+1}}\prn{\w_t}
-
\prn{\I-\mP_{t+1}}\prn{\teacher}
}^2
 \\
 \explain{\text{\propref{prop:projections}}}
 ~~
 \le
 ~~~
 &
\frac{1}{k+1}
\sum_{t=0}^{k}
\prn{
 \tnorm{\w_t-\teacher}^2
 -
 \tnorm{\mP_{t+1}\w_t-\mP_{t+1}\teacher}^2
 }
 \\
 ~~
=
 ~~~
 &
\frac{1}{k+1}
\sum_{t=0}^{k}
\prn{
 \tnorm{\w_t-\teacher}^2
 -
 \tnorm{\w_{t+1}-\teacher}^2
 }
 \\
 \explain{\text{telescoping}}
 ~~
 =
 ~~~
&
\frac{1}{k+1}\left(\left\Vert \w_{0}-\w^{\star}\right\Vert ^{2}-{\left\Vert \w_{k+1}-\w^{\star}\right\Vert ^{2}}\right)\le\frac{1}{k+1}\Vert
\w_{0}
-
\w^{\star}\Vert^{2}
\triangleq
\frac{1}{k+1}\dist^{2\!}\prn{\w_{0}, \teachers}\,.
\end{align*}

\item \textbf{Cyclic $T\ge3$ tasks}

Using \lemref{lem:dimension_independent} and the translations from Remark~\ref{rmk:translations}, it follows that
\begin{align*}
\begin{split}
\dist^2\!\prn{\w_k, \feasible_m}
\stackrel{\text{\ref{rmk:translations}}}{=}
\dist^2\!\prn{\w_k-\teacher, \mathcal{C}_m}
&
=
\norm{\prn{\w_k-\teacher}
-
\mQ_m(\w_k-\teacher)}^2
\\
\explain{\text{Remark~\ref{rmk:translations}}}
& = 
\norm{\M^{n}(\w_0-\teacher)
-
\mQ_m(\M^{n}(\w_0-\teacher))}^2
\\
\explain{\text{\ref{lem:m_task_forgetting}}}
&
\le
2m
\norm{\M^{n}(\w_0-\teacher)
-
\M^{n+1}(\w_0-\teacher)}
\norm{\w_0-\teacher}
\\
\explain{\text{\ref{lem:T_over_n_norm}}}
&
\le
2m\sqrt{\frac{T}{n+1}} \norm{\w_0-\teacher}^2
\\
&
\le
2T\sqrt{\frac{T}{n}} \norm{\w_0-\teacher}^2
=
\frac{2T^2}{\sqrt{k}}\norm{\w_0-\teacher}^2
\triangleq
\frac{2T^2}{\sqrt{k}}\dist^{2\!}\prn{\w_{0}, \teachers}
~.
\end{split}
\end{align*}
\end{enumerate}
\end{proof}

\pagebreak

\paragraph{Proof for \lemref{lem:dimension_independent}}

\medskip

\begin{proof}[Proof for \ref{lem:bound_on_projection_to_m_task}] 
For $m=1$, 
using the fact that $\0_{\dim}\in \mathcal{C}_1\cap\dots\cap \mathcal{C}_T$, \propref{prop:projections} immediately gives
\begin{align*}
    \left\Vert
    \vv-\mQ_{m}(\vv)
    \right\Vert ^{2}
    =
    \left\Vert
    \vv-\mQ_{1}(\vv)
    \right\Vert ^{2}
    \le
    \left\Vert \vv\right\Vert ^{2}
    -
    \left\Vert \mQ_{1}(\vv)\right\Vert ^{2}
    \stackrel{\text{contraction}}{\le}
    1\cdot
    \left(
    \left\Vert \vv\right\Vert ^{2}
    -
    \left\Vert \mQ_{T}\circ\cdots\circ\mQ_{1}(\vv)\right\Vert ^{2}
    \right)~.
\end{align*}

Now we will prove the case when  $m=2,3,\dots, T-1$.
\begin{enumerate}
\item 
First, let $\I(\vv)=\vv$ be the identity operator
and define the (``part-cyclic'') operator
$\M_\ell
\triangleq
\mQ_{\ell}
\circ\cdots\circ
\mQ_{1}
,\,\forall \ell\in\cnt{T}$.
\linebreak
We show recursively that
(for any $m\ge 1$)
\begin{align*}
\begin{split}
\I &=
{
    \left(
    \I
    -
    \mQ_{1}
    \right)
    +
    \left(
    \mQ_{1}
    -
    \M_2
    \right)
    +
    \left(
    \M_2
    -
    \M_3
    \right)
    +
    \dots
    +
    \left(
    \M_{m-1}
    -
    \M_m
    \right)
    +
    \M_{m}
}
\\ 
&=
{
    \left(
    \I
    -
    \mQ_{1}
    \right)
    +
    \left(
    \mQ_{1}
    -
    {\mQ_2
    \circ
    \M_1}
    \right)
    +
    \left(
    \M_2
    -
    \mQ_3
    \circ
    \M_2
    \right)
    +
    \dots
    +
    \left(
    \M_{m-1}
    -
    \mQ_m
    \circ
    \M_{m-1}
    \right)
    +
    \M_{m}
}
\\
&=
{
    \left(
    \I
    -
    \mQ_{1}
    \right)
    +
    \left(\I-\mQ_{2}\right)
    \circ
    \M_{1}
    +
    \left(\I-\mQ_{3}\right)
    \circ\M_2
    +
    \dots
    +
    \left(\I-\mQ_{m}\right)
    \circ\M_{m-1}
    +
    \M_m
}
\\
&=
\left(
\I
-
\mQ_{1}
\right)
+
\M_m
+
\sum_{\ell=1}^{m-1}
\left(
\I
-
\mQ_{\ell+1}
\right)
\circ
\M_{\ell}~.
\end{split}
\end{align*}
Equivalently, using similar steps we have
$\mQ_m = 
\left(
\mQ_m
-
\mQ_m\circ\mQ_{1}
\right)
+
\mQ_m\circ
\M_m
+
\sum_{\ell=1}^{m-1}
\left(
\mQ_m
-
\mQ_m\circ\mQ_{\ell+1}
\right)
\circ
\M_{\ell}$.
\item Subtracting both of the equations above, we get
\begin{align*}
\left(\I-\mQ_{m}\right)
&=
\left(
\I
-
\mQ_{1}
\right)
-
\left(
\mQ_m
-
\mQ_m\scirc\mQ_{1}
\right)
+
\M_m
-
\mQ_m\scirc
\M_m
+
\sum_{\ell=1}^{m-1}
\left(
\I
-
\mQ_{\ell+1}
\right)
\scirc
\M_{\ell}
-
\left(
\mQ_m
-
\mQ_m\scirc\mQ_{\ell+1}
\right)
\scirc
\M_{\ell}
\\
&=
\left(
\I
-
\mQ_{m}
\right)
-
\left(
\I
-
\mQ_{m}
\right)
\scirc
\mQ_{1}
+
\left(
\I-\mQ_m
\right)
\scirc
\M_m
+
\sum_{\ell=1}^{m-1}
\left(
\I
-
\mQ_{m}
\right)
\scirc
\M_{\ell}
-
\left(
\mQ_{\ell+1}
-
\mQ_m\scirc\mQ_{\ell+1}
\right)
\scirc
\M_{\ell}
\\
&=
\left(
\I
-
\mQ_{m}
\right)
-
\left(
\I
-
\mQ_{m}
\right)
\scirc
\mQ_{1}
+
\cancel{
\left(
\I-\mQ_m
\right)
\scirc
\mQ_m\scirc\M_{m-1}
}
+
\!
\sum_{\ell=1}^{m-1}
\left(
\I
-
\mQ_{m}
\right)
\scirc
\M_{\ell}
-
\left(
\I
-
\mQ_{m}
\right)
\scirc
\mQ_{\ell+1}\scirc\M_{\ell}
~,
\end{align*}
where the third term is canceled since
$
\left(
\I-\mQ_m
\right)
\circ
\mQ_m=\mQ_m-\mQ_m^2=\0$ due to the idempotence of projection operators.

\item
Finally, we use the above and \clmref{clm:square_ineq} to show that
\begin{align*}
\begin{split}
    &\norm{
    \vv-\mQ_{m}(\vv)}^2
    =
    \norm{
    (\I-\mQ_{m})(\vv)}^2
    \\
    &=
    \norm{
    (\left(
    \I
    -
    \mQ_{m}
    \right)
    -
    \left(
    \I
    -
    \mQ_{m}
    \right)
    \circ
    \mQ_{1})(\vv)
    +
    \tsum_{\ell=1}^{m-1}
    \left(\left(
    \I
    -
    \mQ_{m}
    \right)
    \circ
    \M_{\ell}
    -
    \left(
    \I
    -
    \mQ_{m}
    \right)
    \circ
    \mQ_{\ell+1}\circ\M_{\ell}
    \right)(\vv)
    }^2
    \\
    &
    \le 
    m
    \left(
    \norm{
    \left(
    \I
    -
    \mQ_{m}
    \right)
    (\vv)
    -
    \left(
    \I
    -
    \mQ_{m}
    \right)
    \left(
    \mQ_{1}(\vv)
    \right)
    }^2
    +
    \sum_{\ell=1}^{m-1}
    \norm{
    \left(
    \I
    -
    \mQ_{m}
    \right)
    \left(
    \M_{\ell}(\vv)
    \right)
    -
    \left(
    \I
    -
    \mQ_{m}
    \right)
    \left((\mQ_{\ell+1}\circ\M_{\ell})(\vv)\right)
    }^2
    \right)~.
    \end{split}
    \end{align*}
    And since
    according to \propref{prop:projections}
    we have
    $\forall\vv,\vu\in\reals^{\dim}\!:
    \tnorm{
    {\left(\I-\mQ_m\right)(\vv)
    \!-\!
    \left(\I-\mQ_m\right)(\vu)}
    }^2
    \le
    \norm{\vv-\vu}^2
    $,
    we get
    \begin{align*}
    \begin{split}
    \norm{
    \vv-\mQ_{m}(\vv)}^2
    &
    \le 
    m
    \left(
    \norm{\vv-\mQ_{1}(\vv)}^2
    +
    \sum_{\ell=1}^{m-1}
    \norm{\M_{\ell}(\vv)-
    \mQ_{\ell+1}\left(\M_{\ell}(\vv)\right)}^2
    \right)
    \\
    \explain{\text{\propref{prop:projections}}}
    &\le 
    m
    \biggprn{
    \norm{\vv}^2
    -
    \tnorm{
    \underbrace{
    \mQ_{1}(\vv)}_{=\M_1(\vv)}
    }^2
    +
    \sum_{\ell=1}^{m-1}
    \Bigprn{
    \norm{\M_{\ell}(\vv)}^2
    -
    \tnorm{
    \underbrace{\mQ_{\ell+1}\left(\M_{\ell}\right(\vv))}_{=\M_{\ell+1}(\vv)}}^2
    }
    }
    \\
    \explain{\text{telescoping}}
    &=
    m
    \left(
    \norm{\vv}^2
    -
    \norm{\M_m(\vv)}^2
    \right)
    =
    m
    \left(
    \norm{\vv}^2
    -
    \norm{
    \left(\mQ_{m}\circ\cdots\circ\mQ_{1}\right)(\vv)
    }^2
    \right)
    \\
    \explain{\text{contraction}}
    &
    \le
    m
    \left(
    \norm{\vv}^2
    -
    \norm{
    \left(\mQ_{T}\circ\cdots\circ\mQ_{1}\right)(\vv)
    }^2
    \right)
    ~.
\end{split}
\end{align*}
\vspace{-3em}
\end{enumerate}
\end{proof}

\newpage

\begin{proof}[Proof for \ref{lem:bound_of_one_cycle}] 
Above we defined the identity operator $\I(\vv)=\vv$ 
and the operator
$\M_\ell
\triangleq
\mQ_{\ell}
\circ\cdots\circ
\mQ_{1}$
and showed recursively that
$\I =
\left(
\I
-
\mQ_{1}
\right)
+
\M_m
+
\sum_{\ell=1}^{m-1}
\left(
\I
-
\mQ_{\ell+1}
\right)
\circ
\M_{\ell}$. 
Clearly, this also yields:
\begin{align*}
\begin{split}
\I - 
\M
=
\I - 
\M_T
=
\left(
\I
-
\mQ_{1}
\right)
+
\sum_{\ell=1}^{T-1}
\left(
\I
-
\mQ_{\ell+1}
\right)
\circ
\M_{\ell}~.
\end{split}
\end{align*}
Then, we prove our lemma:
\begin{align*}
\begin{split}
\norm{\vv
    -
    \mQ_{T}
    \circ\cdots\circ
    \mQ_{1}
    (\vv)
}^2
&=
\norm{\vv - \M(\vv)
}^2
=
\norm{
\prn{\left(
\I
-
\mQ_{1}
\right)
+
\tsum_{\ell=1}^{T-1}
\left(
\I
-
\mQ_{\ell+1}
\right)
\circ
\M_{\ell}}(\vv)
}^2
\\
&
=
\norm{
\left(
\I
-
\mQ_{1}
\right)(\vv)
+
\tsum_{\ell=1}^{T-1}
\left(
\I
-
\mQ_{\ell+1}
\right)
\prn{
\M_{\ell}(\vv)
}
}^2
\\
\explain{\text{\clmref{clm:square_ineq}}}
&\le
T
\left(
\norm{
    \left(
    \I
    -
    \mQ_{1}
    \right)(\vv)
}^2
+
\tsum_{\ell=1}^{T-1}
\norm{
\left(
\I
-
\mQ_{\ell+1}
\right)
\prn{
\M_{\ell}(\vv)
}
}^2
\right)
\\
\explain{\text{\propref{prop:projections}}}
&\le
T
\biggprn{
\norm{\vv}^2
-
\tnorm{
    \underbrace{
    \mQ_{1}(\vv)}_{=\M_1(\vv)}
    }^2
+
\tsum_{\ell=1}^{T-1}
\Bigprn{
\norm{\M_{\ell}(\vv)}^2
-
\tnorm{\underbrace{\mQ_{\ell+1}\prn{\M_{\ell}(\vv)}}_{
=\M_{\ell+1}(\vv)
}}^2
}
}
\\
\explain{\text{telescoping}}&
=
T
\left(
\norm{\vv}^2
-
\norm{
    \M_T(\vv)
}^2
\right)
=
T
\left(
\norm{\vv}^2
-
\norm{
    \left(
    \mQ_{T}
    \circ\cdots\circ
    \mQ_{1}
    \right)(\vv)
}^2
\right)~.
\end{split}
\end{align*}
\end{proof}

\medskip

\begin{proof}[Proof for \ref{lem:first_task_forgetting}] 
We focus on the cyclic operator $\M=\mQ_{T}\circ\cdots\circ\mQ_{1}$
and remind that since each $\mQ_\ell$ is a non-expansive operator, their composition $\M$ is also a non-expansive operator.
Hence, $
\left\Vert \M^{n}(\vu)\right\Vert_{2}^{2}
-
\left\Vert \M^{n+1}(\vu)\right\Vert_{2}^{2}
\ge 0$.

Then, we show that
\begin{align*}
&
\left\Vert \M^{n}(\vv)\right\Vert _{2}^{2}-\left\Vert \M^{n+1}(\vv)\right\Vert _{2}^{2}
=
\underbrace{\left(\M^{n}(\vv)\right)^{\top}\M^{n}(\vv)-\left(\M^{n+1}(\vv)\right)^{\top}\M^{n+1}(\vv)}_{\ge0}
 \\
\explain{\forall z\ge0:~z\,=\,\norm{z}}
 &
 =
 \tnorm{\left(\M^{n}(\vv)\right)^{\top}\M^{n}(\vv)-\left(\M^{n+1}(\vv)\right)^{\top}\M^{n+1}(\vv)
 +
 \underbrace{\left(\M^{n}(\vv)\right)^{\top}\M^{n+1}(\vv)-\left(\M^{n}(\vv)\right)^{\top}\M^{n+1}(\vv)}_{=0}
 }
\\
&
 =
 \norm{\left(\M^{n}(\vv)\right)^{\top}\M^{n}(\vv)-\left(\M^{n}(\vv)\right)^{\top}\M^{n+1}(\vv)+\left(\M^{n}(\vv)\right)^{\top}\M^{n+1}(\vv)-\left(\M^{n+1}(\vv)\right)^{\top}\M^{n+1}(\vv)
 }
 \\
 &
 =\left\Vert \left(\M^{n}(\vv)\right)^{\top}\left(\M^{n}(\vv)-\M^{n+1}(\vv)\right)
 +
 \left(\M^{n}(\vv)-\M^{n+1}(\vv)\right)^{\top}
 \M^{n+1}(\vv)\right\Vert 
 \\
\explain{\text{triangle inequality}}
&
\le
\left\Vert \left(\M^{n}(\vv)\right)^{\top}\left(\M^{n}(\vv)-\M^{n+1}(\vv)\right)\right\Vert 
+
\left\Vert 
\left(\M^{n}(\vv)-\M^{n+1}(\vv)\right)^{\top}
\M^{n+1}(\vv)\right\Vert 
\\
\explain{\text{Cauchy-Schwarz}}	
&
\le\underbrace{\left\Vert \M^{n}(\vv)\right\Vert }_{\le\norm{\vv}}
\left\Vert \M^{n}(\vv)-\M^{n+1}(\vv)\right\Vert 
+
\left\Vert \M^{n}(\vv)-\M^{n+1}(\vv)\right\Vert \underbrace{\left\Vert \M^{n+1}(\vv)\right\Vert }_{\le\norm{\vv}}
\\
\explain{\text{contraction}}
&
\le
2\left\Vert \M^{n}(\vv)-\M^{n+1}(\vv)\right\Vert 
\norm{\vv}~.
\end{align*}
\end{proof}

\medskip

\begin{proof}[Proof for \ref{lem:m_task_forgetting}] 
This stems directly from
\ref{lem:bound_on_projection_to_m_task}
and
\ref{lem:first_task_forgetting}:
\begin{align*}
\norm{\M^{n}(\vv)-\mQ_m(\M^{n}(\vv))}^2 
\le
m
\prn{
\norm{\M^{n}(\vv)}^2
-
\norm{\M^{n+1}(\vv)}^2
}
\le
2m \norm{\M^{n}(\vv)-\M^{n+1}(\vv)}\norm{\vv}~.
\end{align*}
\end{proof}

\medskip

\begin{proof}[Proof for \ref{lem:T_over_n_norm}] 
First, we use \ref{lem:bound_of_one_cycle} 
to show that
\begin{align*}
\begin{split}
\sum_{t=0}^{n-1} \norm{\left(\M^t-\M^{t+1}\right)(\vv)}^2
&=
\sum_{t=0}^{n-1} 
\norm{
\M^t(\vv)
-
\M\prn{\M^t(\vv)}
}^2
\\
&
\le
T\sum_{t=0}^{n-1}
\left(
\norm{\M^t(\vv)}^2
-
\norm{\M^{t+1}(\vv)}^2
\right)
=
T
\left(
\norm{\vv}^2
-
\norm{\M^{n}\vv}^2
\right)
\le
T \norm{\vv}^2~.
\end{split}
\end{align*}
Then,
we use the non-expansiveness property (\propref{prop:projections}) 
again to show that
$$\norm{\left(\M^t-\M^{t+1}\right)(\vv)}^2
=
\norm{
\M\left(\M^{t-1}(\vv)\right)
-
\M\left(\M^{t}(\vv)\right)
}^2
\le
\norm{
\M^{t-1}(\vv)
-
\M^{t}(\vv)
}^2
=
\norm{\left(\M^{t-1}-\M^{t}\right)(\vv)}^2$$
and conclude that the series 
$\left(\norm{\left(\M^t-\M^{t+1}\right)\vv}^2\right)_t$
is monotonically non-increasing.

All of the above means that $\forall{\vv \in \reals^{\dim}}$,
it holds that
\begin{align*}
    \norm{\left(\M^{n-1}-\M^{n}\right)(\vv)}^2
    &=\!\!
    \min_{t=0,\dots,n-1} \norm{\left(\M^t-\M^{t+1}\right)(\vv)}^2
    \le 
    \frac{1}{n}
    \sum_{t=0}^{n-1} \norm{\left(\M^t-\M^{t+1}\right)(\vv)}^2
    \le
    \frac{T}{n} \norm{\vv}^2~.
\end{align*}
\end{proof}

\newpage

\subsection{Proofs for Random Orderings (\secref{sec:random_ordering})}
\label{app:random_proofs}

Here, we will make use of the following result
from the POCS literature, 
proven in \citet{nedic2010randomPOCS}.
\begin{quote}
\vspace{-.7cm}
\item
\textbf{Proposition~8 from \citet{nedic2010randomPOCS}.}
Let $\feasible_1,\dots,\feasible_T$ be closed convex sets
with a non-empty intersection $\teachers$.
\linebreak
Starting from an arbitrary deterministic $\w_0$,
let $\w_1, \dots, \w_k$ be $k$ iterates obtained by iteratively projecting onto the $T$ sets according to an ordering $\tau$ sampled from an arbitrary i.i.d.~distribution
${p(\tau(\ell)=m)
=
p(\tau(\ell')=m),
~~\forall \ell,\ell'\in\cnt{k},~\forall m\in\cnt{T}}$,
with all non-zero probabilities, 
\ie
${p_{\min} \triangleq \min_{m\in\cnt{T}} p(\tau(\cdot)=m)>0}$.
Assume that $\teachers$ has a non-empty interior,
\ie $\mathcal{B}(\overline{\w}, \delta)\subseteq\teachers$
for some $\overline{\w}$.
Then we have
$$
\expectation_{\tau}\left[{\dist^2(\w_k, \teachers)}\right]
\le
\prn{1-p_{\min} \tfrac{\delta^2}{\tnorm{\w_0 - \overline{\w}}^2}}^k 
\dist^2\left(\w_0, \teachers\right)\,.
$$
\end{quote}

\begin{recall}[\thmref{thm:random_rates}]
 For $T$ jointly-separable tasks
learned in a random ordering,
our quantities of interest 
(\defref{def:quantities})
converge linearly as
\begin{align*}
\expectation_{\tau}\!
\big[
\underbrace{
    \max_{m\in\cnt{T}}
    F_m (\w_k)}_{
    \substack{
    \text{Maximum forgetting}
    }
}
\big]
\le
\expectation_{\tau}\!
\big[
\underbrace{
    \max_{m\in\cnt{T}}
    \dist^{2\!}\prn{\w_k, \feasible_{m}}
}_{
    \substack{
    \text{Maximum distance}
    \\
    \text{to any feasible set}
    }
}
\big]
R^2
\le 
\expectation_{\tau}\!
\big[
\underbrace{
\dist^{2\!}\prn{\w_k, \teachers}
}_{
\substack{
\text{Dist.~to}
\\
\text{offline feasible set}
}
}
\big]
R^2
\le
\exp\!
\prn{
-\tfrac{k}{4T\tnorm{\teacher}^2 R^2}
}
\norm{\w^{\star}}^2 \!R^2
\,,
\end{align*}
where $\norm{\w^{\star}}^2 \!R^2\!\ge\! 1$ is the problem complexity (Rem.~\ref{rmk:complexity}).
\end{recall}

\begin{proof}
We start by applying the above Proposition~8 from \citet{nedic2010randomPOCS}
with the non-empty interior from \lemref{lem:feasible_interior} and our specific initialization and ``uniform'' distribution,
that is,
$$\overline{\w}=2\teacher,\,\delta=\tfrac{1}{R},\,
p_{\min} = \tfrac{1}{T},
\,\w_0=\0_{\dim},\,
\dist(\w_0,\teachers)=\tnorm{\teacher}\,.$$
We get
$$
\expectation_{\tau}\left[{\dist^2(\w_k, \teachers)}\right]
\le
\prn{1-\frac{1}{4T\norm{\teacher}^2R^2}}^k 
\underbrace{\dist^2\left(\w_0, \teachers\right)}_{
=\norm{\teacher}^2
}
\le
\norm{\w^{\star}}^2
\exp\!
\prn{
-\frac{k}{4T\tnorm{\teacher}^2 R^2}
}\,,
$$
where we used the algebraic identity
$\forall z\!\in\!\prn{0,1},k\!\ge\!1:~
\prn{1-z}^k \le \exp\prn{-kz}
$.

We complete our proof by noticing that according to
\lemref{lem:euclidean_to_hinge}~and~\thmref{thm:optimality_guarantees},
we have (for any ``instantiation'' of $\tau$ and the sequence $(\w_t)$ it induces),
$$
0 \le 
\max_{m\in\cnt{T}}
F_m\prn{\w_k}
\le 
\max_{m\in\cnt{T}}
\dist^{2\!}\prn{\w_k, \feasible_m}
R^2
\le
\dist^{2\!}\prn{\w_k, \teachers}
R^2
< \infty
\,,
~~~
\forall k\in\naturals\,,
$$
and so this order must hold for the expectations as well, and the theorem follows immediately.

\end{proof}

\bigskip
\bigskip

\begin{recall}[\lemref{lem:random_limit}]
Under~a~random \iid~ordering
 (and the separability assumption~\ref{asm:separability}), 
 the iterates converge \emph{almost surely} to ${\teachers}$,
such that 
$$
\lim_{k\to \infty} 
\dist(\w_k, \teachers)=0,
~~~
\norm{\teacher} \le \norm{\w_\infty} \le 2\norm{\teacher}
\,.
$$
\end{recall}

\begin{proof}
We start by proving the almost sure convergence.
From \lemref{lem:feasible_distance_monotonicity},
we know that $\prn{\dist\prn{\w_{k},\teachers}}_k$
\linebreak
is (non-negative and) pointwise monotonically decreasing.
Then, by the monotone convergence theorem, 
it has a pointwise limit $\dist\prn{\w_{\infty},\teachers}$
such that
$\expectation\left[{\dist\prn{\w_{\infty},\teachers}}\right]
=
\lim_{k\to\infty}
\expectation\left[{\dist\prn{\w_{k},\teachers}}\right]$.
According to the rates from \thmref{thm:random_rates},
we have that 
$\lim_{k\to\infty}
\expectation\left[{\dist\prn{\w_{k},\teachers}}\right]=0$.
In turn, this means that 
$\expectation\left[{\dist\prn{\w_{\infty},\teachers}}\right]=0$.
Overall, we get that the limit $\dist\prn{\w_{\infty},\teachers}$ is pointwise non-negative with a zero mean, so it must be equal to zero with probability $1$.

Finally, from the optimality guarantees of \thmref{thm:optimality_guarantees},
we get that the limit $\w_{\infty}$, which is almost surely in $\teachers$, must be $2$-optimal.
\end{proof}

\bigskip

\bigskip

\newpage

\subsection{Proofs for the Average Iterate (\secref{sec:avg_iterate})}
\label{app:averaging}

\begin{recall}[Proposition~\ref{prop:avearge_iterate}]
    After $n$ cycles 
    under a cyclic ordering 
    ($k=nT$) we have
    $$
    \underbrace{
        \max_{m\in\cnt{T}}
        F_m (\avgitr_k)}_{
        \substack{
        \text{Maximum forgetting}
        }
    }
    \le
    \underbrace{
        \max_{m\in\cnt{T}}
        \dist^{2\!}\prn{\avgitr_k, \feasible_{m}}
    }_{
        \substack{
        \text{Max.~dist.~to any feasible set}
        }
    }
    \!
    R^2
    \le 
    \frac{T^2}{k}
    \!
    \norm{\w^{\star}}^2 \!R^2$$
    and after $k$ iterations under a random ordering we have
\begin{align*}
\expectation_{\tau}\!
    \Big[
    \underbrace{
        {\frac{1}{T}}
        \sum_{m=1}^T
        \!\!
        F_m (\avgitr_k)}_{
        \substack{
        \text{Average forgetting}
        }
    }
    \Big]
    \le
    \expectation_{\tau}
    \Big[
    \underbrace{
        {\frac{1}{T}}
        \sum_{m=1}^T
        \!\!
        \dist^{2\!}\prn{\avgitr_k, \feasible_{m}}
    }_{
        \substack{
        \text{Avg.~distance~to feasible sets}
        }
    }
    \Big]
    R^2
    \le
    \frac{1}{k}\norm{\w^{\star}}^2 \!R^2
\end{align*}
(implying a $\frac{T}{k}\norm{\w^{\star}}^2 \!R^2$ bound on the expected \emph{maximum} forgetting and \emph{maximum} distance to any feasible set).
\end{recall}

\begin{proof}
We split our proof into two (one part for each ordering type).

\subsubsection{Cyclic ordering}
For task $m$, 
we exploit the averaging over all the iterates by combining it with the convexity of $\norm{\vv-\mQ_{m}(\vv)}^2$,
proved in \lemref{lem:residual_is_convex}.
For simplicity,
we assume that $T\mid k$.
Recall Remark~\ref{rmk:translations} on the translations.
We have that,
\begin{align*}
d^{2}\left(\avgitr_{k},\feasible_{m}\right) 
& 
\stackrel{\text{\ref{rmk:translations}}}{=}
\dist^2\!\prn{\avgitr_{k}-\teacher, \mathcal{C}_m}
=
\left\Vert \left(\I-\mQ_{m}\right)
(\avgitr_{k}-\teacher)\right\Vert 
=
\left\Vert \left(\I-\mQ_{m}\right)
\prn{
\frac{1}{k}
\tsum_{t=1}^{k}
\bigprn{\w_{t}-\teacher}
}\right\Vert 
\\
\explain{
\substack{
\text{\lemref{lem:residual_is_convex}}
\\
\text{Jensen's inequality}
}
} 
& \leq\frac{1}{k}\sum_{t=1}^{k}\left\Vert \left(\I-\mQ_{m}\right)(\w_{t}\!-\!\teacher)\right\Vert 
=
\frac{1}{k}\sum_{t=1}^{k}\left\Vert \left(\I-\mQ_{m}\right)\left(\mQ_{t\text{ mod }T}\circ\smallcdots\circ\mQ_{1}\right)\left(\mQ_{T}\circ\smallcdots\circ\mQ_{1}\right)^{\left\lfloor t/T\right\rfloor }\left(\w_{0}\!-\!\teacher\right)\right\Vert ^{2}
\\
\explain{
\text{assuming }
T\mid k} & =\frac{1}{k}\sum_{n=1}^{\left\lfloor k/T\right\rfloor }\sum_{t=1}^{T}\left\Vert \left(\I-\mQ_{m}\right)\left(\mQ_{t}\circ\smallcdots\circ\mQ_{1}\right)\left(\mQ_{T}\circ\smallcdots\circ\mQ_{1}\right)^{n}\left(\w_{0}-\teacher\right)\right\Vert ^{2}\\
 & =\frac{1}{k}\sum_{t=1}^{T}\sum_{n=1}^{\left\lfloor k/T\right\rfloor }\left\Vert \left(\I-\mQ_{m}\right)\left(\mQ_{t}\circ\smallcdots\circ\mQ_{1}\right)\left(\mQ_{T}\circ\smallcdots\circ\mQ_{1}\right)^{n}\left(\w_{0}-\teacher\right)\right\Vert ^{2}
 \\
 & 
 =
 \frac{1}{k}\sum_{t=1}^{T}\sum_{n=1}^{\left\lfloor k/T\right\rfloor }
 \Bignorm{\left(\I-\mQ_{m}\right)\left(\mQ_{t}\circ\smallcdots\circ\mQ_{1}\!\circ\!\mQ_{T}\circ\smallcdots\circ\mQ_{\itr+1}\right)^{n}
\Bigprn{\underbrace{\mQ_{t}\circ\smallcdots\circ\mQ_{1}\left(\w_{0}-\teacher\right)}_{\triangleq\vu_{t}}
 }
 }^{2}
 \\
 & =\frac{1}{k}\sum_{t=1}^{T}\sum_{n=1}^{\left\lfloor k/T\right\rfloor }\left\Vert \left(\I-\mQ_{m}\right)\left(\mQ_{t}\circ\smallcdots\circ\mQ_{1}\!\circ\!\mQ_{T}\circ\smallcdots\circ\mQ_{\itr+1}\right)^{n}\left(\vu_{t}\right)\right\Vert ^{2}
 \\
\explain{
\substack{
\ref{lem:bound_on_projection_to_m_task}
}} & 
\le
\frac{T}{k}\sum_{t=1}^{T}\sum_{n=1}^{\left\lfloor k/T\right\rfloor }\left\Vert \left(\mQ_{t}\circ\smallcdots\circ\mQ_{1}\!\circ\!\mQ_{T}\circ\smallcdots\circ\mQ_{\itr+1}\right)^{n}\left(\vu_{t}\right)\right\Vert ^{2}-\left\Vert \left(\mQ_{t}\circ\smallcdots\circ\mQ_{1}\!\circ\!\mQ_{T}\circ\smallcdots\circ\mQ_{\itr+1}\right)^{n+1}\left(\vu_{t}\right)\right\Vert ^{2}\\
\explain{\text{telescoping}} 
& =\frac{T}{k}\sum_{t=1}^{T}\left\Vert \left(\mQ_{t}\circ\smallcdots\circ\mQ_{1}\!\circ\!\mQ_{T}\circ\smallcdots\circ\mQ_{\itr+1}\right)\left(\vu_{t}\right)\right\Vert ^{2}-\left\Vert \left(\mQ_{t}\smallcdots\mQ_{1}\!\circ\!\mQ_{T}\circ\smallcdots\circ\mQ_{\itr+1}\right)^{\left\lfloor k/T\right\rfloor +1}\left(\vu_{t}\right)\right\Vert ^{2}
\\
 & \leq\frac{T}{k}\sum_{t=1}^{T}\left\Vert \left(\mQ_{t}\circ\smallcdots\circ\mQ_{1}\!\circ\!\mQ_{T}\circ\smallcdots\circ\mQ_{\itr+1}\right)\left(\vu_{t}\right)\right\Vert ^{2}
  =
  \frac{T}{k}\sum_{t=1}^{T}\left\Vert \left(\mQ_{t}\circ\smallcdots\circ\mQ_{1}\!\circ\!\mQ_{T}\circ\smallcdots\circ\mQ_{1}\right)\left(\w_{0}-\teacher\right)\right\Vert ^{2}
 \\
 \explain{\text{contraction}}
 & 
 \leq
 \frac{T}{k}\sum_{t=1}^{T}\left\Vert \w_{0}-\teacher\right\Vert ^{2}
 =
 \frac{T^{2}}{k}\left\Vert\teacher\right\Vert ^{2}\,.
\end{align*}
Overall, we bounded the distance in the cyclic setting:
\hfill
$
    \dist^2\prn{\avgitr_{k}, \feasible_m}
    \le
    \frac{T}{n}
    \left\Vert
    \teacher
    \right\Vert ^{2}
    =
    \frac{T^2}{k}
    \tnorm{\teacher
    }^{2},
    ~\,\forall m\!\in\!\cnt{T}
$.
\linebreak
Then, given \lemref{lem:euclidean_to_hinge}, 
we conclude that
\hfill
$
    F_{m}(\avgitr_k)
    \le
    \dist^2(\avgitr_{k}, \feasible_{m})
    R^2
    \le
    \frac{T^2}{k}\norm{\teacher}^2
    R^2,
    ~\,\forall m\!\in\!\cnt{T}
$.

\subsubsection{Random ordering}

Our random ordering result is related to Proposition~6 and Equations~(12) and (13) in \citet{nedic2010randomPOCS} which analyzed the expected average and maximum distance to any feasible set.
Even more strongly related is Remark~15 in \citet{evron2022catastrophic} which analyzed the 
expected \emph{average} forgetting in continual linear regression.
Importantly, their proof in Appendix~E.2 does not exploit the linearity of the projection operators and works for general projections as well.
For completeness only, we closely follow their proof and present our own proof below.
\begin{align*}
\frac{1}{R^2}
\expectation_{\tau}
    \Big[
    \underbrace{
        \frac{1}{T}
        \sum_{m=1}^T
        F_m (\avgitr_k)}_{
        \substack{
        \text{Average forgetting}
        }
    }
    \Big]
\stackrel{\text{\lemref{lem:euclidean_to_hinge}}}{\le}
&
\expectation_{\tau}
    \Big[\,
    \frac{1}{T}
    \sum_{m=1}^T
    \dist^{2}(\avgitr_k, \feasible_m)
    \Big]
=
\expectation_{\tau}
    \Big[\,
    \frac{1}{T}
    \sum_{m=1}^T
    \norm{\left(\I-\mP_m\right)\avgitr_k}^2
    \Big]
\\
\explain{\text{i.i.d.~uniform}}
=\,
&
\mathbb{E}_{
\mP_{1},\dots,\mP_{k},\mP}
\Bignorm{\left(\I-\mP\right)
\bigprn{\overline{\w}_{k}}
}_{2}^{2}
\\
=\,
&
\mathbb{E}_{
\mP_{1},\dots,\mP_{k},\mP}
\Bignorm{\left(\I-\mP\right)
\bigprn{
\frac{1}{k}\sum_{\itr=1}^{k}\w_{\itr}}
}_{2}^{2}
\\
\explain{\text{\lemref{lem:residual_is_convex},     Jensen's inequality}}
\le\,
&
\frac{1}{k}
\sum_{\itr=1}^{k}
\mathbb{E}_{
\mP_{1},\dots,\mP_{t},\mP}
\Bignorm{\left(\I-\mP\right)
\bigprn{\w_{\itr}}
}_{2}^{2}
\\
=\,
&
\frac{1}{k}
\sum_{\itr=1}^{k}
\mathbb{E}_{
\mP_{1},\dots,\mP_{t},\mP}
\bigg[
\Bignorm{
\left(\I-\mP\right)\bigprn{\w_{\itr}}
-
\underbrace{\left(\I-\mP\right)\bigprn{\teacher}}_{=\0}
}_{2}^{2}
\bigg]
\\
\explain{\text{\propref{prop:projections}}}
\le\,
&
\frac{1}{k}
\sum_{\itr=1}^{k}
\mathbb{E}_{
\mP_{1},\dots,\mP_{t},\mP}
\bigg[
\Bignorm{
\w_{\itr}-\teacher}_{2}^{2}
-
\Bignorm{
\mP(\w_{\itr})-\underbrace{\mP(\teacher)}_{=\teacher}}_{2}^{2}
\bigg]
\\
\explain{\text{i.i.d.}}
=\,
&
\frac{1}{k}
\sum_{\itr=1}^{k}
\mathbb{E}_{
\mP_{1},\dots,\mP_{\itr+1}}
\bigg[
\Bignorm{
\w_{\itr}-\teacher}_{2}^{2}
-
\Bignorm{
\underbrace{\mP_{\itr+1}(\w_{\itr})}_{=\w_{\itr+1}}-\teacher}_{2}^{2}
\bigg]
\\
\explain{\text{telescoping}}
=\,
&
\frac{1}{k}
\mathbb{E}_{
\mP_{1},\dots,\mP_{k}}
\bigg[
\underbrace{\Bignorm{\w_{1}-\teacher}_{2}^{2}}_{
\substack{
\le \norm{\w_{0}-\teacher}_{2}^{2},\\
\text{by
\corref{cor:monotonicity}
}}}
-
\underbrace{\Bignorm{\w_{k+1}-\teacher}_{2}^{2}
}_{\ge0}
\bigg]
\le
\frac{1}{k}
\norm{\teacher}^2
~.
\end{align*}
\end{proof}

\newpage

\section{Proof for the Extended Settings (\secref{sec:extensions})}
\label{app:extensions}

Our next proof follows and generalizes our proof of \thmref{thm:weakly_regularized} in \appref{app:algorithmic_bias}.
We have split the statements and proofs of these two cases for clarity and to make sure that our fundamental \thmref{thm:weakly_regularized} stands on its own.

For ease of readability, we mark \rnotice{$\lambda_t$} and \nnotice{$\mathbf{B}_t$} in red and blue, respectively.

We provide here a unified proof for Theorems~\ref{thm:weak_schedule}~and~\ref{thm:weak_weighted}. To this end, we also define a unified scheme.

\begin{algorithm}[h!]
   \caption{Scheduled and Weighted Sequential Max-Margin
    \label{proc:sche_weighted_adaptive}}
\begin{algorithmic}
   \STATE {\algmargin\bfseries Initialization:} 
   $\w_0 = \0_{\dim}$
   \STATE {\algmargin{\bfseries Iterative update for each task $t\in\cnt{k} 
   $:}} 
    \vspace{-.3em}
    \begin{align}
    \algmargin
    \w_{t}
    =
    {\argmin}_{\w}
    \,
    & 
    \!
    \norm{\w-c_t\w_{t-1}}^2_{\nnotice{\B_{t}}},
    \quad\quad\quad\quad\quad\quad\quad\quad 
    c_t\triangleq{
        \lim_{\lambda\to 0}
        \!
        \tfrac{\ln\rnotice{\lambda_{t-1}}}{\ln\rnotice{\lambda_{t}}}
        }
    \label{sche_weigh}
    \\
    \suchthat
    \,
    \enskip & 
    \,y\w^\top \x \!\ge\! 1,
    \,
    \forall (\x,y)\!\in\! \dataset_{t}
    \nonumber
    \end{align}
    \vspace{-1.5em}
\end{algorithmic}
\end{algorithm}

As we explain in the main body of the paper, we assume that $\forall t\!\in\!\cnt{k}$,
the singular values of $\nnotice{\B_{t}}$ 
are bounded,
\ie it holds that
${0<\nnotice{\mu_t}\le 
\sigma_{\min}(\nnotice{\B_t}) \le
\sigma_{\max}(\nnotice{\B_t})
\le \nnotice{M_t}<\infty}$
for some finite $\nnotice{\mu_t,M_t}\in\reals_{>0}$
(independent of $\lambda$).
Moreover, given $\lambda>0$,
we parameterize the regularization strengths as 
$\rnotice{\lambda_{t}}\!\triangleq\!\rnotice{\lambda_{t}(\lambda)}>0$
for arbitrary functions
$\rnotice{\lambda_{t}}\!:\reals_{>0}\!\to\!\reals_{>0}$
holding that
$\lim_{\lambda\to 0}\rnotice{\lambda_{t}(\lambda)} \!=\! 0$
and
$
\m@th\displaystyle
\lim_{\lambda\to 0}
\rnotice{\tfrac{\ln\lambda_{t-1}}{\ln\lambda_{t}}}<\infty$ is well-defined
$\forall t\!\in\!\cnt{k}$.

\begin{proof}[Proof for Theorems~\ref{thm:weak_schedule}~and~\ref{thm:weak_weighted}]
We will prove by induction on $t\ge0$
that 
the scale of the residual 
$
\m@th\displaystyle
\residual_{t}^{(\lambda)}
\triangleq
{\w_{t}^{(\lambda)} 
\!-\!
{\ln\prn{\tfrac{1}{\rnotice{\lambda_t}}}\w_{t}}}
$
at each iteration is 
$\bignorm{\residual_{t}^{(\lambda)}}
    =\bigO\prn{
        \sum_{t'=1}^{t}
        \bigprn{
        \prod_{n=t'}^{t}
        \nnotice{\tfrac{M_{n}}{\mu_{n}}}
        }
        \ln\ln\prn{\tfrac{1}{\rnotice{\lambda_{t'}}}}
    }$;
and that consequently (since $\ln\prn{\tfrac{1}{\lambda}}\w_t$ grows faster;
see Remark~\ref{rmk:validity}), 
the iterates are either identical
(\ie $\w_{t}^{(\lambda)}=\w_{t}$)
or converge in the same direction when $\lambda\to0$,
\ie $
\lim_{\lambda\to0}
\frac{\w_t^{(\lambda)}}{\tnorm{\w_t^{(\lambda)}}}
=
\frac{\w_t}{\tnorm{\w_t}}
$.

\paragraph{For $t=0$:}
By the conditions of the theorem, it trivially holds that 
$\w_0^{(\lambda)}=\w_0=\0_{\dim}$ and
$\residual_{0}^{(\lambda)}
=\0_{\dim}$.

\medskip

\paragraph{For $t\ge1$:}

The solved optimization problem
(recall Remark~\ref{rmk:simplification})
is:
%
$$\displaystyle
\w^{(\lambda)}_{t}
=
\argmin_{\w \in \reals^{\dim}} 
\lamloss(\w)
\triangleq
\argmin_{\w \in \reals^{\dim}} 
{
\sum_{\x\in \dataset_{t}}
e^{-\w^\top \x }
+
\frac{\rnotice{\lambda_{t}}}{2}
\norm{\nnotice{\sqrt{\B_{t}}}\w\!-\!\nnotice{\sqrt{\B_{t}}}\w^{(\lambda)}_{t-1}}^2
}\,.
$$

\bigskip

We follow the same ideas as in our proof for \thmref{thm:weakly_regularized} in \appref{app:algorithmic_bias}
(with some adjustments since, for instance, our objective here is no longer $\lambda$-strongly convex,
but rather $\rnotice{\lambda_t}\nnotice{\mu_{t}}$-strongly convex, 
since its Hessian matrix holds
${\nabla^2 \lamloss(\w)
\succeq \rnotice{\lambda_t}\nnotice{\B_{t}}
\succeq \rnotice{\lambda_t}\nnotice{\mu_{t}}}\I
\succ \0$).

First, we compute the gradient of $\lamloss$, normalized by $\rnotice{\lambda_t}$:
\begin{align*}
    &\frac{1}{\rnotice{\lambda_t}}
    \nabla
    \lamloss(\w)
    =
    -\frac{1}{\rnotice{\lambda_t}}
    \sum_{\x\in\dataset_{t}}\x\exp\left(-\w^{\top}\x\right)+
    \nnotice{\B_{t}}\w-\nnotice{\B_{t}}\w^{(\lambda)}_{t-1}\,.
\end{align*}
Then, we plug in
$\w
=
\prn{\ln\prn{\tfrac{1}{\rnotice{\lambda_t}}}+\wsign_{t}\ln\ln\prn{\tfrac{1}{\rnotice{\lambda_t}}}}\w_{t}
+\tilde{\w}_{t}
=
\ln\prn{\tfrac{1}{\rnotice{\lambda_t}}\ln^{\wsign_{t}\!}\prn{\tfrac{1}{\rnotice{\lambda_t}}}}
\w_{t}
+\tilde{\w}_{t}$,
for some sign $\wsign_{t}\in\{-1,+1\}$ and
a vector $\tilde{\w}_{t}$ with a norm independent of $\lambda$ (both will be defined below).
\begin{align*}
    &\frac{1}{\rnotice{\lambda_t}}
    \nabla
    \lamloss\Bigprn{
    \prn{\ln\prn{\tfrac{1}{\rnotice{\lambda_t}}}
    +
    \wsign_{t}\ln\ln\prn{\tfrac{1}{\rnotice{\lambda_t}}}
    }
    \w_{t}
    +
    \tilde{\w}_{t}}
    \\
    &=
    -\frac{1}{\rnotice{\lambda_t}}
    \sum_{\x\in\dataset_{t}}
    \x
    e^{
    \ln\prn{\rnotice{\lambda_t}\ln^{-\wsign_t\!}\prn{\tfrac{1}{\rnotice{\lambda_t}}}}\w_{t}^\top 
    \x}
    e^{-\tilde{\w}_{t}^\top \x}
    +
    \ln\Bigprn{\tfrac{1}{\rnotice{\lambda_t}}
    \ln^{\wsign_t\!}\prn{\tfrac{1}{\rnotice{\lambda_t}}}}
    \nnotice{\B_{t}}\w_{t}
    +\nnotice{\B_{t}}\tilde{\w}_{t}
    -\nnotice{\B_{t}}\w^{(\lambda)}_{t-1}
    \\
    &=
    -\frac{1}{\rnotice{\lambda_t}}
    \sum_{\x\in\dataset_{t}}
    \x
    \prn{\rnotice{\lambda_t}\ln^{-\wsign_t\!}\prn{\tfrac{1}{\rnotice{\lambda_t}}}}^{\w_{t}^\top \x}
    e^{-\tilde{\w}_{t}^\top \x}
    +
    \ln\Bigprn{\tfrac{1}{\rnotice{\lambda_t}}
    \ln^{\wsign_t\!}\prn{\tfrac{1}{\rnotice{\lambda_t}}}}
    \nnotice{\B_{t}}\w_{t}
    +
    \nnotice{\B_{t}}\tilde{\w}_{t}
    \!-\!\nnotice{\B_{t}}\w^{(\lambda)}_{t-1}
    \,.
\end{align*}
Now, denoting the set of support vectors by
$\supp_{t}\triangleq\left\{ \x\in \dataset_{t}\mid\w_{t}^{\top}\x=1\right\}$
(which might be empty for $t\ge 2$),
and using the inductive assumption
that
$
\w_{t-1}^{(\lambda)}=
\residual_{t-1}^{(\lambda)}
+
\ln\prn{\tfrac{1}{\rnotice{\lambda_{t-1}}}}\w_{t-1}
$
(where $\bignorm{\residual_{t-1}^{(\lambda)}}=
    \bigO\prn{
        \sum_{t'=1}^{t-1}
        \bigprn{
        \prod_{n=t'}^{t-1}
        \nnotice{\tfrac{M_{n}}{\mu_{n}}}
        }
        \ln\ln\prn{\tfrac{1}{\rnotice{\lambda_{t'}}}}
    }$), 
\linebreak    
we get 
\begin{flalign*}
\mathrlap{
\frac{1}{\rnotice{\lambda_t}}
    \nabla
    \lamloss\Bigprn{
    \prn{\ln\prn{\tfrac{1}{\rnotice{\lambda_t}}}
    +
    \wsign_{t}\ln\ln\prn{\tfrac{1}{\rnotice{\lambda_t}}}
    }
    \w_{t}
    +
    \tilde{\w}_{t}}
}
\\
\mathrlap{
=
    -
    \frac{1}{\rnotice{\lambda_t}}
    \prn{\rnotice{\lambda_t}\ln^{-\wsign_t\!}\prn{\tfrac{1}{\rnotice{\lambda_t}}}}^{1}
    \sum_{\x\in\supp_{t}}
    \x
    e^{-\tilde{\w}_{t}^\top \x}
    -
    \frac{1}{\rnotice{\lambda_t}}\!
    \sum_{\x\notin\supp_{t}}
    \x
    \prn{\rnotice{\lambda_t}\ln^{-\wsign_t\!}\prn{\tfrac{1}{\rnotice{\lambda_t}}}}^{\w_{t}^\top \x}
    e^{-\tilde{\w}_{t}^\top \x}
    +
}
\\
&& \mathllap{
+
    \ln \prn{\tfrac{1}{\rnotice{\lambda_t}}}
    \nnotice{\B_{t}}\w_{t}
    -
    \ln \prn{\tfrac{1}{\rnotice{\lambda_{t-1}}}} \nnotice{\B_{t}}\w_{t-1}
    +
    \ln\ln^{\wsign_t\!}\prn{\tfrac{1}{\rnotice{\lambda_t}}}
    \nnotice{\B_{t}}\w_{t}
    +
    \nnotice{\B_{t}}\tilde{\w}_{t}
    -
    \nnotice{\B_{t}}\residual_{t-1}^{(\lambda)}
}
\\
\mathrlap{
=
-
\ln^{-\wsign_t\!}\prn{\tfrac{1}{\rnotice{\lambda_t}}}
\sum_{\x\in\supp_{t}}
\x
e^{-\tilde{\w}_{t}^\top \x}
-
\sum_{\x\notin\supp_{t}}
\x
{\rnotice{\lambda_t}}^{\w_{t}^\top \x-1}
\prn{\ln\prn{\tfrac{1}{\rnotice{\lambda_t}}}}^{-\wsign_t\w_{t}^\top \x}
e^{-\tilde{\w}_{t}^\top \x}
+
}
\\
&& \mathllap{
+
    \ln \prn{\tfrac{1}{\rnotice{\lambda_{t}}}}
    \nnotice{\B_{t}}
    \prn{\w_{t}-
    \rnotice{\tfrac{\ln\lambda_{t-1}}{\ln\lambda_{t}}}\w_{t-1}}
    +
    \ln\ln^{\wsign_t\!}\prn{\tfrac{1}{\rnotice{\lambda_{t}}}}
    \nnotice{\B_{t}}\w_{t}
    +
    \nnotice{\B_{t}}\tilde{\w}_{t}
    -
    \nnotice{\B_{t}}\residual_{t-1}^{(\lambda)}
    \,.
}
\end{flalign*}

By the triangle inequality 
and since $\bignorm{\ln\ln^{\wsign_t\!}\prn{\tfrac{1}{\rnotice{\lambda_{t}}}}}
=
\bignorm{\wsign_t\ln\ln\prn{\tfrac{1}{\rnotice{\lambda_{t}}}}}
=
\overbrace{\abs{\wsign_t}}^{= 1}
\overbrace{\ln\ln\prn{{1}/{\rnotice{\lambda_{t}}}}}^{>0\text{, when }\rnotice{\lambda_t}<{1}/{e}}$, we have,
\begin{align*}   
&\norm{\frac{1}{\rnotice{\lambda_t}}
\nabla
\lamloss\Bigprn{\ln\prn{\tfrac{1}{\rnotice{\lambda_{t}}}}+
    \wsign_{t}\ln\ln\prn{\tfrac{1}{\rnotice{\lambda_{t}}}}\w_{t}+\tilde{\w}_{t}}
}
\\
&
\le
\,
\underbrace{\biggnorm{
    \ln \prn{\tfrac{1}{\rnotice{\lambda_t}}}
    \nnotice{\B_{t}}
    \prn{\w_{t}\!-\!\rnotice{\tfrac{\ln\lambda_{t-1}}{\ln\lambda_{t}}}\w_{t-1}}
    -
    \ln^{-\wsign_{t}\!}
    \prn{\tfrac{1}{\rnotice{\lambda_t}}}
    \!
    \sum_{\x\in\supp_{t}}\!
    \x
    e^{-\tilde{\w}_{t}^\top \x}
}
}_{\triangleq\vect{a}_1(\rnotice{\lambda_{t}})}
\\
&
\hspace{1cm}
+
\underbrace{
\biggnorm{
    \sum_{\x\notin\supp_{t}}
    \x
    {\rnotice{\lambda_{t}}}^{\w_{t}^\top \x-1}
    \!
    \prn{\ln\prn{\tfrac{1}{\rnotice{\lambda_{t}}}}}^{-\wsign_t\w_{t}^\top \x}
    e^{-\tilde{\w}_{t}^\top \x}
}
}_{\triangleq\vect{a}_2(\rnotice{\lambda_{t}})}
+
\ln\ln\prn{\tfrac{1}{\rnotice{\lambda_{t}}}}
\Bignorm{
    \nnotice{\B_{t}}\w_{t}
}
+
\Bignorm{
    \nnotice{\B_{t}}\tilde{\w}_{t}
}
+
\Bignorm{
    \nnotice{\B_{t}}\residual_{t-1}^{(\lambda)}
}
\\
&
\le \vect{a}_1(\rnotice{\lambda_{t}})
+
\vect{a}_2(\rnotice{\lambda_{t}})
+
\ln\ln\prn{\tfrac{1}{\rnotice{\lambda_{t}}}}
\nnotice{M_{t}}\bignorm{\w_{t}}
+
\nnotice{M_{t}}
\bignorm{\tilde{\w}_{t}}
+
\nnotice{M_{t}}
\bignorm{\residual_{t-1}^{(\lambda)}}\,.
\end{align*}

\newpage

Here also, we distinguish between two different  behaviors of
the SMM solution $\w_t$
(recall that ${c_t\triangleq{
        \lim_{\lambda\to 0}
        \!
        \tfrac{\ln\rnotice{\lambda_{t-1}}}{\ln\rnotice{\lambda_{t}}}
        }}$):
\begin{enumerate}
    \item \textbf{When 
    $\w_{t}
    \triangleq 
    \mP_{t} (c_t \w_{t-1})
    \neq 
    c_t \w_{t-1}$ 
    (and necessarily 
    $|{\supp_{t}}|\ge1$):}
    %
    We choose $\wsign_{t}\!=\!-1$ and $\vect{a}_1(\lambda)$ becomes:
    \begin{align*}
    \vect{a}_1(\rnotice{\lambda_{t}})
    &
    =
    \bignorm{
    \ln \prn{\tfrac{1}{\rnotice{\lambda_{t}}}}
    \nnotice{\B_{t}}
    \prn{\w_{t}\!-\!c_t \w_{t-1}}
    -
    \!
    \ln^{-\wsign_{t}\!}
    \prn{\tfrac{1}{\rnotice{\lambda_t}}}
    \!\!
    \sum_{\x\in\supp_{t}}
    \!
    \x
    e^{-\tilde{\w}_{t}^\top \x}
    }
    =
    \ln \prn{\tfrac{1}{\rnotice{\lambda_{t}}}}
    \!
    \bignorm{
    \nnotice{\B_{t}}
    \prn{\w_{t}\!-\!c_t \w_{t-1}}
    -\!\!
    \sum_{\x\in\supp_{t}}
    \!
    \x
    e^{-\tilde{\w}_{t}^\top \x}
    }.
    \end{align*}
    We thus wish to choose 
    $\tilde{\w}_{t}$
    so as to \emph{zero} $\vect{a}_1(\rnotice{\lambda_{t}})$.
    That is,
    according to the KKT conditions of
    \eqref{sche_weigh},
    we have
    $$
    \tsum_{\x\in\supp_{t}}
    \x
    e^{-\tilde{\w}_{t}^\top \x}
    =
    \nnotice{\B_{t}}\prn{
    \w_{t}
    -c_t \w_{t-1}
    }
    \triangleq
    \tsum_{\x\in\supp_{t}}
    \x
    \alpha(\x)
    \,,
    $$
    where $\bm{\alpha}\in\reals_{\ge0}^{\abs{\dataset_{t}}}$ 
    is the dual solution of \eqref{sche_weigh}.
    From \lemref{lem:kkt_w_tilde} and \corref{cor:finite_w_tilde} (applied with our $\B_t$ and $c_t$ here), we know that there almost surely exists such a vector $\tilde{\w}_{t}$ whose norm
    is $\bigO\prn{1}$ (independent of $\lambda$).
    
    Furthermore, 
    since 
    $\lim_{\lambda\to 0} 
    \prn{\lambda^{c-1}
    \ln^{c} \prn{\nicefrac{1}{\lambda}}
    }=0,~\forall c>1$,
    it holds that
    $\vect{a}_2(\rnotice{\lambda_{t}})$ becomes
    \vspace{-0.4em}
    \begin{align*}   
    \vect{a}_2(\rnotice{\lambda_{t}})
    &=
    \biggnorm{
    \sum_{\x\notin\supp_{t}}
    \underbrace{
        \x
        e^{-\tilde{\w}_{t}^\top \x}
    }_{={\bigO\prn{1}}}
    \underbrace{
        \rnotice{\lambda_{t}}^{\overbrace{\w_{t}^\top \x-1}^{>0}}
        \prn{\ln \prn{\nicefrac{1}{\rnotice{\lambda_{t}}}}}^{
        {\w_{t}^\top \x}
        }
    }_{\to 0}
    }
    \xrightarrow[]{\lambda\to 0} 0\,.
    \end{align*}

    In conclusion, we can choose $\tilde{\w}_t$ and $\wsign_{t}$
    such that $\norm{\tilde{\w}_t}=\bigO\prn{1}$,
    $\vect{a}_1(\rnotice{\lambda_{t}})=0$,
    and $\vect{a}_2(\rnotice{\lambda_{t}})\to0$.

    \medskip
    
    \item
    \textbf{When 
    $\w_{t}
    \triangleq 
    \mP_{t} (c_t \w_{t-1})
    = c_t \w_{t-1}$
    (and possibly $\supp_{t} = \emptyset$):
    }
    %
    We choose $\tilde{\w}_t=
    \0_{\dim}
    $ and $\wsign_{t}=1$.
    It follows that
    \vspace{-0.3em}
    \begin{align*}
    \vect{a}_1(\rnotice{\lambda_{t}})
    &
    =\bignorm{
    \,
    \ln^{-\wsign_{t}\!}
    \prn{\tfrac{1}{\rnotice{\lambda_t}}}
    \sum_{\x\in\supp_{t}}
    \x\,
    }
    =
    \underbrace{
    \ln^{-1\!}
    \prn{\tfrac{1}{\rnotice{\lambda_t}}}
    }_{\to 0}
    \underbrace{
    \bignorm{\,
    \tsum_{\x\in\supp_{t}}
        \x
        \,
    }
    }_{={\bigO\prn{1}}}
    \xrightarrow[]{\lambda\to 0} 0\,,
    \\
    \vect{a}_2(\rnotice{\lambda_{t}})
    &=
    \biggnorm{
    \sum_{\x\notin\supp_{t}}
    \underbrace{
        \x
    }_{={\bigO\prn{1}}}
    \underbrace{
        \rnotice{\lambda_{t}}^{\overbrace{\w_{t}^\top \x-1}^{>0}}
    }_{\to 0}
    \underbrace{
        \prn{\ln \prn{\tfrac{1}{\rnotice{\lambda_{t}}}}}^{
        {\overbrace{-\w_{t}^\top \x}^{<-1}}
        }
    }_{\to 0}
    }
    \xrightarrow[]{\lambda\to 0} 0\,.
    \end{align*}
\end{enumerate}

\bigskip

Finally, we use the $\rnotice{\lambda_t}\nnotice{\mu_{t}}$-strong convexity of our objective
and \lemref{lem:strong-convexity}
to bound the distance to the optimum 
by
\begin{align*}
    \bignorm{\residual_{t}^{(\lambda)}}
    &\triangleq
    \bignorm{
    \w_{t}^{(\lambda)}
    -
    \ln\prn{\tfrac{1}{\rnotice{\lambda_{t}}}}\w_{t}
    }
    =
    \bignorm{
    \w_{t}^{(\lambda)}
    -
    \prn{
    \prn{\ln\prn{\tfrac{1}{\rnotice{\lambda_{t}}}}
    +
    \wsign_t\ln\ln\prn{\tfrac{1}{\rnotice{\lambda_{t}}}}}\w_{t}
    +
    \tilde{\w}_{t}
    }
    +
    \wsign_t\ln\ln\prn{\tfrac{1}{\rnotice{\lambda_{t}}}}
    \w_{t}
    +
    \tilde{\w}_{t}
    }
    \\
    \explain{\text{triangle ineq.}}
    &
    \le
    \bignorm{
    \w_{t}^{(\lambda)}
    -
    \prn{
    \prn{\ln\prn{\tfrac{1}{\rnotice{\lambda_{t}}}}
    +
    \wsign_t\ln\ln\prn{\tfrac{1}{\rnotice{\lambda_{t}}}}}\w_{t}
    +
    \tilde{\w}_{t}
    }
    }
    +
    \bignorm{
    \wsign_t\ln\ln\prn{\tfrac{1}{\rnotice{\lambda_{t}}}}
    \w_{t}
    }
    +
    \bignorm{
    \tilde{\w}_{t}
    }
    \\
    \explain{\text{\lemref{lem:strong-convexity}}}
    &
    \le 
    \frac{1}{\rnotice{\lambda_{t}}\nnotice{\mu_{t}}}
    \norm{
    \nabla
    \lamloss\Bigprn{
        \prn{\ln\prn{\tfrac{1}{\rnotice{\lambda_{t}}}}
        + \wsign_t\ln\ln\prn{\tfrac{1}{\rnotice{\lambda_{t}}}}}\w_{t}
        +
        \tilde{\w}_{t}
    }
    }
    +
    \ln\ln\prn{\tfrac{1}{\rnotice{\lambda_{t}}}}
    \bignorm{\w_{t}}
    +
    \bignorm{\tilde{\w}_{t}}
    \\
    \explain{\text{the above}}
    &
    \le
    \nnotice{\frac{1}{\mu_{t}}}
    \biggprn{\vect{a}_1(\rnotice{\lambda_{t}})
    +
    \vect{a}_2(\rnotice{\lambda_{t}})
    +
    \ln\ln\prn{\tfrac{1}{\rnotice{\lambda_{t}}}}
    \nnotice{M_{t}}\bignorm{\w_{t}}
    +
    \nnotice{M_{t}}
    \bignorm{\tilde{\w}_{t}}
    +
    \nnotice{M_{t}}
    \bignorm{\residual_{t-1}^{(\lambda)}}
    }
    +
    \ln\ln\prn{\tfrac{1}{\rnotice{\lambda_{t}}}}
    \bignorm{\w_{t}}
    +
    \bignorm{\tilde{\w}_{t}}
    \\
    &
    =
    \underbrace{
    \nnotice{\frac{1}{\mu_{t}}}
    \vect{a}_1(\rnotice{\lambda_{t}})+
    \nnotice{\frac{1}{\mu_{t}}}
    \vect{a}_2(\rnotice{\lambda_{t}})
    }_{\to 0\text{, since }\nnotice{\mu_{t}}\text{ is finite}}
    +
    \ln\ln\prn{\tfrac{1}{\rnotice{\lambda_{t}}}}
    \nnotice{\prn{1+\frac{M_t}{\mu_{t}}}}
    \underbrace{
    \bignorm{\w_{t}}
    }_{=\bigO\prn{1}}
    \,+\,
    \nnotice{\prn{1+\frac{M_t}{\mu_{t}}}}
    \underbrace{
        \bignorm{
            \tilde{\w}_{t}
        }
    }_{
    =\bigO\prn{1}}
    \,\,\,\,+\,\,\,
    ~
    \nnotice{\frac{M_t}{\mu_{t}}}
    \bignorm{\residual_{t-1}^{(\lambda)}}
    ~
    ,
    %
\end{align*}
and since we have by the induction assumption that 
$\bignorm{\residual_{t-1}^{(\lambda)}}
=\bigO\prn{
    \sum_{t'=1}^{t-1}
    \bigprn{
    \prod_{n=t'}^{t-1}
    \nnotice{\tfrac{M_{n}}{\mu_{n}}}
    }
    \ln\ln\prn{\tfrac{1}{\rnotice{\lambda_{t'}}}}
}$ 
and also $\nnotice{\tfrac{M_{n}}{\mu_{n}}}\ge1$,
we can conclude that 
$
    \bignorm{\residual_{t}^{(\lambda)}}
    =\bigO\prn{
        \sum_{t'=1}^{t}
        \bigprn{
        \prod_{n=t'}^{t}
        \nnotice{\tfrac{M_{n}}{\mu_{n}}}
        }
        \ln\ln\prn{\tfrac{1}{\rnotice{\lambda_{t'}}}}
    }
    $.
\end{proof}

\vspace{1em}
    
\begin{remark}[Applicability of our analysis for finite $\lambda$]
    \label{rmk:validity}
    When the singular values of the weight matrices of all tasks are between
    $\nnotice{0<\mu\le\mu_t\le M_t \le M <\infty}$, 
    the residual from our analysis becomes of the order 
    ${
    \bignorm{\residual_{t}^{(\lambda)}}
    \triangleq
    \bignorm{
    \w_{t}^{(\lambda)}
    -
    \ln \tprn{\tfrac{1}{\rnotice{\lambda_{t}}}} \w_{t}}
    =\bigO\prn{
        \sum_{t'=1}^{t}
        \bigprn{
        \nnotice{\hfrac{M}{\mu}}
        }^{t-t'+1}
        \ln\ln\prn{\tfrac{1}{\rnotice{\lambda_{t'}}}}
    }
    }
    $.
    When this bound on the condition number, \ie $\nnotice{\hfrac{M}{\mu}}$, is strictly larger than~$1$,
    we get an exponential growth of the residuals.
    Therefore, 
    within a few tasks under a \emph{finite} regularization strength $\lambda$, the bound on the residuals might become even larger than the scale of the scaled SMM solutions 
    \ie larger than
    $\ln \tprn{\tfrac{1}{\rnotice{\lambda_{t}}}}$.
    In turn, this will invalidate our analysis.
    Of course, in the limit of $\lambda\to 0$, our analysis still applies.
    As we explained in Remark~\ref{rmk:limits_order},
    we take $\lambda\to 0$ \emph{after} fixing the number of iterations $k$ (or $t$).

    The question of determining the specific value of $\lambda$ that practically ensures that
    $\frac{\w_t^{(\lambda)}}{\tnorm{\w_t^{(\lambda)}}}
    \approx \frac{\w_t}{\tnorm{\w_t}}$, as well as understanding the true impact of $\nnotice{\hfrac{M}{\mu}}>1$ on the residuals, 
    remains an intriguing and open research question. 
\end{remark}

\newpage

\subsection{Proofs for Regularization Strength Scheduling (\secref{sec:scheduling})}
\label{app:scheduling}

\begin{recall}[Proposition~\ref{prop:large_p_no_convergence}]
There exists a construction of two jointly-separable tasks
in which the iterates of the cyclic ordering do not converge to $\teachers$, for any $p>1$.
Specifically, for any ${p> 1,\norm{\teacher}>1}$,
it holds that
\begin{align*}
    \lim_{k\to\infty}
    \frac{\dist(\w_{k},\teachers)}{\tnorm{\teacher}}
    =
    \frac{\norm{\teacher}^2(p-1)}{
        2+\norm{\teacher}^2(p-1)
    }
    \sqrt{1-\frac{1}{\norm{\teacher}^2}}
    \,.
\end{align*}
\end{recall}

\begin{proof}
First, we explain our construction in the following figure.
\vspace{-0.5em}
\begin{figure}[ht!]
\centering
\begin{minipage}[t!]{0.61\linewidth}
\caption{
    We consider a 2-dimensional setting of 2 tasks, one positively-labeled normalized sample per task,
    \ie $(\x_1, +1),\,(\x_2, +1)$. 
    Both $\x_1,\x_2$ are on the unit sphere (hence $R=1$)
    and are symmetric w.r.t.~the vertical axis.
    \\
    In this setting, when $p>1$, it can be readily seen that at each iteration we perform an orthogonal projection onto a closed affine subspace. Therefore, if we converge to $\teachers$, we must converge to the minimum-norm solution $\teacher$ specifically (see \citet{halperin1962product} for instance).
    We thus wish to study $\lim_{k\to\infty}{\tnorm{\w_k-\teacher}}$.
    \\
    To ease our notations, we again translate our space by $-\teacher$ such that instead of using the ``affine'' halfspaces $\feasible_1, \feasible_2$,
    we use their ``homogeneous'' counterparts 
    $\mathcal{C}_m \triangleq \feasible_m - \teacher$
    (see Remark~\ref{rmk:translations}).
    Under this ``change of coordinates'', instead of analyzing $\norm{\w_k-\teacher}$, we can  simply analyze $\norm{\w_k}$.
    \\
    Since the setting is 2-dimensional and each task has a single sample, we can parameterize 
    the orthogonal projections as 
    ${\mP_1=\vu\vu^\top\!=\I\!-\!\x_1\x_1^\top}$ and ${\mP_2=\vv\vv^\top\!=\I\!-\!\x_2\x_2^\top}$ 
    (where $\norm{\vu}\!=\!\norm{\vv}\!=\!1$). 
    We denote the angle between the two subspaces that we project onto as
    ${\theta \triangleq 
    \arccos\prn{\vu^{\top}\vv}
    =
    \pi \!-\! \arccos{\left(\x_{1}^{\top} \x_{2}\right)}}$. 
    \\
    Notice that (before the translation) the min-norm solution is $\m@th\displaystyle{\teacher\!=\!\frac{1}{\sin\prn{\hfrac{\theta}{2}}
    }\!\begin{bmatrix}0 \\ 1\end{bmatrix}}$.
}
\end{minipage}
\hfill
\begin{minipage}[t!]{0.37\linewidth}
\frame{\includegraphics[width=.99\linewidth]{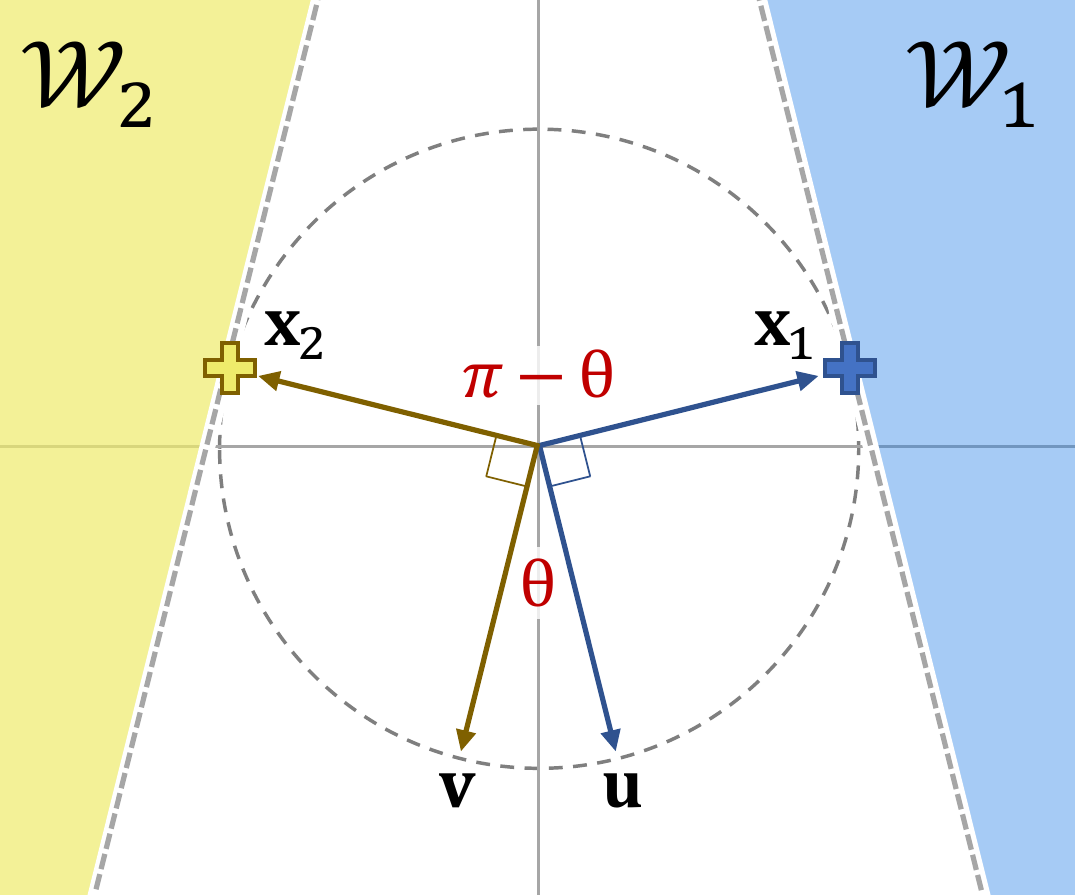}}
\end{minipage}
\end{figure}

\vspace{-.3em}

In the Scheduled Sequential Max-Margin \procref{proc:scheduled_adaptive},
each iterate is multiplied by  $\beta\triangleq\frac{1}{p}$
before being projected onto the next subspace,
equivalently to
\begin{align*}
    \w_{k} = 
    \w_{2n} = 
    \mP_2
    \Bigprn{
        (1-\beta)\w_0
        +\beta\mP_1
        \bigprn{(1-\beta)\w_0
        +
        \beta\w_{2\prn{n-1}}
        }
    }
\end{align*}
To understand this recursive expression,
let $n=1$. 
We get (the projections are linear)
\hfill
$
    \w_2 = 
    (1-\beta)\mP_2\w_0+\beta\mP_2\mP_1\w_0
$.
\linebreak
Then, for $n=2$, we get
\hfill
$
    \w_4 = 
    (1-\beta)\mP_2\w_0+\beta(1-\beta)\mP_2\mP_1\w_0
    +\beta^2(1-\beta)\mP_2\mP_1\mP_2\w_0 
    + \beta^3\left(\mP_2\mP_1\right)^2\w_0
$.

Recursively, we get
\hfill
$
\m@th\displaystyle
    \w_k = \w_{2n} = 
    (1\!-\!\beta)
    \!
    \left(
    \sum_{k=0}^{n-1}\beta^{2k}\left(\mP_2\mP_1\right)^k\mP_2\w_0
    + 
    \sum_{k=1}^{n-1}\beta^{2k-1}\left(\mP_2\mP_1\right)^{k}\w_0
    \right)
    \!
    +
    \beta^{2n-1}\left(\mP_2\mP_1\right)^{2}\w_0
$.

Plugging in the parameterization of the projections,
we get
\begin{align*}
    {\w_{k}} = 
    {\w_{2n}} = 
    (1-\beta)\left(
    \left[\sum_{k=0}^{n-1}(\beta\cos{\theta})^{2k}\right]\vv\vv^\top\w_0
    +\left[\sum_{k=1}^{n-1}(\beta\cos{\theta})^{2k-1}\right]\vv\vu^\top\w_0
    \right)
    +(\beta\cos{\theta})^{2n-1}\vv\vu^\top\w_0
    \,.
\end{align*}
Since $\beta=\tfrac{1}{p}<1$, we have
\begin{align*}
    \lim_{n\to\infty}{\w_{2n}} 
    = 
    (1-\beta)
    \prn{
    \frac{\vv^\top\w_0}{1-(\beta\cos\theta)^2}
    +
    \frac{\beta\vu^\top\w_0 \cos{\theta}}{1-(\beta\cos\theta)^2}
    }\vv
    = 
    (1-\beta)
    \prn{
    \frac{\vv^\top\w_0+\beta\vu^\top\w_0 \cos{\theta}}{1-(\beta\cos\theta)^2}
    }\vv\,.
\end{align*}
From the definition of $\vv,\vu$ it holds that $\vv^{\top}\w_0=\vu^{\top}\w_0=\cos{\frac{\theta}{2}}\tnorm{\w_0}$
and also that $\tnorm{\w_0}=\frac{1}{\sin\prn{\hfrac{\theta}{2}}}$
(after the translation). 
Putting it all together, we get
\hfill
   ${\m@th\displaystyle
    \lim_{k\to\infty}\!{\w_k} 
    =
    \lim_{n\to\infty}{\w_{2n}} 
    = 
    (1\!-\!\beta)
    \cos{\tfrac{\theta}{2}}
    \tnorm{\w_0}
    \prn{
    \frac{1+\beta \cos{\theta}}{1-(\beta\cos\theta)^2}
    }\vv
    = 
    \tnorm{\w_0}
    \frac{(1\!-\!\beta)
    \cos{\tfrac{\theta}{2}}}{
    1-\beta\cos\theta}
    \vv}$.

Recalling that $\norm{\vv}\!=\!1$ and returning to the original coordinate system (by reversing the aforementioned translation), we get:
%
%
%
\begin{align*}
    \lim_{k\to\infty}\frac{\tnorm{\w_k - \teacher}}{\tnorm{\w_0 - \teacher}}
    = 
    \frac{
        \prn{1-\beta}
        \cos{\frac{\theta}{2}}
    }{1-\beta\cos{\theta}}
    =
    \frac{1-\beta}{1-\beta\bigprn{1-\tfrac{2}{\norm{\teacher}^2}}}
    \sqrt{1-\frac{1}{\norm{\teacher}^2}}
    =
    \frac{p-1}{
        p-1+\tfrac{2}{\norm{\teacher}^2}
    }
    \sqrt{1-\frac{1}{\norm{\teacher}^2}}
    \,.
\end{align*}
\vspace{-.3em}
\end{proof}

\newpage

\subsection{Additional Material for Weighted Regularization (\secref{sec:weighted})}
\label{app:weighted}

\begin{recall}[Proposition~\ref{prop:regression_fisher}]
Using a Fisher-information-based weighting scheme 
of $\B_{t}=\sum_{i=1}^{t-1}\sum_{\x\in \dataset_{i}}\x\x^{\top}$,
there is no forgetting in continual linear regression.
\end{recall}

\begin{proof}[Details and proof.]
Consider training a linear regression model with the square loss on a sequence of tasks, obtaining iterates $\w_1,..., \w_{k}$. 
We follow \citet{evron2022catastrophic},
assuming that the tasks are jointy-realizable by a linear predictor, \ie there exists $\tilde{\w}$ such that for all $t$, $\forall(\x,y)\in \dataset_t:\x^{\top}\tilde{\w}=y$.
The forgetting after training task $t$ is defined 
(as in \citet{evron2022catastrophic})
to be the average training loss of $\w_t$ on tasks $1,...,t$.

Now, we identify a specific regularization scheme that guarantees no forgetting.%

For linear regression with square loss, define
\begin{align*}
\w_{1}\!\in
&
\argmin_{\w}
\!\!\!\!
\sum_{\left(\x,y\right)\in \dataset_{1}}
\!\!\!\!\!
\left(\x^{\top}\w-y\right)^{2}
\,,
\\
\w_{t}\!\in
&
\argmin_{\w}
\!\!\!\!
\sum_{\left(\x,y\right)\in \dataset_{t}}
\!\!\!\!\!
\left(\x^{\top}\w\!-\!y\right)^{2}
\!\!+\!
\frac{\lambda}{2}
\!
\norm{\w\!-\!\w_{t-1}}_{\B_{t}}^2
\!,\,
t\!\ge\!2.
\end{align*}
 where 
 $\B_{t}=\sum_{i=1}^{t-1}\sum_{\x\in \dataset_{i}}\x\x^{\top}$.
Then for any $t\in\cnt{k}$ and any $\lambda>0$, 
we show that there is no forgetting.

\bigskip

The proof follows by induction. For $t=1$ the statement is trivially
correct. Assume by induction that the statement is correct for some
task $t$, meaning there is no forgetting for $\w_t$ on tasks $1,...,t$.
We want to show that the statement is correct for task $t+1$, \ie for $\w_{t+1}$ there is no forgetting on tasks $1,...,t+1$.

Note that for task $t+1$ the regularization term is 
\begin{align*}
\frac{\lambda}{2}\left(\w-\w_{t}\right)^{\top}\B_{t+1}\left(\w-\w_{t}\right) & =\frac{\lambda}{2}\left(\w-\w_{t}\right)^{\top}\left(\sum_{i=1}^{t}\sum_{\x\in \dataset_{i}}\x\x^{\top}\right)\left(\w-\w_{t}\right)\\
 & =\frac{\lambda}{2}\sum_{i=1}^{t}\sum_{\x\in \dataset_{i}}\left(\w-\w_{t}\right)^{\top}\x\x^{\top}\left(\w-\w_{t}\right)\\
 & =\frac{\lambda}{2}\sum_{i=1}^{t}\sum_{\x\in \dataset_{i}}\left(\x^{\top}\w-\x^{\top}\w_{t}\right)^{2}~.
\end{align*}
By the realizability assumption and the inductive assumption, we have that for all $i=1,2,...,t$
it holds that ${\forall\left(\x,y\right)\in \dataset_{i}:\,\x^{\top}\w_{t}=y}$,
therefore we get
\[
\frac{\lambda}{2}\left(\w-\w_{t}\right)^{\top}\B_{t+1}\left(\w-\w_{t}\right)=\frac{\lambda}{2}\sum_{i=1}^{t}\sum_{\left(\x,y\right)\in \dataset_{i}}\left(\x^{\top}\w-y\right)^{2}~.
\]
It follows that
\begin{align*}
\w_{t+1} & \in\argmin_{\w}
\left[\sum_{\left(\x,y\right)\in S_{t+1}}\!\!\!\left(\x^{\top}\w-y\right)^{2}+\frac{\lambda}{2}\sum_{i=1}^{t}\sum_{\left(\x,y\right)\in \dataset_{i}}\left(\x^{\top}\w-y\right)^{2}\right]~.
\end{align*}
Finally, by the realizability assumption, $\w_{t+1}$ can achieve zero loss on all tasks $1,...,t+1$, and thus there is no forgetting.
\end{proof}

\newpage

\subsubsection{Additional experiments for Fisher-Information-based weighting schemes}
\label{app:weighted_examples}

Here we demonstrate two different behaviors of Fisher-based weighting schemes.

To the right, we see a larger version of Figure \ref{fig:weighted}.
The two positively-labeled datapoints of the first task are
${\x_1=\begin{bmatrix}
-0.3714,  0.9285
\end{bmatrix}^\top}$
and ${\x_2=\begin{bmatrix}
0.9285,  0.3714
\end{bmatrix}^\top}$.
Since these examples are orthogonal, the Fisher Information matrix is proportional to
${\x_1\x_1^\top + \x_2\x_2^\top 
=
\I}$ (there is a multiplicative factor stemming from the normalized probability $\frac{\exp\prn{y_n \tilde{\w}^\top \x_n}}{\exp\prn{y_n \tilde{\w}^\top \x_n} + \exp\prn{-y_n \tilde{\w}^\top \x_n}}
=
\frac{\exp\prn{1}}{\exp\prn{1} + \exp\prn{-1}}
$ of these two datapoints which are both support vectors).
The datapoint of the second task is $\x_3=\begin{bmatrix}
-0.9285,  0.3714
\end{bmatrix}^\top$.
Notice how the weighted regularization yields the same solution as the vanilla regularization, since $\B_2 \propto \I$.

To the left, we run a different experiment. 
The two positively-labeled datapoints of the first task are
${\x_1=\begin{bmatrix}
-2.4660,  0.4110
\end{bmatrix}^\top}$
and ${\x_2=\begin{bmatrix}
0.9285,  0.3714
\end{bmatrix}^\top}$.
The datapoint of the second task is $\x_3=\begin{bmatrix}
-0.4642,  0.1857
\end{bmatrix}^\top$.
In contrast to the previous example, here the Fisher Information matrix is proportional to
$
\begin{bmatrix}
6.9431 & -0.6687 
\\ 
-0.6687 & 0.3068
\end{bmatrix}$,
which notably assigns much less penalty to changes in the vertical axis compared to changes in the horizontal axes. As a result, the learner remains within the feasible set of the first task and projects onto the offline feasible set, effectively retaining the knowledge of the first task without any forgetting.

This experiment suggests that the key to avoiding forgetting in weighted regularization schemes may lie in employing ill-conditioned weighting matrices.
This stands in contrast to linear \emph{regression}, where the Fisher-information matrix always prevents forgetting 
(\propref{prop:regression_fisher}), even when the datapoints of the previous task are orthogonal and the weighting matrix is proportional to the identity matrix.

\medskip

\begin{figure}[h!]
\centering
\begin{minipage}[t]{0.92\linewidth}
\caption{
Solving two different task sequences while using their corresponding Fisher-Information matrices for weighting.
\\
We plot the iterates obtained by solving the weakly-regularized weighted \procref{proc:adaptive} with Normalized Gradient Descent (NGD). 
\\
We asserted that the obtained iterates agree with those of the weighted SMM \procref{proc:weighted_adaptive}.
}
\vspace{.9em}
\end{minipage}
\begin{subfigure}[h]{0.66\textwidth}
    \centering
    \includegraphics[width=.95\linewidth]{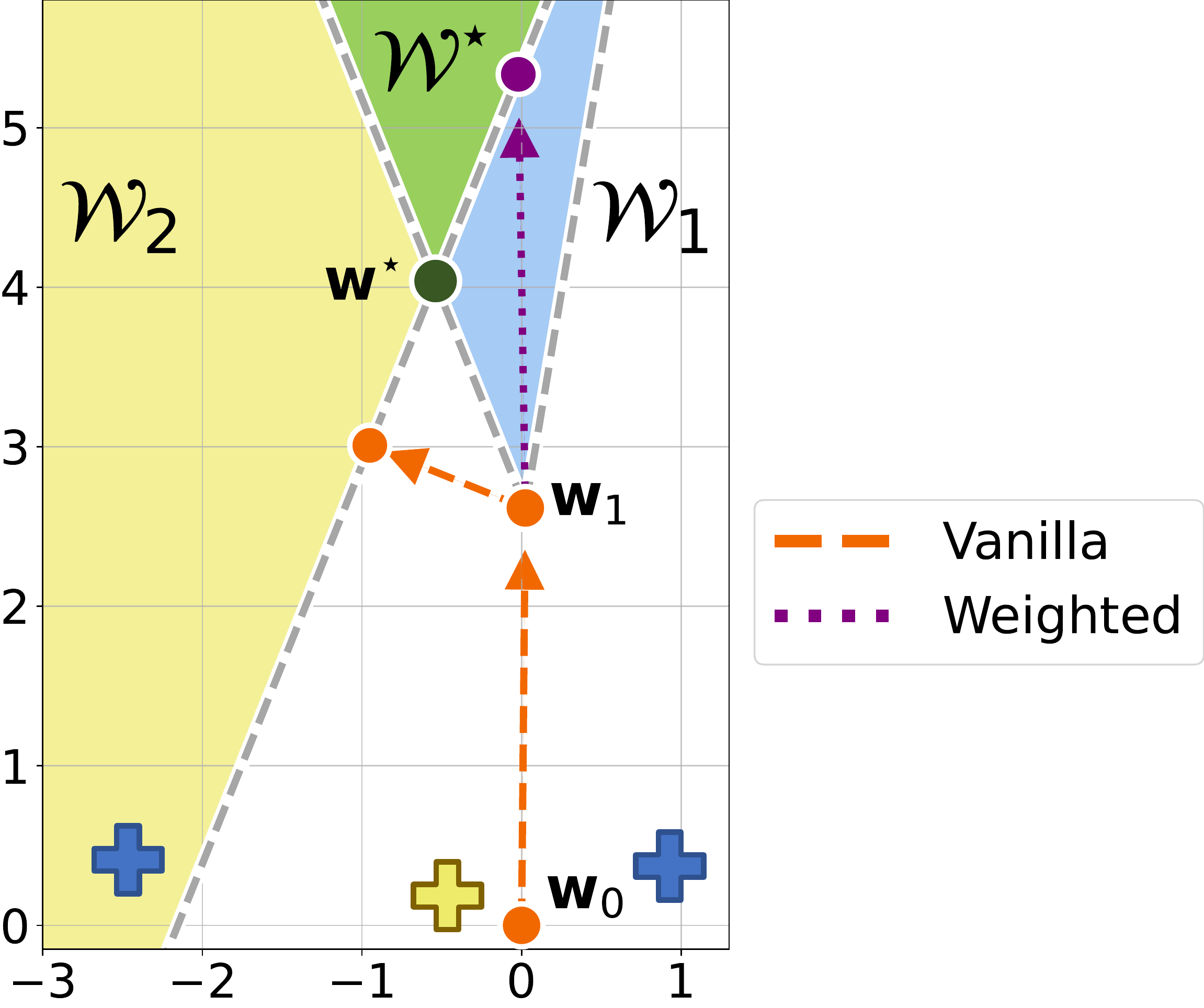}
\end{subfigure}
%
%
\begin{subfigure}[h]{0.3\textwidth}
    \centering
    \includegraphics[width=.95\linewidth]{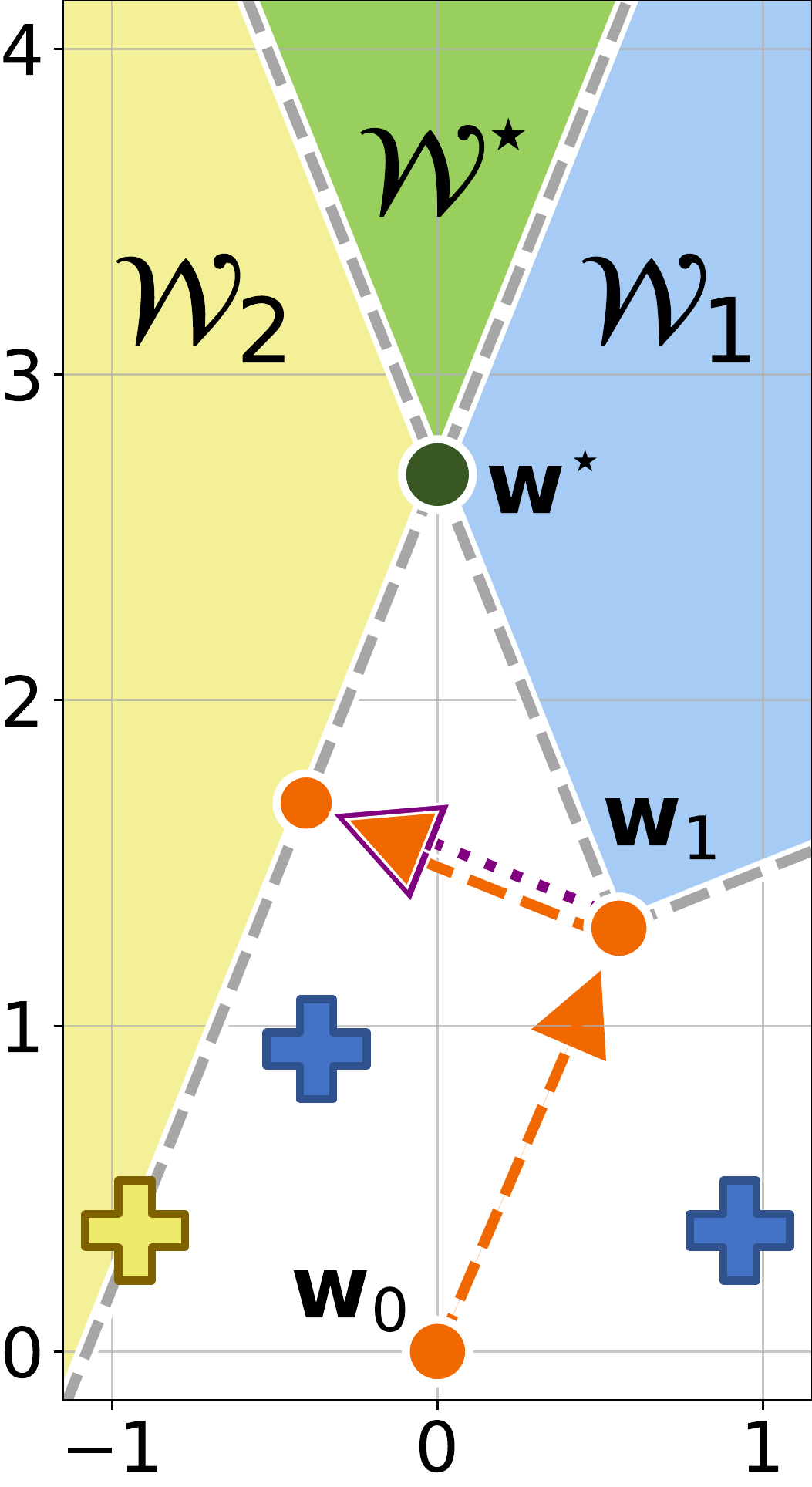}
\end{subfigure}
\end{figure}

\newpage

\section{Additional Material for the Early Stopping Discussion (\secref{sec:early})}
\label{app:early}

In Figure \ref{fig:early_stopping_vs_weakly_reg} we used $\lambda=\epsilon=\exp(-200)$. To make sure that this choice is small enough, we rerun the same setting as described in Figure \ref{fig:early_stopping_vs_weakly_reg} but with $\lambda=\epsilon=\exp(-400)$ and $\lambda=\epsilon=\exp(-600)$. 
The following figure demonstrates that the observed phenomenon remains unchanged --- early stopping and weak regularization consistently lead to distinct solutions. 

\vspace{.2cm}

\begin{figure}[h!]
\centering
\begin{subfigure}[h]{0.75\textwidth}
    \centering
    \includegraphics[width=.95\linewidth]{figures/early200.pdf}
\caption{Same as Figure \ref{fig:early_stopping_vs_weakly_reg} with $\lambda=\epsilon=\exp(-200)$}
\end{subfigure}

\vspace{.5cm}

\begin{subfigure}[h]{0.75\textwidth}
    \centering
    \includegraphics[width=.95\linewidth]{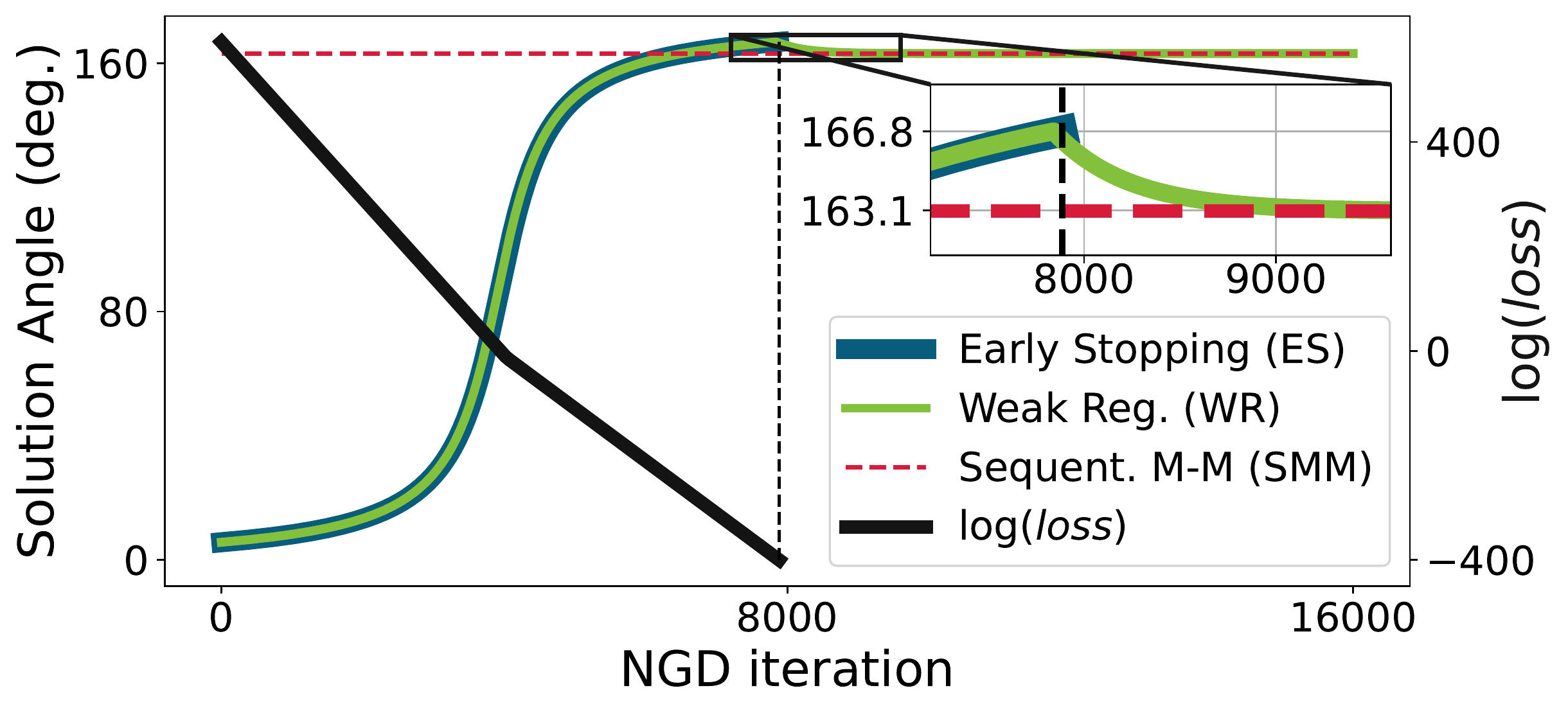}
\caption{$\lambda=\epsilon=\exp(-400)$}
\end{subfigure}

\vspace{.5cm}

\begin{subfigure}[h]{0.75\textwidth}
    \centering
    \includegraphics[width=.95\linewidth]{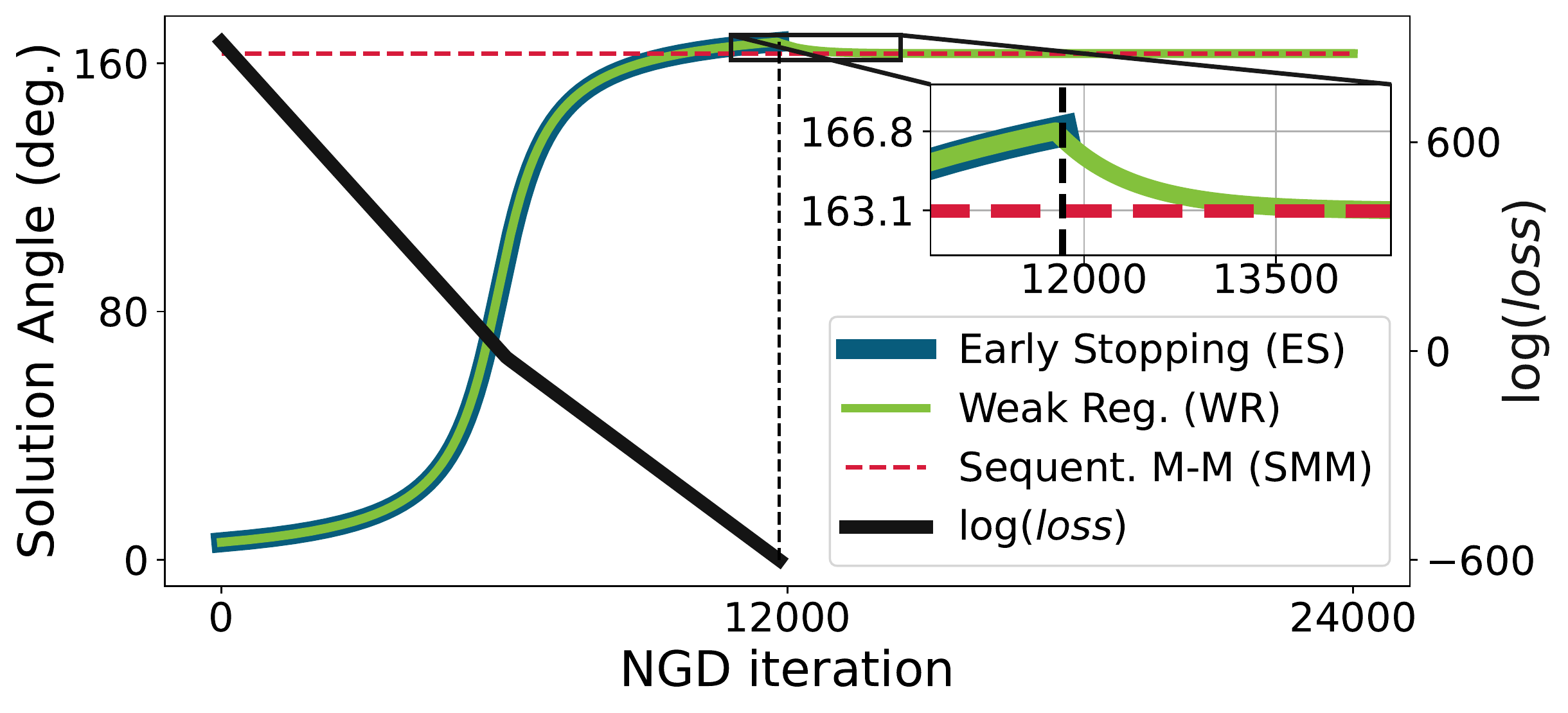}
\caption{$\lambda=\epsilon=\exp(-600)$}
\end{subfigure}
\caption{Repeating the experiment in Figure \ref{fig:early_stopping_vs_weakly_reg} with smaller $\lambda$ and $\epsilon$.}
\label{early_stopping_eps}
\end{figure}

\end{document}